\documentclass[11pt]{article}
\usepackage[margin=1in, letterpaper]{geometry}
\usepackage[T1]{fontenc}

\title{Bandit PCA with Minimax Optimal Regret \vspace{0.5cm}}

\author{
   \textbf{Mo\"ise Blanchard}\\
   Georgia Tech\\
   \small{\texttt{mblanchard41@gatech.edu}}
   \and 
    \textbf{Dmitrii Ostrovskii}\\
   Georgia Tech\\
   \small{\texttt{ostrov@gatech.edu}}
   \and
   \textbf{Aadirupa Saha}\\
   University of Illinois, Chicago\\
   \small{\texttt{aadirupa@uic.edu}}
}

\date{}

\usepackage{fullpage}
\usepackage{microtype}
\usepackage{centernot}
\usepackage{amssymb}
\usepackage{pifont}
\usepackage{nicefrac}
\usepackage{cancel}
\usepackage{array}
\usepackage{graphicx}
\usepackage{dsfont}
\usepackage{booktabs} 
\usepackage{amssymb}
\usepackage{color}
\usepackage{amsmath}
\usepackage{amsthm}
\usepackage{thmtools}
\usepackage{thm-restate}
\usepackage[linesnumbered,ruled,noend]{algorithm2e}
\usepackage{enumitem}
\usepackage{setspace}
\usepackage{tikz}
\usepackage{bm}
\usepackage{hyperref}
\usepackage[capitalize]{cleveref}
\hypersetup{colorlinks=true,linkcolor=blue,citecolor=blue}
\usepackage{caption}
\usepackage{float}
\usepackage{mdframed}

\usepackage[mathscr]{eucal}
\usepackage{thm-restate}

\newtheorem{theorem}{Theorem}
\newtheorem{lemma}[theorem]{Lemma}

  \newtheorem{proposition}[theorem]{Proposition}
  \newtheorem{remark}[theorem]{Remark}

\newcommand{\Bcal}{\mathcal{B}}
\newcommand{\Ccal}{\mathcal{C}}
\newcommand{\Dcal}{\mathcal{D}}
\newcommand{\Ecal}{\mathcal{E}}
\newcommand{\Fcal}{\mathcal{F}}
\newcommand{\Gcal}{\mathcal{G}}
\newcommand{\Hcal}{\mathcal{H}}
\newcommand{\Ical}{\mathcal{I}}

\newcommand{\Ncal}{\mathcal{N}}

\newcommand{\Scal}{\mathcal{S}}

\newcommand{\Ucal}{\mathcal{U}}

\newcommand{\Wcal}{\mathcal{W}}

\newcommand{\Rcal}{\mathcal{R}}

\newcommand{\Cbb}{\mathbb{C}}
\newcommand{\Ebb}{\mathbb{E}}
\newcommand{\Nbb}{\mathbb{N}}
\newcommand{\Pbb}{\mathbb{P}}

\newcommand{\Rbb}{\mathbb{R}}

\newcommand{\1}{\mathds{1}}

\definecolor{dark_red}{rgb}{0.2,0,0}

\newcommand{\tr}{\mathrm{tr}}
\newcommand{\diag}{\mathrm{diag}}
\newcommand{\Span}{\mathrm{span}}
\newcommand{\Vol}{\mathrm{Vol}}
\newcommand{\Reg}{\mathrm{Reg}}
\newcommand{\Var}{\mathrm{Var}}

\newcommand{\paren}[1]{\left( #1 \right)}
\newcommand{\sqb}[1]{\left[ #1 \right]}
\newcommand{\set}[1]{\left\{ #1 \right\}}
\newcommand{\floor}[1]{\left\lfloor #1 \right\rfloor}
\newcommand{\ceil}[1]{\left\lceil #1 \right\rceil}
\newcommand{\abs}[1]{\left|#1\right|}
\newcommand{\norm}[1]{\left\|#1\right\|}
\newcommand{\dotp}[2]{\langle #1,#2 \rangle}

\newcommand{\comment}[1]{}

\usepackage{amsthm}
\usepackage{mathrsfs}
\usepackage{bigints}

\newcommand{\cE}{\mathcal{E}}
\newcommand{\cB}{\mathcal{B}}
\newcommand{\cU}{\mathcal{U}}

\newcommand{\R}{\mathbb R}
\newcommand{\E}{\mathbb E}
\newcommand{\Pp}{\mathbb P}
\newcommand{\Tr}{\operatorname{tr}}
\newcommand{\rank}{\operatorname{rank}}
\newcommand{\Gr}{\operatorname{Gr}}
\newcommand{\op}{\mathrm{op}}

\newcommand{\Exp}{\operatorname{Exp}}
\newcommand{\GammaDist}{\operatorname{Gamma}}
\newcommand{\BetaDist}{\operatorname{Beta}}

\newcommand{\Unif}{\operatorname{Unif}}

\newcommand{\Beta}{\operatorname{Beta}}

\newcommand{\Algo}{\mathscr{A}}
\newcommand{\Disc}{\mathscr{D}}

\newcommand{\vphtop}{{\vphantom\top}}

\renewcommand{\preceq}{\preccurlyeq}
\renewcommand{\succeq}{\succcurlyeq}

\renewcommand{\le}{\leqslant}
\renewcommand{\ge}{\geqslant}

\renewcommand{\leq}{\le}
\renewcommand{\geq}{\ge}

\begin{document}

\pagenumbering{gobble} 

\maketitle

\vspace{1.5cm}

\begin{abstract}
We study the bandit-feedback version of online principal component analysis (Bandit PCA):
in each round~$t = 1,\dots,T$, the adversary selects a~$d \times d$ symmetric gain matrix~$G_t$ with spectrum in~$[0,1]$ and rank at most~$r$; the learner simultaneously selects a unit vector~$\smash{w_t \in S^{d-1}}$ and receives the reward~$\smash{w_t^\top G_t^\vphtop w_t^\vphtop}$. 
The learner receives no other feedback, and aims to minimize the regret against the best unit vector in hindsight.
This problem was introduced by Kot{\l}owski and Neu (2019), who gave an algorithm with regret~$O(d\sqrt{rT \log T})$ and showed the lower bound of~$\Omega(r\sqrt{T/\log T})$.
We improve upon both of these bounds and essentially bridge the gap between them, establishing the minimax regret of order~$\smash{r\sqrt{dT}}$ up to polylogarithmic factors in~$d$ and~$T$. 
The upper bound is attained by a novel algorithm, which combines online mirror descent on the spectrahedron of (real) density matrices with a multiscale exploration scheme in which the eigenspaces with different spectral magnitudes are updated at different rates. For the lower bound, we construct an adaptive adversary that refines a hidden large-reward subspace based on the learner's actions, in such a way that low regret is impossible without estimating the subspace; 
as a result, lower-bounding the regret reduces to studying the arising subspace estimation problem. 
Finally, we discuss connections of Bandit PCA with adaptive-measurement quantum tomography.
\end{abstract}

\newpage
\doublespacing
\tableofcontents
\singlespacing

\newpage 
\pagenumbering{arabic}

\section{Introduction}
\label{sec:intro}

Online learning with bandit feedback stands at the intersection of
statistical decision theory and sequential optimization: the learner must act without access to counterfactual outcomes, relying solely on the scalar reward of the chosen action.  Classical results for multi-armed and linear bandits~\cite{auer2002finite,bubeck2012minimax} give regret bounds scaling as $\tilde O(\mathrm{poly}(d)\sqrt{T})$ in the number of rounds~$T$, where $d$ is the ambient dimension.
In this paper, we study matrix-valued bandit setting, in which actions are unit vectors in~$\mathbb{R}^d$, rewards are evaluations of quadratic forms, and the adversary controls a sequence of positive-semidefinite (PSD) gain matrices. 
Following~\cite{kotlowski2019bandit} who first introduced and studied this setting, we shall refer to it as \emph{Bandit Principal Component Analysis} (Bandit PCA).
Simply put, Bandit PCA is the noncommutative generalization of the classic multi-armed bandit (MAB) problem, obtained by replacing the probability simplex by the spectrahedron of unit-trace PSD matrices; respectively, the set of~$d$ ``arms'' is replaced with that of~$d \times d$ rank-one projectors (extreme points of the spectrahedron).
Conversely, MAB can be interpreted as the restricted version of Bandit PCA, where both parties must play in a fixed orthonormal basis; removing this constraint creates an extra challenge on both sides. 
Before discussing this further, let us formalize the setting.

\paragraph{Bandit PCA game.}
In each round~$t \in [T] := \{1, \dots, T\}$, the adversary selects a PSD {gain matrix}~$G_t$ with operator norm at most 1; simultaneously, the learner selects a unit vector~$w_t \in S^{d-1}$ and receives the reward~$\smash{\tr(G_t^\vphtop w_t^\vphtop w_t^\top) = w_t^\top G_t^\vphtop w_t^\vphtop}$. 
The learner receives partial feedback, observing only the scalar reward and not~$G_t$; meanwhile, the adversary observes~$w_t$, but only {after} selecting~$G_t$.
The goal of the learner is to minimize the regret against the best fixed direction in hindsight:
\[
  \Reg_T \;:=\; \max_{w \in S^{d-1}} \Ebb \Bigg[\sum_{t \in [T]} \tr(G_t w w^\top)
             \;-\; \sum_{t \in [T]} \tr(G_t^\vphtop w_t^\vphtop w_t^\top) \Bigg].
\]
Throughout, we will study the case where gain matrices are constrained to have rank at most $\smash{r\in[d]}$, which we will refer to as {\em rank-$r$ Bandit PCA}. Thus, the setting is defined by three parameters~$T, d, r$.
Bandit PCA is nontrivially connected with other problems. Let us briefly discuss these connections.

\paragraph{Rank-one case: connections to classic online PCA and phase retrieval.} 
For rank-1 Bandit PCA, i.e.~when~$G_t^\vphtop = x_t^\vphtop x_t^\top$~with~$x_t $ in the unit ball, the reward simplifies to~$(w_t^\top x_t)^2$. As noted in~\cite{kotlowski2019bandit}, the full-information counterpart of this scenario is the classic online PCA as introduced in~\cite{warmuth2006randomized}, where the learner aims to project a sequence of unknown vectors onto chosen directions~$w_1, \dots w_T$ selected online to maximize the cumulative sum~$(w_t^\top x_t)^2$; the benchmark is the top eigenvector of the cumulative gain matrix\footnote{Subsequent literature made the point of removing the rank-one constraint and replacing~$x_t x_t^\top$ with a general PSD gain matrix~\cite{nie2016online}. Arguably, this renders the general rank-$r$ Bandit PCA no less relevant than its rank-1 case.} $\sum_{t \in [T]} x_t^\vphtop x_t^\top$.
On the other hand,~\cite{kotlowski2019bandit} also observe that rank-1 bandit PCA can be thought of as an {adversarial} online version of {\em phase retrieval} -- the task of recovering a hidden vector~$x$ from the measurements~$\smash{|w_t^\top x|^2}$. The more conventional formulation (see e.g.~\cite{lattimore2021bandit}) corresponds to sequential measurements of a fixed target~$x^\star$, whereas in our setting the adversary changes~$x_t$ in each round.

\paragraph{Connection to quantum tomography.} This problem also naturally connects to (incoherent) \emph{quantum state tomography} \cite{banaszek2013focus,chen2023does,flammia2024quantum}, in which a learner aims to estimate an unknown quantum state $\rho$---Hermitian matrix in $\Cbb^{d\times d}$ with unit trace---from a sequence of measurements on copies of this unknown state.\footnote{In the \emph{coherent} version of quantum state tomography, a learner is allowed to make arbitrary measurements on the product state $\rho^{\otimes n}$ \cite{haah2016sample,o2016efficient}. However, implementing such measurements in the real world seems challenging, and as such the \emph{incoherent} version of the problem has received increasing interest for its practicality \cite{huang2022quantum}.} In comparison, in Bandit PCA the learner tracks the top eigenvector of a changing gain matrix, which can be interpreted as tracking the most probable eigenstate of an evolving measurement operator. (In the following discussion, quantum states and operators are real.)

Furthermore, consider the scenario where gain matrices are restricted to the spectrahedron~$\Scal_d$ of {\em unit-trace} PSD matrices; note that in rank-1 Bandit PCA, this holds automatically.
In this scenario, a Bandit PCA learner performs {\em quantum tomography} of an evolving state~\cite{aaronson2007learnability}, with the~$r = 1$ case corresponding to the tomography of a {pure} state. 
Indeed, each matrix in~$\Scal_d$ corresponds to a quantum state~$\rho_t$, and selecting~$w_t\in S^{d-1}$ corresponds to performing a quantum measurement onto $\rho_t$ with the positive operator valued measurement (POVM)~$\{w_tw_t^\top,I-w_tw_t^\top\}$, which returns a Bernoulli random variable with success probability given by the reward~$\dotp{\rho_t}{w_tw_t^\top}=\tr(G_tw_tw_t^\top)$. As a minor difference with our setting, note that in Bandit PCA, one directly observes the mean of this measurement $\dotp{\rho_t}{w_tw_t^\top}$ rather than the random bit. Up to this minor difference, however, the Bandit PCA problem with gain matrices in~$\Scal_d$ corresponds to the following quantum stabilization problem: at each round $t\in[T]$, the learner faces an unknown quantum state $\rho_t$ and selects a pure state $w_tw_t^\top$ onto which it aims to stabilize~$\rho_t$ by making the POVM measurement $\{w_tw_t^\top, I-w_tw_t^\top\}$. The learner then receives the reward of~$1$ if the collapsed state is indeed~$w_tw_t^\top$, and of~$0$ otherwise. 

\paragraph{Prior work and challenges.}
Since the reward at each iteration is linear in the selected projection matrix $w_tw_t^\top$, the bandit PCA problem can be naively reduced to standard linear bandits, by vectorizing the matrices involved.\cite{warmuth2006randomized}. This vectorization needs ambient dimension $O(d^2)$, which leads to the regret $O(d^2\sqrt{T\log T})$ for the continuous version of the Exponential Weights algorithm (EXP2); see~\cite{dani2007price,hazan2010learning,hoeven2018many}. On the other hand, even when the eigenbasis of the gain matrix is fixed and known a priori (equivalently, if these matrices are diagonal), the problem reduces to classical MAB with $d$ arms, for which the regret admits an~$\Omega(\sqrt{dT})$ lower bound \cite{auer2002nonstochastic}.

The above gap was tightened by~\cite{kotlowski2019bandit} who provided an algorithm 
that achieves regret $O(d\sqrt{rT\log T})$ when the gain matrices are of rank at most $r$. In fact, their regret bound more generally holds when $r$ is the average squared Frobenius norm of the gain matrices. On the lower-bound side, they showed a $\Omega(r\sqrt{T/\log T})$ regret lower bound,\footnote{They prove this bound specifically for $r=d$ but this can be directly extended to lower rank by restricting the problem to a subspace of dimension $r$.} leaving a gap of order $\sqrt{d\min(r,d/r)}$.
Their algorithm follows the classical template of Online Mirror Descent (OMD) \cite{nemirovski1983problem,beck2003mirror,hazan2016introduction} over the spectrahedron
$\Scal_d$ in which the learner sequentially performs mirror descent updates on a density matrix $W_t\in\Scal_d$. Implementing this template relies on two crucial coupled challenges: (1) choosing how to sample at each round the vector $w_t$ in order to match the OMD density matrix in expectation $\Ebb_t[w_tw_t^\top]=W_t$; (2) using the corresponding scalar feedback $\tr(G_tw_tw_t^\top)$ to construct an ideally unbiaised estimator $\widehat G_t$ of the gain matrix $G_t$ with small variance. Addressing both these challenges required from~\cite{kotlowski2019bandit} to introduce novel techniques. First, they devised a sampling scheme that aims to separately estimate diagonal and off-diagonal entries of the gain matrix $G_t$; second, they used the negative log determinant $-\log\det(W_t)$ as a regularizer \cite{tsuda2005matrix}. 
It is worth noting that the standard analysis seems to fail for more standard choices of regularizers, including the standard and Tsallis matrix entropies~\cite{allen2015spectral}; see the discussion in~\cite[Appendix C]{kotlowski2019bandit}. This contrasts with the full-information setting, for which the standard matrix entropy yields optimal regret \cite{nie2016online}.

While we preserve the template of OMD with~$-\log\det(\cdot)$ regularizer, it turns out that closing the regret gap requires further innovations on both fronts -- both for sampling and for gain estimation.
At the high level, the suboptimal regret factors in the upper bound of~\cite{kotlowski2019bandit} arise from the estimation of the off-diagonal terms in the gain matrix, each scaling as~$\Theta(d^2)$. 
These are crucial for PCA, since they enclose crucial information on whether the top principal component has rotated: as a concrete example in dimension~2, consider~$G = (1-\theta^2) e_1^\vphtop e_1^\top \pm \theta (e_1^\vphtop e_2^\top + e_2^\vphtop e_1^\top)$ and note that~$\theta \gg \theta^2$ for~$\theta \ll 1$.
Hence, improving the estimation quality of the off-diagonal entries of the gain matrix, and exploiting the structure of such an estimate in a sampling scheme, is key to lowering the regret. 
As a concrete implementation of this idea, we introduce (1) a \emph{layered sampling scheme} for $w_t$, in which we group eigenvectors from the OMD matrix $W_t$ by eigenvalue magnitude, then perform a uniform exploration over carefully selected groups; (2) a layered \emph{epoch-based estimator} of the gain matrices. Intuitively, for each group, we batch the estimation of their off-diagonal blocks throughout their epoch. This turns out to significantly reduce the variance for off-diagonal blocks once amortized over the epoch.

\begin{table}[t]
\centering
\begin{tabular}{lccc}
\toprule
\textbf{Approach}
& \textbf{Upper bound}
& \textbf{Lower bound} 
& \textbf{Reference}\\
\midrule
Linear bandits / EXP2
& $d^2\sqrt{T}$
& --
&\cite{dani2007price,hoeven2018many}\\
MAB reduction & -- & $\sqrt{dT}$ &
\cite{auer2002nonstochastic}\\
OMD with log-det reg.
& $d\sqrt{rT}$
& $r\sqrt{T}$ 
&\cite{kotlowski2019bandit}\\
Epoch-OMD + layered exploration
& $r\sqrt{dT}$
& $r\sqrt{dT}$ 
&\textbf{This work}\\
\bottomrule
\end{tabular}
\caption{Regret bounds for bandit PCA with spectral norm at most one and rank $\leq r$ (or nuclear norm $\leq r$). Logarithmic factors in $d$ and $T$ are omitted for simplicity.}
\label{tab:bandit-pca-bounds}
\end{table}

\paragraph{Our contributions.} We characterize the minimax regret of rank-$r$ Bandit PCA up to logarithmic factors. Specifically, our first main result is an algorithm coming with the following regret guarantee.
\begin{theorem}[Simplified \cref{thm:ub}]
\label{thm:ub-intro}
  Fix $r\in[d]$. Suppose $\|G_t\|_{\mathrm{op}} \leq 1$ and $\|G_t\|_* \leq r$
  for all $t \in [T]$. Then, Algorithm~\ref{alg:final_alg} run with appropriate parameters achieves, for some universal constant $C>0$,
  \begin{equation*}
    \Reg_T \leq C  r\sqrt{d T}\,\log^3(edT).
  \end{equation*}
\end{theorem}
This is complemented by the following adaptive regret lower bound, tight up to logarithmic factors. (Note that, in addition to~\cref{thm:ub-intro}, we always have the trivial upper bound $\Reg_T\leq T$.)
\begin{theorem}
\label{thm:lb-intro}
  Fix $r\in[d]$. For any algorithm there is an adaptive instance satisfying $\|G_t\|_{\op}\leq 1$ and $\mathrm{rank}(G_t)\leq r$ for all $t\in[T]$ such that, for some universal constant $c>0$,
  \begin{equation*}
      \Reg_T \geq c\min \bigl(r\sqrt{dT/\log(ed)},T \bigr)/\log(eT).
  \end{equation*}
\end{theorem}

These results are summarized in Table~\ref{tab:bandit-pca-bounds}. Interestingly, these findings show that both previous upper and lower bounds were loose and tight for different regimes of the rank parameter~$r$. 
Specifically, for small rank~$r=\Theta(1)$, the naive MAB lower bound is, in fact, tight, and our algorithm achieves an~$\tilde O(\sqrt{d})$ improvement over previous work. 
On the other hand, for large rank~$r=\Theta(d)$, the previous algorithm from \cite{kotlowski2019bandit} was already tight up to logarithmic factors and we instead improve the lower bound by an $\tilde O(\sqrt{d})$ factor.
Our proposed algorithm (Algorithm~\ref{alg:final_alg}) follows the same OMD template with a negative log-determinant regularizer as in \cite{kotlowski2019bandit} but introduces the following new ideas:

\begin{enumerate}[label=(\roman*)]
    \item \textbf{Layered sampling scheme (\cref{alg:layered_sampling}).}
    Eigenvectors of the OMD iterate $W_t$ are partitioned into $L = O(\log(dT))$
    \emph{layers} indexed by eigenvalue magnitude~$\mu_\alpha =
    2^{-\alpha}$, forming the corresponding layer subspace $E_\alpha$. We then select $W_t$ by sampling the layer to be explored then sampling $w_t$ uniformly within the span of all previous layers $E_{\leq\alpha}:=\Span(E_\beta,\beta\leq \alpha)$.
    \item \textbf{Epoch-based gain estimator with batching (\cref{alg:gain_estimation}).} We estimate gain matrices by layer $\alpha\in[L]$ along the reverse-L-shaped region represented by $G^{(\alpha)}:=P_{E_{\leq\alpha}}GP_{E_{\leq\alpha}} - P_{E_{<\alpha}}GP_{E_{<\alpha}}$ (see~\cref{fig:layered_estimators}). This corresponds to the diagonal block and all off-diagonal blocks involving layer $E_{\alpha}$. To fully reap the benefit of the uniform sampling of $w_t$ onto $E_{\leq\alpha}$, we \emph{batch} $G^{(\alpha)}$ estimates along \emph{epochs}. This is because the variance analysis involves bounding the spectral norm of the gain matrix estimator; meanwhile, the uniform sampling of $w_t$ for layer $\alpha$ explores nearly orthogonal directions, which does not significantly increase the spectral norm of the estimate (at least for the first $O(\dim(E_{\leq\alpha}))$ samples in this layer). To achieve the optimal variance reduction (amortized over epochs), we 
    use a hierarchical schedule for epochs of each layer: epochs for layer $\alpha\in[L]$ last for $2^\alpha$ rounds, and so last for $2$ layer-$(\alpha-1)$ epochs.
    \item \textbf{Lazy update for layers (\cref{alg:subspace_update}) and diagonal handling.} Using a batching strategy within the OMD template requires handling two specific details. First, we need to ensure that the exploration subspace $E_{\leq\alpha}$ is kept constant throughout a layer-$\alpha$ epoch. To this end, we ``lazily'' update each~$E_\alpha$ only at the end of its respective epoch. Second, the projection step of OMD updates for the previous layers ($\beta<\alpha$) may indirectly affect the diagonal terms in the $E_\alpha\times E_\alpha$ block of the iterate $W_t$. To address this, it suffices to separately estimating the diagonal of~$G$ as done in \cite{kotlowski2019bandit}.
\end{enumerate}

\section{Problem statement}
\label{sec:problem}

\paragraph{Notation.}
We write~$[n] := \{1,\ldots,n\}$, 
We use the standard notation~$S^{d-1} := \{x \in \mathbb{R}^d :
\|x\|_2 = 1\}$ and~$B_d(x,r) := \{y \in \mathbb{R}^d : \|y-x\|_2
\leq r\}$ for the unit sphere and~$\ell_2$-balls in~$\R^d$. 
$\mathbb{S}_+^d$ denotes the cone of $d\times d$ real
symmetric PSD matrices; $\mathcal{S}_d := \{W \in \mathbb{S}_+^d :
\tr(W) = 1\}$ is its unit-trace section. $\|\cdot\|_{\mathrm{op}}$ and
$\|\cdot\|_*$ denote the operator and nuclear norms; $\langle A, B \rangle := \tr(AB^\top)$ the matrix inner product. For a subspace~$E \subseteq \mathbb{R}^d$, $P_E$ denotes the orthogonal projector onto~$E$. For a sequence $(M_t)_{t\geq 1}$ and $a \leq b$, we let~$M_{a:b} := \sum_{s=a}^b M_s$. 
We write $\mathbb{E}_t[\,\cdot\,]$ for conditional expectation given all history at the start of round $t$. We use the standard Landau $O,\Theta,\Omega$ notation.
More notation will be introduced as needed.

\paragraph{General Bandit PCA protocol.}
Fix a horizon $T \in \mathbb{N}$ and a rank budget $r \in[d]$. 
In each round $t=1,\ldots,T$, the interaction proceeds as follows:
\begin{enumerate}[label=(\arabic*)]
  \item The adversary selects a gain matrix
        $G_t \in \mathbb{S}_+^d$ satisfying
        $\|G_t\|_{\mathrm{op}} \leq 1$ and $\mathrm{rank}(G_t)\leq r$. When specified, we will consider the nuclear-norm convex relaxation of the rank constraint: $\|G_t\|_*\leq r$.
  \item The learner selects a vector in the unit ball:~$w_t \in B_d(0,1).$
  \item The learner observes the scalar reward
        \[
        \ell_t^2 := \tr(G_t\, w_t w_t^\top) = w_t^\top G_t w_t.
        \]
\end{enumerate}
The learner's choice of $w_t$ may depend on the past history $\Hcal_{t-1}:=\sigma(w_i,\ell_i^2,i<t)$ but \emph{not} on $G_t$. 
The learner aims to minimize the \emph{regret} against the best fixed direction:
\begin{equation}
\label{eq:regret}
  \Reg_T := \max_{w \in S^{d-1}}
  \Ebb\sqb{\sum_{t=1}^T \tr(G_t ww^\top)
  -
  \sum_{t=1}^T \tr (G_tw_tw_t^\top)},
\end{equation}
where the expectation is taken over the internal randomness of the learner.
This objective is often referred to as~\emph{pseudo-regret} (see e.g.\cite{bubeck2012regret}).
Note that, since all gain matrices are PSD, this coincides with the regret against the unit ball $B_d(0,1)$, and the best comparator is precisely the top eigenvector of the cumulative gain matrix $\Ebb[G_{1:T}] = \sum_{t=1}^T\Ebb[G_t].$

\section{Near-optimal algorithm for Bandit PCA}
\label{sec:ub}

\begin{figure}[t]
\centering

\begin{minipage}[t]{0.708\textwidth}
\vspace{0pt}
\centering

\begin{algorithm}[H]
\caption{OMD with Layered Exploration}
\label{alg:final_alg}

\DontPrintSemicolon

\SetKwInOut{Input}{Input}

\Input{stepsize $\eta$, exploration prob.\ $\gamma$, nb.~of layers $L\geq\log(\frac{d}{\gamma})$}

Initialization: set $E_{0,\leq L}=\Rbb^d$ and fix any orthonormal basis $u_{0,1},\ldots,u_{0,d}\in\Rbb^d$

\For{$t=1,\ldots,T$}{
    Let $W_t:=(c_t I-\eta\widehat G_{1:t-1})^{-1}$ with $c_t\in\Rbb$ such that $W_t\in\Scal_d$\label{line:OMD_update}
    
    Let $U_t:=(1-\gamma)W_t+\frac{\gamma}{d} I$ \label{line:uniform_mixing_U}

    $(E_{t,\leq [L]}, u_{t,[d]}) := \textsc{SubspaceUpdate}_t(U_t,E_{t-1,\leq [L]},u_{t-1,[d]})$.
    
    $(w_t,Z_t,I_t):= \textsc{LayeredSampling}(E_{t,[L]})$

    Select $w_t$ and receive feedback $\ell_t^2:=\tr(G_t w_tw_t^\top)$ \label{line:normalize_query+feedback}

    $\widehat G_t:= \textsc{GainMatrixEstimation}_t(w_t,\ell_t^2,Z_t,I_t, E_{t,[L]}, u_{t,[L]})$
}
\end{algorithm}

\end{minipage}
\hfill
\begin{minipage}[t]{0.284\textwidth}
\vspace{0pt}
\centering

\begin{tikzpicture}[font=\scriptsize,scale=0.66]
    \draw[thick] (0,0) rectangle (5,5);
    \draw (0,1)--(4,1)--(4,5);
    \draw (0,3)--(2,3)--(2,5);
    \draw (0,4)--(1,4)--(1,5);

    \node at (0.5,4.5) {$G^{(1)}$};
    \node at (1.5,3.5) {$G^{(2)}$};
    \node at (4.5,0.5) {$G^{(L)}$};

    \node at (-0.5,4.5){$E_1$};
    \node at (-0.5,3.5){$E_2$};
    \node at (-0.5,2){$\vdots$};
    \node at (-0.5,0.5){$E_L$};

    \node at (0.5,5.3){$E_1$};
    \node at (1.5,5.3){$E_2$};
    \node at (3,5.3){$\cdots$};
    \node at (4.5,5.3){$E_L$};
\end{tikzpicture}
\captionof{figure}{
Layered estimation of gains: each inverse-L-shaped submatrix is estimated separately ($\alpha \in [L]$).
}
\label{fig:layered_estimators}

\end{minipage}

\end{figure}

Before presenting our algorithm, we introduce some specific notation. Given non-decreasing subspaces $(E_{\leq\alpha})_{\alpha\geq 1}$, we let~$E_{t,<\alpha}:=E_{t,\leq \alpha-1}$, with the convention $E_{t,\leq 0}=\{0\}$. We also let $E_{t,\alpha}:=E_{t,\leq\alpha}\cap E_{t,<\alpha}^\perp$.
We now describe our main algorithm (Algorithm~\ref{alg:final_alg}); its regret analysis will be sketched in~\cref{subsec:regret_proof_sketch}.

\subsection{Algorithm construction}
\label{subsec:alg_description}
The algorithm uses the standard template of Online Mirror Descent (OMD) \cite{nemirovski1983problem,beck2003mirror,hazan2016introduction} over the spectrahedron $\Scal_d$ with negative log-determinant potential~$R(W):=-\log\det(W)$~\cite{kotlowski2019bandit}. Specifically, at each round $t\in[T]$, the algorithm maintains the OMD iterate
\[
W_t:=(c_t I-\eta\widehat G_{1:t-1})^{-1},
\]
where $\widehat G_1,\ldots,\widehat G_{t-1}$ are previously constructed gain estimates, and $c_t\in\Rbb$ is such that $W_t\in\Scal_d$. This is the closed-form OMD update under the negative log-determinant regularizer. 
For minor technical purposes, we inflate $W_t$ towards the uniform distribution:
\[
U_t := (1-\gamma) W_t + \frac{\gamma}{d},
\]
for a small term $\gamma\in(0,1)$ (which will be taken as $\gamma=1/T$).
Following this template, at round $t$ we select a direction $w_t$ matching the iterate $U_t$ in expectation: $\Ebb_t[w_tw_t^\top]=U_t$, then use the received feedback $\ell_t^2=\tr(G_tw_tw_t^\top)$ to construct a gain matrix estimate $\widehat G_t$. This gives the high-level structure given as~\cref{alg:final_alg}. 
Further algorithmic components are split into three subroutines---\textsc{SubspaceUpdate}, \textsc{LayeredSampling}, and \textsc{GainMatrixEstimation}---which we now describe.

\paragraph{Epoch-based update of exploration subspaces (Algorithm~\ref{alg:subspace_update}).} The algorithm crucially relies on a layered decomposition of the search space $\Rbb^d$. Concretely, we maintain a sequence of nested subspaces $\{0\}\subseteq E_{t,\leq 1}\subseteq \ldots \subseteq E_{t,\leq L}:=\Rbb^d$. We will refer to these as \emph{exploration subspaces}: when gatehring information about layer $\alpha$, the algorithm will uniformly sample unit vectors in $E_{t,\leq\alpha}$. Morally speaking, $E_{t,\leq\alpha}$ corresponds to the span of eigenvectors of $U_t$ with eigenvalues $\geq\mu_\alpha:=2^{-\alpha}$. 
As a remark, the added uniform exploration in $U_t$ compared to $W_t$ is precisely designed so that the total required number of layers $L$ is at most logarithmic ($U_t\succeq \tfrac{\gamma}{d}I$).

For our purposes, however, we instead update these subspaces according to the following hierarchical schedule of epochs: layer-$\alpha$ epochs last $2^\alpha$ rounds, and information about layer $\alpha$ is batched within each such epoch and exploited only at its end, to track eigenspaces of the iterate $W_t$ through the OMD iterate. This lazy update allows for a consistent subspace exploration which will be crucial for gain matrix estimations.
Layer-$\alpha$ exploration subspaces are therefore only updated at the beginning of coarser layer-$\beta$ epochs for $\beta>\alpha$.\footnote{Note that we do not update $E_{t,\leq\alpha}$ if $t$ is only the beginning of a layer-$\alpha$ epoch. This is because the previous OMD update will rotate eigenvalues within $E_{t,\leq\alpha}$ but the complete subspace $E_{t,\leq\alpha}$ is preserved. However, at the beginning of a layer-$\beta$ epoch for $\beta>\alpha$, eigenvalues within $E_{t,\leq\beta}$ will rotate, which also rotates the subspace $E_{t,\leq\alpha}$ within $E_{t,\leq\beta}$.}

Formally, at round $t\in[T]$ we define
\begin{equation*}
    \alpha_t := \max\{\alpha \in \{0,\ldots,L\} :
  t \equiv 1 \bmod 2^\alpha\},
\end{equation*}
so that round $t$ is exactly the start of layers $1,2,\ldots,\alpha_t$. Then, all explorations subspaces for layers $\alpha\geq \alpha_t$ are kept unchanged, while we update layer-$\beta$ exploration subspaces for $\beta\in[\alpha_t-1]$ via $E_{t,\leq\alpha}:=F_{t,\leq\alpha}\cap E_{t-1,\leq\alpha_t}$,
where $F_{t,\leq\alpha}$ is the span of eigenvectors of $U_t$ with eigenvalue $\geq\mu_\alpha$. Note that the second term in this update formula ensures that $E_{t,\leq\alpha}$ remains inside the coarser exploration subspace $E_{t,\leq\alpha_t}$. The corresponding \textsc{SubspaceUpdate} subroutine is summarized as~\cref{alg:subspace_update}.

\begin{algorithm}[t]
\caption{$\textsc{SubspaceUpdate}_t(U_t,E_{t-1,\leq [L]},u_{t-1,[d]})$}
\label{alg:subspace_update}

\DontPrintSemicolon

\SetKwInOut{Input}{Input}

\Input{round $t$, PSD matrix $U_t$, previous subspaces $E_{t-1,\leq \alpha}$ for $\alpha\in[L]$, eigenvectors $u_{t-1,[d]}$}

    Decompose $U_t=\sum_{i\in[d]}\lambda_{t,i} u_{t,i}u_{t,i}^\top$, ensuring that $u_{t,i}=u_{t-1,i}$ if~$u_{t-1,i}$ is an eigenvector of $U_t$ \label{line:eigenval_decompose_U}
    
    For $\alpha\in[L]$ let $F_{t,\leq\alpha}$ be the subspace of eigenvectors of $U_t$ with eigenvalues $\geq \mu_\alpha:=2^{-\alpha}$ \label{line:define_F}
    
    Let $\alpha_t=\max\{\alpha\in \{0,\ldots,L\}: t=1\bmod{2^\alpha}\}$ \tcp*{$t$ is the start of a layer-$\alpha_t$ epoch} \label{line:find_alpha_t}
    
    For $\alpha\in[L]$, define $E_{t,\leq \alpha}:=\begin{cases}
        F_{t,\leq \alpha} \cap E_{t-1,\leq \alpha_t} & \alpha<\alpha_t\\
        E_{t-1,\leq \alpha} &\alpha\geq\alpha_t
        \end{cases}$ \tcp*{Update subspaces} \label{line:update_subspaces}

    \Return $(E_{t,\leq\alpha})_{\alpha\in[L]}, (u_{t,i})_{i\in[d]}$
\end{algorithm}

\paragraph{Layered sampling scheme (Algorithm~\ref{alg:layered_sampling}).}
Given the exploration subspaces $(E_{t,\leq\alpha})_{\alpha\in[L]}$, the sampling scheme
subroutine constructs a distribution over decision vectors $w_t \in B_d(0,1)$
whose outer product satisfies $\Ebb_t[w_t w_t^\top] = U_t$. For technical purposes, we estimate diagonal and off-diagonal terms of the gain matrix separately, and a fair coin $Z_t\sim\mathrm{Bernoulli}(1/2)$ is tossed to choose between diagonal and off-diagonal estimation. For diagonal rounds ($Z_t=0$), we can use the standard strategy of sampling an eigenvector $w_t=u_{t,i_t}$ of $U_t$ proportionally to its eigenvalue $\lambda_{t,i_t}$ \cite{kotlowski2019bandit}. In the following, we focus on off-diagonal rounds ($Z_t=1$) which is the main case of interest.

We sample $w_t$ as a mixture of uniforms along each exploration subspace. Specifically, for each layer $\alpha \in [L]$, we denote $d_{t,\leq\alpha} := \dim(E_{t,\leq\alpha})$ and assign \emph{layer sampling probabilities} $p_{t,\alpha} := \tfrac{1}{4}\mu_\alpha\, d_{t,\leq\alpha}$. We then draw a layer index $I_t\in[L]$ with probability $\Pbb[I_t=\alpha]=p_{t,\alpha}$ and explore the corresponding layer by drawing $w_t$ \emph{uniformly} from the sphere within the layer:
\begin{equation*}
    w_t \sim \mathrm{Unif}(S^{d-1} \cap E_{t,\leq I_t}).
\end{equation*}

Note that on each layer $\alpha\in[L]$ the uniform sampling has expected covariance $P_{E_{T,\leq\alpha}}/d_{t,\leq\alpha}$.
With the remaining probability $1-\sum_\alpha p_{t,\alpha}$, we draw $w_t$ arbitrarily to cover the residual matrix
\begin{equation*}
    R_t := U_t - \sum_{\alpha\in[L]} p_{t,\alpha}
  (P_{E_{t,\leq\alpha}}/d_{t,\leq\alpha}).
\end{equation*}
To give some intuitions about the choice of layer sampling probability, recall that roughly speaking, each exploration subspace $E_{t,\leq\alpha}$ contains the eigenvectors of $U_t$ of eigenvalue $\geq \mu_\alpha$. Hence, we may afford to uniformly sample within $E_{t,\leq\alpha}$ with probability $O(\mu_\alpha d_{t,\leq\alpha})=p_{t,\alpha}$ without overflowing the covariance budget $U_t$, that is, $R_t\succeq 0$. In the proof, we show that this ``good scenario'' holds with high probability; otherwise, we may sample $w_t$ arbitrarily. The sampling scheme \textsc{LayeredSampling} is summarized in \cref{alg:layered_sampling} (we check that line~\ref{line:sample_w_to_complete_R} is valid in \cref{lemma:check_sampling_valid}).

\begin{algorithm}[t]
\caption{$\textsc{LayeredSampling}(E_{t,\leq [L]})$}
\label{alg:layered_sampling}

\DontPrintSemicolon

\SetKwInOut{Input}{Input}

\Input{exploration subspaces $E_{t,\leq \alpha}$ for $\alpha\in[L]$}

    Let~$\beta_t = \min\{\alpha\in[L] \hspace{-0.1cm}: \hspace{-0.05cm} E_{t,\leq\alpha}\hspace{-0.05cm}\neq\{0\}\hspace{-0.05cm}\}$. \hspace{-0.2cm} For\hspace{-0.05cm} $\alpha\in[L]$, let $d_{t,\leq\alpha} =\dim(E_{t,\leq\alpha} )$ and~$p_{t,\alpha} = \frac{1}{4}\mu_\alpha d_{t,\leq\alpha}$
    \label{line:define_dim_proba}
        
    Sample $Z_t\sim\mathrm{Bernoulli}(1/2)$ \label{line:sample_Z}

    \If{$Z_t=1$ and $R_t:=U_t - \sum_{\alpha=\beta_t}^L p_{t,\alpha}  \frac{P_{E_{t,\leq\alpha}}}{d_{t,\leq\alpha}} \succeq 0$ \label{line:good_event_holds}} 
    {
        
    
        Sample \hspace{-0.05cm}$I_t \hspace{-0.05cm}\in \hspace{-0.05cm}\{0\} \hspace{-0.05cm}\cup\hspace{-0.05cm} [L]$ with $\Pbb[I_t\hspace{-0.05cm}=\hspace{-0.07cm}\alpha]=p_{t,\alpha}$ for $\alpha\hspace{-0.05cm}\in\hspace{-0.05cm}[L]$ \hspace{-0.02cm}and\hspace{-0.02cm} $\Pbb[I_t\hspace{-0.05cm}=\hspace{-0.05cm}0]=p_{t,0}:=1\hspace{-0.05cm}-\hspace{-0.05cm}\sum_{\alpha\in[L]}p_{t,\alpha}$ \label{line:sample_I_t}
    
        \lIf{$I_t\in[L]$}{Sample $w_t\sim\mathrm{Unif}(S^{d-1}\cap E_{t,\leq I_t})$
        }\label{line:sample_w_good_scenario}
        \lElse{Sample $w_t$ with $\|w_t\|\leq 1$ arbitrarily such that $\smash{\Ebb[w_tw_t^\top]=\frac{R_t}{p_{t,0}}\1[p_{t,0}>0]}$}  \label{line:sample_w_to_complete_R}

        \Return $\smash{(w_t,Z_t,I_t)}$
    }
    \lElse{\DontPrintSemicolon
        Sample $i_t\sim\bm\lambda$ and \Return $(w_t:=u_{t,i_t},Z_t,\_)$ \tcp*[f]{Naive sampling}\hspace{-0.2cm} \label{line:sampling_diagonal}
    } 
    \end{algorithm}

\paragraph{Layered gain matrix estimation (Algorithm~\ref{alg:gain_estimation}).}
After observing $\ell_t^2 = \tr(G_t w_t w_t^\top)$, the subroutine
constructs an estimator $\widehat{G}_t$ of the gain matrix $G_t$. On diagonal rounds ($Z_t=0$) we use the standard importance sampling estimator 
\begin{equation*}
    O_t :=\frac{2\ell_t^2}{\lambda_{t,i_t}}u_{t,i_t} u_{t,i_t}^\top,
\end{equation*}
as in \cite{kotlowski2019bandit}. We now focus on the estimation of off-diagonal entries of the gain matrices, which is the main case of interest.
The purpose of the layered sampling in \cref{alg:layered_sampling} is to separately estimate the L-shaped regions $G_t^{(\alpha)}:=P_{E_{t,\leq\alpha}}G_t P_{E_{t,\leq\alpha}} - P_{E_{t,<\alpha}}G_t P_{E_{t,<\alpha}}$ of gain matrices $G_t$, as illustrated in \cref{fig:layered_estimators}. These regions arise relatively naturally from the estimator variance analysis. Additionally, for each layer $\alpha\in[L]$, responses are batched over the entire corresponding layer-$\alpha$ epoch (of length $2^\alpha$).

Specifically, for a layer-$\alpha$ epoch $[t_0,t]$, the accumulated estimator of~$G_{t_0:t}$ on~$E_{t,\leq\alpha}\times E_{t,\leq\alpha}$ is 
    \begin{equation*}
      A_{t,\alpha} \;:=\; \frac{4}{\mu_\alpha}
      \sum_{s \in [t_0,t]} \mathbf{1}[Z_s=1,\,R_s\succeq 0,\,I_s=\alpha]
      \cdot \ell_s^2\,\bigl[(d_{t,\leq\alpha}+2)\,w_s w_s^\top
      - P_{E_{t,\leq\alpha}}\bigr].
    \end{equation*}
    The form $\tfrac{1}{2}d_{t,\leq\alpha}[(d_{t,\leq\alpha}+2)w_s w_s^\top - P_{E_{t,\leq\alpha}}]$
    is an unbiased estimator of a matrix from uniform
    sphere samples~\cite{martinsson2020randomized}, which we rescale by $2/p_{t,\alpha}$ to
    compensate for the layer sampling probability. We then keep only the inverse-L-shaped region corresponding to $G_{t_1:t_2}^{(\alpha)}$ (see \cref{fig:layered_estimators}) and also delete the remaining diagonal (already captured by $O_t$), avoiding double-counted components:
    \begin{align*}
        B_{t,\alpha} := A_{t,\alpha}
      - P_{E_{t,<\alpha}} A_{t,\alpha} P_{E_{t,<\alpha}}
      - \sum_{u_{t,i}\in \Ucal_{t,\alpha}}
        \bigl(u_{t,i}^\top A_{t,\alpha}\, u_{t,i}\bigr)\,
        u_{t,i} u_{t,i}^\top,
    \end{align*}
    where $\Ucal_{t,\alpha}$ denotes the eigenvectors of $U_t$ spanning $E_{t,\alpha}$. 
    The final estimate is then the sum of all the estimated components 
    \begin{equation*}
        \widehat{G}_t := O_t + \sum_{\alpha \in [L]} B_{t,\alpha}.
    \end{equation*}
The estimation subroutine \textsc{GainMatrixEstimation} is summarized in \cref{alg:gain_estimation}.

\begin{algorithm}[t]
\caption{$\textsc{GainMatrixEstimation}_t(w_t,\ell_t^2,Z_t,I_t, E_{t,\leq [L]}, u_{t,[L]})$}
\label{alg:gain_estimation}

\DontPrintSemicolon

\SetKwInOut{Input}{Input}

\Input{$t$, $w_t$, response $\ell_t^2$, sampling parameters $(Z_t,I_t)$, subspaces $E_{t,\leq[L]}$, eigenvectors $u_{t,[d]}$}

    Let $O_t:=\1[Z_t=0]\frac{2}{\lambda_{t,i_t}} \ell_t^2 w_t w_t^\top$ \tcp*{Estimating diagonal terms}\label{line:estimate_diagonal}

    \For{$\alpha\in[L]$}{
        \If{$t=0\bmod{2^\alpha}$ or $t=T$ \texttt{(end of layer-$\alpha$ epoch)}}{
            $A_{t,\alpha}:= \frac{4}{\mu_\alpha}\sum_{s=2^\alpha\floor{(t-1)/2^\alpha}+1}^t \1[Z_s=1,R_s\succeq 0,I_s=\alpha]\ell_s^2 \paren{(d_{t,\leq \alpha}+2)w_sw_s^\top-P_{E_{t,\leq\alpha}}} $
            \label{line:estimate_off_diagonal_block_A}
            
            $B_{t,\alpha}:= A_{t,\alpha} - P_{E_{t,<\alpha}} A_{t,\alpha} P_{E_{t,<\alpha}} -\sum_{i\in J_{t,\alpha}} (u_{t,i}^\top A_{t,\alpha} u_{t,i}) u_{t,i}u_{t,i}^\top$ where $J_{t,\alpha}$ contains the indices of the eigenvectors $u_{t,i}$ spanning $E_{t,\alpha}$ \label{line:def_B_matrix}
        }
        \lElse{$A_{t,\alpha}=B_{t,\alpha}:=0$} \label{line:no_A_epoch_not_finished}
    }
    \Return $\widehat G_t:=O_t + \sum_{\alpha\in[L]} B_{t,\alpha}$ \tcp*{Final estimate}\label{line:final_gain_estimate}

\end{algorithm}

We emphasize that the batching step is crucial to achieve an improvement of the variance of each term $B_{t,\alpha}$ amortized over its layer-$\alpha$ epoch. 
This variance depends on the spectral norm $\|B_{t,\alpha}\|_{\op}\lesssim \|A_{t,\alpha}\|_{\op}$. Recall that the covariance terms involved in $A_{t,\alpha}$ correspond to uniform unit vectors in the corresponding layer-$\alpha$ exploration subspace $E_{t,\leq\alpha}$. Hence, the spectral norm is essentially not increased after batching $2^{\alpha-1} p_{t,\alpha}=\Theta(d_{t,\leq\alpha})$ covariance terms,\footnote{For instance, a $d$-dimensional standard Wishart with $\nu$ degrees of freedom $\Wcal_d(\nu, I)$ has spectral norm $\asymp \max(d,\nu)$.} which is roughly their expected number accross a layer-$\alpha$ epoch of length $2^\alpha$.

\paragraph{On the computational complexity of \cref{alg:final_alg}.} 
The computational bottleneck in each round is the eigenvalue decomposition of the iterate~$U_t$ to obtain the eigenvectors $(u_{t,i})_{i\in[d]}$  in line~\ref{line:eigenval_decompose_U} of \cref{alg:subspace_update}. A naive implementation would take an $O(d^3)$ runtime per iteration, yielding an $O(d^3 T)$ total runtime for \cref{alg:final_alg}. 
However, this implementation can be improved if we use the fact that eigenvectors of the matrix $U_t$ are unchanged except on the block $E_{t,\leq\alpha_t}\times E_{t,\leq\alpha_t}$ that has been updated at the end of round $t-1$. Indeed, 
\begin{equation*}
    \widehat G_{t-1} = O_{t-1} + \sum_{\alpha=1}^{\alpha_t} B_{t-1,\alpha},
\end{equation*}
since $t-1$ is only the end of layer-$\beta$ epochs for $\beta\leq \alpha_t$. As a result, the eigenvalue decomposition of $U_t$ can be computed in runtime $O(d_{t,\leq\alpha_t}^3)$. Note that by \cref{lemma:stability_eigenvalues_main_body}, with high-probability all eigenvalues of $U_t\in\Scal_d$ within $E_{t,\leq\alpha_t}$ are at least $\frac{3}{4}\mu_{\alpha_t}=\Omega(2^{-\alpha_t})$, which implies $d_{t,\leq\alpha_t}=O(\max(2^{\alpha_t},d))$. Together with the hierarchical epoch schedule, the previous arguments lead to the following estimate.
\begin{proposition}
\label{prop:computational_complexity}
    The total runtime of \cref{alg:final_alg} over $T$ rounds is~$O(d^2 T\log (ed))$.
\end{proposition}
The discussion of implementation details needed to achieve this runtime is deferred to~\cref{sec:computational_complexity}.

\subsection{Proof sketch of \cref{thm:ub-intro}}
\label{subsec:regret_proof_sketch}

The proof proceeds as follows: $(i)$ we derive structural properties on exploration subspaces including a crucial eigenvalue stability result, $(ii)$ bound the OMD regret with respect to the estimated gain matrices $\widehat G_t$---in particular, this requires carefully bounding stability terms arising from the variance of our estimators---and $(iii)$ relate the estimated gain matrices $\widehat G_t$ to the true gain matrices $G_t$.

\paragraph{Step 1: Well-posedness and eigenvalue stability.}
The exploration subspaces $(E_{t,\leq\alpha})_{\alpha\in[L]}$ are
useful only if they faithfully track the eigenspaces of $U_t$ with eigenvalues $\geq \mu_\alpha$. Writing $E_{t,\alpha} := E_{t,\leq\alpha} \cap E_{t,<\alpha}^\perp$, we first show that by construction of the layered estimator schedule, $U_t$ is block-diagonal along $(E_{t,\alpha})_{\alpha\in[L]}$ (Lemma~\ref{lemma:well_defined_algorithm},
claim~3). At the start of each layer-$(\alpha+1)$ epoch, when exploration subspaces are refreshed, the eigenvalues of $U_t$ within $E_{t,\alpha}$ lie in $[\mu_\alpha, 2\mu_\alpha)$ by definition of the subspace update. The following result shows they remain mostly stable throughout the epoch:

\begin{lemma}[Eigenvalue stability, simplified \cref{lemma:stability_eigenvalues}]
\label{lemma:stability_eigenvalues_main_body}
    Fix $\gamma\in(0,1/6]$, $L=\log_2(\tfrac{d}{\gamma})$, $\delta\in(0,1]$, and $\eta\lesssim 1/(rL^2\log^2(dLT/\delta))$. Suppose that for all $t\in[T]$, $\|G_t\|_{\op}\leq 1$ and $\|G_t\|_*\leq r$. With probability at least $1-\delta$, for any $t\in[T]$ and $\alpha\in[L]$, all eigenvalues of $U_t$ within $E_{t,\alpha}$ are in $[\tfrac{3}{4}\mu_\alpha,\tfrac{9}{4}\mu_\alpha]$.
\end{lemma}

This requires a delicate sensitivity analysis of the eigenvalue perturbations due to OMD updates.
In turn, this implies that with high-probability, the ``good scenario'' for the sampling scheme \cref{alg:layered_sampling} when residual matrices $R_t$ are all PSD holds with high probability (see \cref{lemma:R_not_negative}). We may assume that this event holds in the following discussion.

\paragraph{Step 2: OMD regret bound with respect to estimated gain matrices $\widehat G_t$.}
The standard analysis of mirror descent (e.g.~\cite{juditsky2011first}) bounds the regret of the OMD iterates $W_t$ in terms of an initial Bregman divergence term and the stability term $\sum_t \dotp{W_{t+1}-W_t}{\widehat G_t}$. We recall that $\widehat G_t$ is composed of a diagonal estimate $O_t$ and inverse-L-shaped estimates $B_{t,\alpha}$ for each layer. To decompose the analysis over each layer separately, we introduce fractional times and OMD iterates corresponding to observing the gain matrices $O_t,B_{t,1},\ldots,B_{t,L}$ in that order sequentially at times $t,t+\tfrac{1}{L+1},\ldots,t+\tfrac{L}{L+1}$.\footnote{It turns out that the OMD regret bound for these fractional times still bounds the desired original regret.} The resulting stability terms $\dotp{W_{t+\frac{\alpha+1}{L+1}}-W_{t+\frac{\alpha}{L+1}}}{B_{t,\alpha}}$ are then bounded using the following structural result, which exploits the block structure of both $W_t$ and $B_{t,\alpha}$.

\begin{restatable}[Stability term bound]{lemma}{lemmastabterm}
\label{lemma:updated_bound_stability_term}
Let the matrices 
\[
    W=\begin{bmatrix}
        W_{11}&0\\
        0&W_{22}\\
    \end{bmatrix}\in\Scal_d
    \quad \text{and} \quad 
    G=\begin{bmatrix}
        0 & G_{12}\\
        G_{12}^\top &G_{22}\\
    \end{bmatrix}
\]
satisfy $\|G\|_{\op}\leq \frac{1}{4\mu}$ and $W_{22}\preceq \mu I_{22}$.
Let $W_1=(W^{-1}+cI-G)^{-1}$ for~$c \in \R$ such that~$W_1\in\Scal_d$. Then
    \begin{equation*}
        \dotp{W_1-W}{G} \leq 11\mu\|G\|_{\op}^2.
    \end{equation*}
\end{restatable}

We apply this result to each layer-$\alpha$ stability term with $\mu=\frac{9}{4}\mu_\alpha$ using \cref{lemma:stability_eigenvalues_main_body} and $G=B_{t,\alpha}$ (up to restricting to $E_{t,\leq\alpha}\times E_{t,\leq\alpha}$). We then bound the operator norm $\|B_{t,\alpha}\|_{\op}\lesssim \|A_{t,\alpha}\|_{\op}$ using Wishart-style concentration bounds (see \cref{lemma:bounding_operator_norm}). This is precisely where we reap benefits of batching: the operator norm $\|A_{t,\alpha}\|_{\op}$ is not significantly increased after batching $O(d_{t,\leq\alpha})$ covariance terms of the form $w_tw_t^\top$ for uniform unit vectors $w_t\sim \mathrm{Unif}(S^{d-1}\cap E_{t,\leq\alpha})$.
Altogether we obtain:

\begin{lemma}[OMD regret bound, simplified \cref{lemma:compute_OMD_regret}]
\label{lemma:compute_OMD_regret_intro}
Under the same assumptions of \cref{lemma:stability_eigenvalues_main_body}, when running \cref{alg:final_alg}, for any $U\in\Scal_d$ we have for some universal constant $c>0$,
    \begin{equation*}
        \Ebb\sqb{\sum_{t=1}^T \dotp{U-W_t}{\widehat G_t}} \leq \frac{D_R(U\parallel I/d)}{\eta} + c \eta r^2 \log^4(dT/\delta) L T + c\frac{d^2T}{\gamma}\delta.
    \end{equation*}
\end{lemma}

\paragraph{Step 3: From batched estimates $\widehat{G}_t$ to true gains $G_t$.} 
The final step relates the left-hand side of Lemma~\ref{lemma:compute_OMD_regret_intro} to the regret of Algorithm~\ref{alg:final_alg} by replacing estimators $\widehat{G}_t$ with true gains $G_t$. By Lemma~\ref{lemma:unbiased_hat_L_final}, the diagonal estimator $O_t$ and the inverse-L-shaped estimators $B_{t,\alpha}$ are unbiased for the diagonal part of $G_t$ and the complementary inverse-L-shaped blocks $G_t^{(\alpha)}$ depicted in Figure~\ref{fig:layered_estimators}, respectively (see \cref{lemma:unbiased_hat_L_final}).

The main subtlety is that these estimators $B_{t,\alpha}$ are \emph{batched} by epochs within the final estimate $\widehat{G}_t$: for a layer-$\alpha$ epoch $[t_0, t_1]$, all estimators $B_{t,\alpha}$ for $t \in [t_0,t_1]$ are deferred to $\widehat{G}_{t_1}$. As a result, Lemma~\ref{lemma:compute_OMD_regret_intro} bounds regret as if $W_{t_1}$ were selected throughout the entire epoch. This is not the case since the OMD $W_t$ is updated at every round. The key observation is that within a layer-$\alpha$ epoch, $W_t$ is updated only on the finer block $E_{t,<\alpha} \times E_{t,<\alpha}$ (when finer-layer epochs end) and along the diagonal (due to the $O_t$ updates and projection step of the OMD step), leaving the off-diagonal inverse-L-shaped region of $W_t$ associated with $B_{t,\alpha}$ identically zero throughout the epoch. Consequently, the inner product $\langle W_t, B_{t,\alpha}\rangle = 0$ holds for all $t \in [t_0, t_1]$, which is what is needed to exchange $W_{t_1}$ for $W_t$ in the regret bound without error. In turn, up to an additive $\gamma T$ term arising from relating $W_t$ to $U_t$, this yields the same regret bound for \cref{alg:final_alg} as the right-hand side of Lemma~\ref{lemma:compute_OMD_regret_intro}. Optimizing over the stepsize parameter $\eta$ and choosing the mixing parameter $\gamma=1/T$ gives Theorem~\ref{thm:ub-intro}.

\section{Lower bound for Bandit PCA}
\label{sec:lb}

In this section, we describe the main ingredients in the proof of~\cref{thm:lb-intro}, our lower bound for rank-$r$ Bandit PCA. 
The cornerstone of the proof is an explicit construction of~\emph{adaptive adversary}, summarized in~\cref{sec:construction_adversary}. 
The adversary samples a hidden subspace~$E \sim \Unif(\Gr(p,d))$ at the start.
During the game, gain matrices are constructed in such a way that the learner is encouraged to discover sequentially orthogonal directions closely aligned with~$E$. The learner's queries are used to infer the already discovered directions; these are then penalized to encourage further discovery.

With this construction at hand, the first stage of the proof is showing that our adversarial strategy indeed works as intended: a low-regret Bandit PCA learner provably discovers all~$p$ directions in~$E$, and thus estimates~$E$ in a suitable sense. 
We then formulate the isolated problem of {\em subspace estimation from adaptive queries}, and prove a query lower bound for this problem. 
This is done via a posterior likelihood argument inspired from~\cite{chen2023does}; further details will be given in~\cref{sec:lb_proof-sketch}.

\subsection{Adaptive adversary construction}
\label{sec:construction_adversary}
Our construction (see~\cref{alg:adversary}) is as follows.
Initially, we sample a subspace~$E \sim \Unif(\Gr(p,d))$ and perturbation parameters~$a,b \sim \Unif([0,\alpha])$\footnote{Parameters~$a,b$ are randomized for some technical reasons, whose discussion is omitted here due to space limitation.} at a small constant scale~$\alpha\ll 1$.
Next, we construct 
\begin{equation}
\label{eq:hidden-matrix}
H := (1+a) P_{E^{\vphantom{\perp}}} - b\tfrac{p}{d} P_{E^\perp}, 
\end{equation}
a slight perturbation of the projector $P_E$, which is then used to generate the gain matrices. This matrix, as well the parameters~$E,a,b$ specifying the instance, are kept hidden from the learner.
During the game, we maintain a growing subspace~$W = W_j \subset \R^d$ of dimension~$j = j(t)$, which models the subspace explored by the learner so far.
In each round~$t \in [T]$, we operate in two stages: interaction with the learner  (lines~\ref{line:expected_gain}--\ref{line:issue-reward}); maintenance/update of the explored subspace~$W$ (lines~\ref{line:trigger-rule}--\ref{line:new-epoch-start}).

\begin{algorithm}[h]
\caption{Adaptive Adversary for Bandit PCA}
\label{alg:adversary}

\DontPrintSemicolon

\SetKwInOut{Input}{Input}

\Input{dimension~$d$, ranks~$r\leq p$, signal and dithering strengths~$\nu,\alpha$, max nb.~of epochs~$J_{\max}$}
    
    Sample $E \sim \Unif(\Gr(p,d))$, i.i.d.\ $a,b \sim \Unif([0,\alpha])$, and let~$H := (1+a)P_E^{\vphantom\perp}- b \frac{p}{d} P_{E^\perp}$ \label{line:adversary-data}
    
    Initialize~$j \leftarrow 0$, $\tau_0 := 0$, $W_0 := \{0\} \subset \R^{d}$, and $L_T := \log(eJ_{\max})+\log\log(eT)$ \label{line:adversary-init}
    
\For{$t=1,\ldots,T$}{

    $\smash{\overline G_t := (I+ \nu P_{W_j^\perp} H P_{W_j^\perp})\1[j<J_{\max}]}$ \label{line:expected_gain}

    
    $G_t := \textsc{GainMatrixSampling}(\overline G_t, r)$ \label{line:gain-sampling}
    
    Receive~$w_t$ and issue reward~$\ell_t^2 := \langle G_t, w_t w_t^\top \rangle$ \label{line:issue-reward}

    
    \If{$\frac{1}{t - \tau_j} \sum_{s=\tau_j+1}^{t} \ell_s^2
        \ge
        \frac{r}{d}
        (1+\nu(1-\alpha))+ C_{\mathrm{adv}}\frac{r}{p}\paren{\sqrt{\frac{L_T}{t-\tau_j}} + \frac{L_T}{t-\tau_j}}$ \textup{\bf and}   $j < J_{\max}$ \label{line:trigger-rule}}{ 
    

    $W_{j+1} := \Span(W_j,w_{\boldsymbol{i}})$ where $\boldsymbol{i} \sim \Unif(\{\tau_j+1,\dots,t\})$ \tcp*[r]{Randomized rounding} \label{line:rounding}
    
    $\tau_{j+1} := t$ and~$j \leftarrow j+1$ \tcp*[r]{New epoch start} \label{line:new-epoch-start}

    }
}
\end{algorithm}


\begin{algorithm}[h]
\caption{$\textsc{GainMatrixSampling}(\overline{G},r)$}
\label{alg:gain_sampling}

\SetKwInOut{Input}{Input}


\Input{expected matrix~$\overline{G} \in \mathbb{S}^d_+$, requested rank~$r \in [4,d/2]$}


Sample~$\beta \sim \BetaDist\left(1, \frac{r}{2}-1\right)$,~$\gamma \sim\GammaDist(\frac{d}{2},1)$ and~$V \sim\Unif(\Gr(r,d))$ independently

\Return $G=\frac{\beta\gamma r}{d} \overline{G}^{1/2} P_V \, \overline{G}^{1/2}$ 
\end{algorithm}

\paragraph{Interaction stage.}
Fixing~$\nu > 0$, we sample~$G_t \in \mathbb{S}^d_+$ of rank~$r$, with expectation proportional to
\[
\overline G_t := I + \nu P_{W^\perp} H P_{W^\perp}.
\]
This form of~$\overline G_t$ gives the learner an extra reward for aligning~$w_t$ along a direction in~$E \cap W^\perp$, cf.~\eqref{eq:hidden-matrix}, so that the already explored directions---those in~$W$---are {not} rewarded.
The sampling mechanism here, implemented in~\cref{alg:gain_sampling}, is such that, conditionally on the history in the previous rounds, each rank-one PSD marginal~$\langle G_t, w w^\top \rangle$ of~$G_t$ has the law
$
\frac{r}{d} \langle \overline G_t, w w^\top \rangle \Exp(1).
$
This ``exponential intensity measurement'' mechanism is instrumental in the simulation process discussed in Section~\ref{sec:lb_proof-sketch}.

\paragraph{Maintenance stage.}
In each round, we first test whether a ``discovery event'' has taken place, and perform maintenance only in that case.
The test (see line~\ref{line:trigger-rule}) consists of computing the running average of the rewards received since the last discovery event, and comparing it against a threshold chosen based on the structure of~$H$ and a confidence margin (to account for the randomness in~$G_t$). 
If the threshold is exceeded, we register a new discovery event in this round. 
We then update the explored subspace (line~\ref{line:rounding}) by expanding it in the ``typical direction of this epoch,'' sampled uniformly from the queries in the just completed epoch~$j$, i.e., in~the rounds after the previous discovery event.
The maintenance stage is omitted altogether once the number of discoveries has exceeded~$J_{\max} \asymp \alpha p$: at this point, the learner has discovered a constant fraction $\alpha$ of the $p$-dimensional subspace $E$, which already captures the main hardness of estimating $E$ (further details of this are given in~\cref{sec:lb_proof-sketch}).

The core property of the adversarial construction described above is that achieving low regret requires making all $J_{\max}$ discoveries with good probability. Intuitively, otherwise any comparator $u\in S^{d-1} \cap (E\cap W^\perp)$ has remained active throughout the game, and thus had consistently higher rewards than the learner. In practice, since we aim to bound the pseudo-regret, we instead use a randomized comparator $u\sim\mathrm{Unif}(S^{d-1}\cap E)$ which is therefore independent of the learner's actions.
Finally, we note that the adversary {\em does not hide} from the learner the ``exposed'' subspace~$W$. This fact is crucial to the simulation procedure which allows for the reduction in the proof (see~\cref{sec:lb_proof-sketch}).

\subsection{Proof sketch of Theorem~\ref{thm:lb-intro}}
\label{sec:lb_proof-sketch}


\paragraph{From Bandit PCA to subspace discovery.} Using the adversarial construction from \cref{alg:adversary}, we first reduce the Bandit PCA game to what we call the~$k$-\emph{subspace discovery problem}. 
Here, a learner makes adaptive queries to an intensity oracle $w\in\Rbb^d\mapsto \dotp{I+\nu H}{ww^\top}\mathrm{Exp}(1)$, with~$H$ given by~\eqref{eq:hidden-matrix}, and aims to recover the subspace~$E$. Crucially, the sampling mechanism of \cref{alg:gain_sampling} ensures that such oracle queries allow to \emph{simulate} the adversary in \cref{alg:adversary}. 
Using this simulation process, we then show that a Bandit PCA algorithm that has low regret when interacting with the adversary of Algorithm~\ref{alg:adversary} can be turned into a successful subspace discovery algorithm, namely one discovering~$k\approx J_{\max}\asymp \alpha p$ directions of $E$.

\paragraph{From subspace discovery to covariance estimation.} For convenience, we further reduce subspace discovery to the problem of \emph{covariance estimation with adaptive queries}. 
In the latter problem, the learner has access to the same intensity oracle but aims to estimate~$E$, thereby estimating the covariance matrix $I+\nu H$. The estimation performance quantified by the operator-norm discrepancy between the true and estimated subspace projectors.
Intuitively, up to a symmetrization argument, an algorithm solving the $k$-subspace discovery problem finds a uniformly-random $k$-dimensional subspace of $E$. 
Hence, making~$\tilde O(\frac{\dim(E)}{k})=\tilde O(\frac{1}{\alpha})$ repeated runs of the subspace discovery algorithm allows to estimate~$E$.

\paragraph{Query lower bound for the covariance estimation problem.} The last step of the proof is to show a query lower bound for the covariance estimation problem using a posterior likelihood argument inspired from \cite{chen2023does}. 
In a nutshell, we follow the posterior distribution on $E$ after $m$ queries and show that its mass within a small Grassmannian ball $B(E,\beta)$ is significantly smaller than that on the dilated ball $B(E,2\beta)$, thereby preventing the algorithm from estimating $E$ up to accuracy $\beta$. Before making any queries, the log-volume ratio between the two balls is~$\log \left(\frac{\mathrm{Vol}(B(E,2\beta))}{\mathrm{Vol}(B(E,\beta))}\right) \asymp p(d-p)$, which corresponds to the effective dimension of the Grassmannian~$\Gr(p,d)$. 
The crux of the argument is then to upper-bound how much information one gets from  a query: we show that each query reduces the log-volume ratio by~$O(\frac{\beta^2\nu^2}{p\max(1,\nu)^2})$ (see \cref{lemma:tighter_computations}). 
Altogether, this gives a $\Omega(\frac{p^2d\max(1,\nu)^2}{\beta^2\nu^2})$ query lower bound for estimating $E$ with accuracy $\beta$.

When the two reductions and the covariance estimation query lower bound are combined together, we are left with the scalings
\[
\nu\asymp \min\left(\sqrt{\frac{d^3}{T\log d}},\frac{d}{r}\right), \quad  p \approx\frac{d}{\max(\nu,2)}
\]
for the signal strength and subspace dimension parameters. This leads to the bound in~\cref{thm:lb-intro}.

\section{Conclusion and perspectives}
\label{sec:conclusion}

In this paper, we study the minimax regret for Bandit PCA with gain matrices of rank~$r$, establishing it to be of order~$r\sqrt{dT}$ up to logarithmic factors.
The upper bound is attained by an efficient algorithm \textsc{OnlineBanditPCA}, which nontrivially modifies the OMD-based approach of~\cite{kotlowski2019bandit} by adding a multiscale exploration scheme; this scheme allows to update the eigenspaces of different spectral magnitudes should be updated at different rates. 
To prove the lower bound, we construct an adaptive adversary which forces the learner to solve a subspace estimation problem with adaptive queries. 
This effect is attained by using the transcript to infer which directions have already been ``exposed'' by the learner, and penalizing for these directions in the future; thus, adaptivity is crucial.
In conclusion, let us discuss some future directions and open problems.

\paragraph{Closing the polylogarithmic gap.}
The~$\log(dT)^3$ overhead in the upper bound of Theorem~\ref{thm:ub-intro} has three identifiable sources: the~$\log\det$ divergence term $d\log(1/\gamma)$, the $\log^2(dT)$ operator-norm concentration of the epoch estimators $A_{t,\alpha}$, and the layering factor~$L = O(\log(dT))$; each of these might be improvable. 
On the other hand, the factor~$\log d$ in Theorem~\ref{thm:lb-intro} originates from the repeated estimation (``boosting'') step in the subspace-discovery-to-covariance-estimation reduction, and it is unclear how to avoid it without directly proving a query lower bound on the subspace discovery problem. 
Meanwhile, we suspect that the~$\log(eT)$ factor in Theorem~\ref{thm:lb-intro} is removable, though it would require a nontrivial modification of our approach. 
Indeed, we have already managed to remove one potential occurrence of this factor, by employing a martingale-concentration analysis of the epoch triggering rule, instead of simply taking the union bound (see Step 2 in the proof of Proposition~\ref{prop:regret-discovery}).
The remaining bottleneck comes from final rescaling of the instance, accounting for the unbounded (by exponentially-concentrated) operator norms of the gain matrices. 
This, in turn, originates from the exponential sampling technique in Algorithm~\ref{alg:gain_sampling}. The exponentially-distributed dilation, while crucial in the proof of Proposition~\ref{prop:regret-discovery}, does not conform to the boundedness assumption on the gains; 
yet, our attempts to prove an analogue of Proposition~\ref{prop:regret-discovery} with other distributions were unsuccessful.

\paragraph{Improving the computational complexity.} In our proposed algorithm, intuitively, the main computational bottleneck is the eigenvalue decomposition of the OMD iterate $U_t$ (see line~\ref{line:eigenval_decompose_U} in \cref{alg:subspace_update}). While a naive implementation would have $O(d^3 T)$ complexity, thanks to the epoch-layered structure of our algorithm, this can be reduced to $O(d^2 T\log d)$ complexity as shown in \cref{prop:computational_complexity}. While we believe that the extra factor $\log d$ could be removed using a more involved implementation (see \cref{remark:computational_complexity}), it is unclear whether the runtime can be significantly improved---for instance depending on the rank $r$ of gain matrices---while preserving tight regret bounds. In comparison, we recall that the sparse sampling algorithm from \cite{kotlowski2019bandit} has $O(dT)$ runtime at the price of a suboptimal regret $\tilde O(d\sqrt{rT})$.

\comment{
A naive implementation of our algorithm therefore has $O(d^3 T)$ complexity. This can however be improved to $O(d^2 T)$ expected complexity by noting that at each round, we only need perform an eigenvalue decomposition of $U_t$ along the block $E_{t,\leq\alpha_t}\times E_{t,\leq\alpha_t}$, which is where eigenvectors may have rotated. Note that under the ``good event'' where eigenvalues remain stable (see \cref{lemma:stability_eigenvalues_main_body}), a subspace $E_{t,\leq\alpha}$ may have dimension at most $O(1/\mu_\alpha)=O(2^\alpha)$ since $U_t\in\Scal_d$. And hence, computing the eigenvalue decomposition of $U_t$ on $E_{t,\leq\alpha_t}\times E_{t,\leq\alpha_t}$ has complexity $O(\min(d,2^{\alpha_t})^3)$. Together with the epoch hierarchical schedule, this yields a total $O(d^2 T)$ expected complexity, where the expectation results from the low-probability ``bad event'' of \cref{lemma:stability_eigenvalues_main_body}. 
}

\paragraph{Extension to the $k$-PCA setting.} A natural generalization of our setting is the \emph{$k$}-PCA problem where, instead of selecting a rank-one projector $w_tw_t^\top$ at each round, the learner selects a rank-$k$ projector. 
The full-information formulation of this problem has been extensively studied~\cite{nie2016online}; meanwhile, for the bandit formulation to the best of our knowledge we are only aware of reductions to linear bandits \cite{dani2007price,hazan2010learning,hoeven2018many} which yield suboptimal regret.
The layered structure introduced in this work may extend to this setting, potentially improving the existing regret bounds.

\paragraph{Further connections with quantum tomography.}
In addition to the interpretation of rank-1 Bandit PCA as quantum tomography of an evolving pure state given in~\cref{sec:intro}, one could draw further connections between Bandit PCA and quantum state tomography with adaptive measurements~\cite{chen2023does}, on a technical level. In particular, the general idea of piecewise-flat spectral approximation of a quantum state or operator (i.e., grouping eigenstates with close eigenvalues together), which is at the heart of our layering exploration scheme in~\cref{alg:final_alg}, was also instrumental in obtaining the sharp rate of quantum tomography in terms of trace norm~\cite{chen2023does}. 

\paragraph{Connection with quantum state estimation.}
Another possibility is to interpret Bandit PCA as an online quantum state estimation (QSE) problem, see e.g.~\cite{zimmert2022pushing,tseng2025online} and references therein. 
Here, the best comparator~$w^\star$---the top eigenvector of the cumulative gain~$G_{1:T}$---is interpreted as the estimated state, and~$G_t$ as linear measurements performed on a sequence of adaptively selected states~$w_1, \dots w_T$. Note that the purity of~$w^\star$ here is superficial: due to the linearity of~$\langle \cdot , G \rangle$, the maximum in~\eqref{eq:regret} can be taken over all~$W \in \Scal_d$ instead of~$ww^\top$, yielding a pure-state maximizer automatically.
Thus, Bandit PCA corresponds to online QSE with linear losses and bandit feedback. 
One can then ask, e.g., what happens if the learner can play arbitrary states~$W_t \in \Scal_d$ in each round.

\newpage
\appendix

\section{Proof of the regret upper bound}
\label{app:ub}

In this section, we prove the regret upper bound on \cref{alg:final_alg}, which implies \cref{thm:ub-intro}.

\begin{theorem}[Regret bound for \cref{alg:final_alg}]
\label{thm:ub}
    For some absolute constant $c_0>0$, the following holds. 
    Fix~$T\geq 1$, $\gamma\in(0,\frac{1}{6})$, $L=\ceil{\log_2(d/\gamma)}$, and~$\eta\leq \frac{c_0}{r \log^4(dT/\gamma)}$. 
    Suppose that $\|G_t\|_{\op}\leq 1$ and $\|G_t\|_*\leq r$ for all $t\in[T]$. Then,
    \begin{equation*}
        \Reg_T(\cref{alg:final_alg}) \lesssim  \frac{d}{\eta}\log(1/\gamma) + \eta r^2 \log^5(dT/\gamma) T + \gamma T.
    \end{equation*}
    In particular, taking~$\gamma=T^{-1}$ and $\eta=\frac{1}{r\log^2 (dT)}\sqrt{\frac{d}{T}}$ results in
    \begin{equation*}
        \Reg_T(\cref{alg:final_alg}) \lesssim r\sqrt{dT}\log^3(dT).
    \end{equation*}
\end{theorem}

We start by checking that the algorithm is well-defined. Note that $\alpha_1=L$ hence initializing $E_{0,\leq L}$ is sufficient to run the algorithm. We also check that line~\ref{line:sample_w_to_complete_R} of the sampling scheme subroutine (\cref{alg:layered_sampling}) is well-defined below:

\begin{lemma}\label{lemma:check_sampling_valid}
    For any $t\in[T]$, if $R_t\succeq 0$, then either (1) $R_t=0$ in which case almost surely $I_t\in[L]$, or (2) $R_t\succ 0$ in which case $\sum_{\alpha\in[L]}p_{t,\alpha}<1$ and $\|\frac{R_t}{p_{t,0}}\|_*=1$; as a result we can indeed generate $w_t$ to satisfy line~\ref{line:sample_w_to_complete_R} of \cref{alg:layered_sampling} since $\frac{R_t}{p_{t,0}}\in\Scal_d$.
\end{lemma}
\begin{proof}
    Note that
    \begin{equation*}
        \tr(R_t)= \tr(U_t) - \sum_{\alpha\in[L],d_{t,\leq\alpha}>0} p_{t,\alpha} \frac{\tr(P_{E_{t,\leq\alpha}})}{d_{t,\leq\alpha}} = 1-\sum_{\alpha\in[L]} p_{t,\alpha} = p_{t,0}.
    \end{equation*}
    This proves the desired result when $R_t\succ 0$.
    If $R_t=0$ this gives $\sum_{\alpha\in[L]}p_{t,\alpha}=1$ and as a result, almost surely $I_t\in[L]$.
\end{proof}

We divide the proof of \cref{thm:ub} into four sections: $(i)$ in \cref{subsec:lemma_estimators} we prove simple properties on the form of our gain matrix estimators, $(ii)$ in \cref{subsec:eigenval_stability} we prove eigenvalue stability properties for exploration subspaces, $(iii)$ in \cref{subsec:OMD_regret_bound} we derive a regret upper bound for the OMD iterates with respect to estimated gain matrices $\widehat G_t$, and $(iv)$ conclude in \cref{subsec:ubthm} by relating this upper bound to the regret of \cref{alg:final_alg}. Throughout the proof, we denote by $\Hcal_t$ the history up to round~$t$ included.

\subsection{General properties of gain matrix estimators}
\label{subsec:lemma_estimators}

We start with the following lemma which gives an unbiaised estimator of a gain matrix $L$ when having access to the observation $\dotp{L}{ww^\top}$ for isotropic vectors $w$. This result known (e.g., see \cite[Section 4.7.1, Example (3)]{martinsson2020randomized}); we include a short proof for completeness.

\begin{lemma}[Sphere estimator unbiasedness]
\label{lemma:unbiased_hat_L_final}
Let $L$ be a symmetric matrix, $w\sim\Ncal(0,I/d)$ and $\ell^2:=\dotp{L}{w w^\top}$. Then,
    \begin{equation*}
        \Ebb\sqb{d^2\ell^2(w w^\top - I/d)}=2L.
    \end{equation*}
    Also, if $w\sim\mathrm{Unif}(S^{d-1})$ and $\ell^2 = \dotp{L}{ww^\top}$, then
    \begin{equation*}
        \Ebb\sqb{d\ell^2((d+2)ww^\top - I)} = 2L.
    \end{equation*}
\end{lemma}

\begin{proof}
    Let $L=O\diag(\lambda_1,\ldots,\lambda_d)O^\top$ be the SVD of $L$. Denote by $u_1,\ldots,u_d$ the columns of $O$.
    Since $\Ncal(0,I/d)$ is isometric, without loss of generality, we may generate $w=\sum_{i\in[d]}Y_iu_i$ where $Y_i\overset{iid}{\sim}\Ncal(0,1/d)$ are sampled independently. Then, $\ell^2=\sum_{i\in[d]}\lambda_i Y_i^2$. Next,
    \begin{equation*}
        w w^\top -\frac{I}{d} = \sum_{i\in[d]} \paren{Y_i^2-\frac{1}{d}}u_iu_i^\top + \sum_{i<j}Y_iY_j(u_iu_j^\top+u_ju_i^\top) 
    \end{equation*}
    Noting that the variables $Y_i$ are all symmetric and independent, with $\Ebb[Y_i^2]=1/d$, we obtain
    \begin{equation*}
        \Ebb\sqb{d^2\ell^2\paren{w w^\top - \frac{I}{d}}} = \sum_{i\in[d]} \Ebb\sqb{d^2\ell^2\paren{Y_i^2-\frac{1}{d}}} u_iu_i^\top =  \sum_{i\in[d]} \lambda_i\Ebb\sqb{(dY_i^2-1)^2} u_iu_i^\top = 2\sum_{i\in[d]}\lambda_iu_iu_i^\top.
    \end{equation*}
    In the last equality we noted that $dY_i^2\sim\chi^2(1)$, which has variance $2$. This ends the proof of the first claim. 
    
    We next turn to the second claim. note that if $w_0\sim\Ncal(0,I/d)$ and we denote $w=w_0/\|w_0\|$, then $w\sim\mathrm{Unif}(S^{d-1})$ and $d\|w_0\|^2\sim \chi^2(d)$ is independent of $w$. As a result, denoting $\ell_0^2:=\dotp{L}{w_0w_0^\top}$ and $\ell^2 = \dotp{L}{ww^\top}$, one has
    \begin{equation*}
        2L \overset{(i)}{=} \Ebb\sqb{d^2\ell_0^2(w_0 w_0^\top - I/d)} = \Ebb\sqb{\ell^2(d^2\|w_0\|^4w w^\top - d\|w_0\|^2 I)}= \Ebb\sqb{\ell^2(d(d+2)w w^\top - d I)}.
    \end{equation*}
    In $(i)$ we used the previous identity. This ends the proof.
\end{proof}

This motivates the choice of our gain matrices estimators.
We will also use the following result on the operator norm of such estimators when batched over a few iterations.

\begin{lemma}\label{lemma:operator_norm_by_phase_final}
    Consider a random sequence of $d\times d$ symmetric matrices and $d$-dimensional vectors $(M_l,w_l)_{l\geq 1}$ with $\|M_l\|_*\leq C$ and $\|M_l\|_{\op}\leq 1$ and $w_l\mid \Hcal_l\sim \mathrm{Unif}(S^{d-1})$ for all $l\geq 1$, where $\Hcal_l=\sigma(M_{a}, a\leq l,w_{l'},l'<l)$ is the history. Let $L\in\Nbb$ be a stopping time for the filtration $(\Hcal_l)_{l\geq 1}$ and define
    \begin{equation*}
        A:=d\sum_{l\in[L]} \dotp{M_l}{w_lw_l^\top} ((d+2)w_lw_l^\top-I).
    \end{equation*}
    Then, for some universal constant $c_0>0$, with probability at least $1-\delta$, we have
    \begin{equation*}
        \|A\|_{\op}^2 \leq c_0C\paren{(d+L)\log\frac{L+1}{\delta} +\log^2\frac{1}{\delta}}
    \end{equation*}
\end{lemma}

\begin{proof}
    For convenience, we denote $g_l:=\dotp{M_l}{w_lw_l^\top}$. We have $|g_l|\leq \tilde g_l:=\dotp{|M_l|}{w_lw_l^\top}$, where $|M_l|$ denotes the same matrix as $|M_l|$ but where we took absolute values on each eigenvalue. We recall that $w_l$ is a uniform unit vector in $S^{d-1}$. In particular, $\Ebb[d\tilde g_l]\leq \|M_l\|_*\leq C$. Hence, from \cite[Theorem 5.1.3]{vershynin2018high}, for any $1$-Lipschitz function $\phi:\Rbb^d\to\Rbb$, $\phi(\sqrt d w_l)$ is subGaussian for some universal parameter $c'>0$. Hence, we can apply a version of Hanson-Wright's inequality for isotropic vectors \cite[Theorem 2.5]{adamczak2015note} which implies that for some universal constant $\tilde c>0$, 
    \begin{equation*}
        \Pbb[d\tilde g_l - C \geq tC] \leq 2e^{-c\min(t^2,t)},
    \end{equation*}
    where we used $\|M_l\|_{\op}\leq \min(1,C)$, $\|M_l\|_F^2 \leq \|M_l\|_{\op}\|M_l\|_*\leq C^2$, and $\mathrm{Cov}(\sqrt d w_l)=I$.
    Hence, the union bound shows that on an event $\Ecal_\delta$ of probability at least $1-\delta/2$, simultaneously for all $k\geq 1$ one has
    \begin{equation*}
        \max_{l\in[k]}(d\tilde g_l) \leq c_0 C\log ((k+1)/\delta),\quad k\geq 1,
    \end{equation*}
    for some universal constant $c_0>0$.
    
    Next, for each $k\geq 1$ we denote $B_k:=\sum_{l\in[k]} dw_lw_l^\top$. Since $\sqrt d w_l$ is an isotropic subgaussian vector for some universal constant parameter \cite[Theorem 3.4.5]{vershynin2018high}, a standard concentration bounds on subgaussian matrices \cite[Theorem 4.6.1]{vershynin2018high} implies that for some universal constant $c'>0$,
    \begin{equation*}
        \Pbb\sqb{\|B_k\|_{\op} \geq c'(\sqrt d+\sqrt k+t)^2} \leq \Pbb\sqb{\|\sqrt d[w_1,\ldots,w_k]\|_{\op} \geq \sqrt{c'}(\sqrt d+\sqrt k+t)} \leq 2e^{-t^2},\quad t\geq 0.
    \end{equation*}
    Hence, on an event $\Fcal_\delta$ of probability at least $1-\delta/2$, simultaneously for all $k\geq 1$ we have
    \begin{equation*}
        \|B_k\|_{\op}\leq c_1(d+k+\log ((k+1)/\delta)),\quad k\geq 1,
    \end{equation*}
    for some universal constant $c_1\geq 1$.

    We can now bound the operator norm of $A$. 
    Under $\Ecal_\delta\cap\Fcal_\delta$ which has prob.~at least $1-\delta$, one has
    \begin{align*}
        \|A\|_{\op} &\leq d\abs{\sum_{l\in[L]}g_l} + d(d+2)\norm{\sum_{l\in[L]} g_l w_lw_l^\top}_{op} \\
        &\leq 3\max_{l\in[L]}(d\tilde g_l)\cdot (L+ \|B_L\|_{\op})\\
        &\leq 3c_0 c_1 C(d+2L+\log ((L+1)/\delta))\log ((L+1)/\delta)\\
        &\leq c_2 C((d+L)\log((L+1)/\delta)+\log^2(1/\delta)),
    \end{align*}
    for some universal constant $c_2>0$. This ends the proof.
\end{proof}

\subsection{Properties of exploration subspaces}
\label{subsec:eigenval_stability}

We start by proving the following structural result on the layered subspaces defined throughout \cref{alg:final_alg} (line \ref{line:update_subspaces}).

\begin{lemma}\label{lemma:well_defined_algorithm}
    The following claims on the subspaces defined by \cref{alg:final_alg} hold:
    \begin{enumerate}
        \item For any $t\geq 1$ one has $E_{t,\leq 1}\subseteq\ldots\subseteq E_{t,\leq L}=\Rbb^d$.
        \item For any layer $\alpha\in[L]$, fix any beginning of a layer-$\alpha$ epoch $t_b=1\bmod 2^\alpha$, and the corresponding end of epoch $t_e=t_b-1+2^\alpha$. Then, the layer-$(\geq\alpha-1)$ subspaces are unchanged throughout the epoch:
        \begin{equation*}
            E_{t,\leq\beta} = E_{t_b,\leq\beta},\quad t\in[t_b,t_e],\beta\geq\alpha-1.
        \end{equation*}
        In particular, for all $t\in[t_b,t_e]$, we have $E_{t,\alpha}=E_{t_b,\alpha}$ as well.
        \item For any $t\geq 1$, $W_t$ and $U_t$ are diagonal by block along the eigenspaces $E_{t,\alpha}$ for all $\alpha\in[L]$. In particular, $J_{t,\alpha}$ in the construction of $B_{t,\alpha}$ is well-defined.
        \item For any $t\geq 1$, and let $t_0:=t-2^{\alpha_t}$ be the beginning of the previous layer-$\alpha_t$ epoch. Then, for any $i\in[d]$ if $u_{t_0,i}\in E_{t_0,\leq\alpha_t}^\perp$ then $u_{t_0,i}$ is also an eigenvector of $U_s$ for all $s\in[t_0,t]$.
    \end{enumerate}
    
\end{lemma}

\begin{proof}
    \textbf{Claim 2.} 
    Note that by construction, $E_{t-1,\leq\alpha}$ is not updated unless $\alpha_t\geq\alpha+1$, or equivalently $t=1\bmod 2^{\alpha+1}$ is the beginning of a layer-$(\alpha+1)$ epoch. 

    \noindent\textbf{Claim 1.} 
    Hence, to prove the first claim by induction it suffices to check that at each iteration $t$ one has $E_{t,\leq 1}\subseteq\ldots\subseteq E_{t,\leq\alpha_t-1}$ which is immediate by construction since the eigenvector subspaces $F_{t,\leq\alpha}$ are non-decreasing with $\alpha$.

    Next, from the second claim, $E_{t,\leq L}=E_{0,\leq L}=\Rbb^d$ is never updated. This ends the proof of Claim 1.

    \noindent\textbf{Claim 3.} 
    We now prove the third claim by induction. Note that $\widehat G_{1:t-1}$, $W_t$, $U_t$ always share the same eigenvectors, hence we may focus on say $\widehat G_{1:t-1}$ at time $t$. At $t=1$, note that $\widehat G_{1:0}=0$ and as a result $W_t=I/d$, that is, all eigenvalues are equal. Since $\alpha_1=L$ and $E_{0,\leq L}$, we can check that one of the $E_{t,\alpha}$ is full $\Rbb^d$ while all others are empty $\{0\}$. These subspaces trivially diagonalize $\widehat G_{1:0}$, which ends the initialization.

    Next, fix $t>1$ and define $t_0:=t-2^{\alpha_t}$, which is the beginning of the previous layer-$\alpha_t$ epoch. From the second claim, we have $E_{t,\alpha}=E_{t_0,\alpha}$ for layers $\alpha> \alpha_t$, and $E_{t,\leq\alpha_t}=E_{t_0,\leq\alpha_t}$. Hence, using the induction, $\widehat G_{1:t_0-1}$ is diagonalizable along the subspaces $E_{t,\leq\alpha_t},E_{t,\alpha_t+1},\ldots,E_{t,L}$. 
    
    Note that since $[t_0,t-1]$ is a layer-$\alpha_t$ epoch, we have $B_{s,\alpha}=A_{s,\alpha}=0$ for all $\alpha>\alpha_t$ and $s\in[t_0,t-1]$. Also, recalling that $E_{t,\leq\alpha_t}$ is constant on this epoch, this implies that
    \begin{equation}\label{eq:affect_only_previous_subspaces}
        \tilde B_s =P_{E_{t_0,\leq\alpha_t}} \tilde B_s P_{E_{t_0,\leq\alpha_t}} \quad \text{where}\quad \tilde B_s:=\sum_{\alpha\in[L]} B_{s,\alpha},\quad s\in[t_0,t-1].
    \end{equation}

    Next, denote by $J_{t_0}$ the indices of the eigenvectors $u_{t_0,1},\ldots,u_{t_0,d}$ of $\widehat G_{1:t_0-1}$ (equivalently $U_{t_0}$) lying in any of the subspaces $E_{t_0,\alpha_t+1},\ldots,E_{t_0,L}$, equivalently, in $E_{t_0,\leq\alpha_t}^\perp$. We check by induction that these vectors are all eigenvectors of $\widehat G_{1:s-1}$ (equivalently $U_s$) for any time $s\in[t_0,t]$, and hence, are included within $u_{s,1},\ldots,u_{s,d}$ when we eigendecompose $U_s$ in \cref{alg:final_alg} (see line \ref{line:eigenval_decompose_U}). This is immediate for $s=t_0$. Now suppose this holds for $s\in[t_0,t-1]$. We already checked in \cref{eq:affect_only_previous_subspaces} that the terms $B_{s,\alpha}$ are restricted to $E_{t_0,\leq\alpha_t}$. Hence, it suffices to focus on the terms $O_s$. Since by induction hypothesis, the vectors $u_{s,1},\ldots,u_{s,d}$ correspond to the vectors $u_{t_0,j}$ for $j\in[J_{t_0}]$ and complementary vectors in $E_{t_0,\leq \alpha_t}$, we can always decompose:
    \begin{equation*}
        O_s = O_{s,\leq\alpha_t} + \sum_{j\in [J_{t_0}]} y_{s,j} u_{t_0,j}u_{t_0,j}^\top, \quad \text{where} \quad O_{s,\leq\alpha_t} = P_{E_{t_0,\leq\alpha_t}} O_{s,\leq\alpha_t} P_{E_{t_0,\leq\alpha_t}},
    \end{equation*}
    for some scalars $y_{s,j}\in\Rbb$ and matrix $O_{s,\leq\alpha_t}$.
    Summing all previous equations together, we can write
    \begin{equation*}
        \widehat G_{1:s} = \widehat G_{1:s-1}+O_s+\tilde B_s  = \widehat G_{1:s,\leq\alpha_t} + \sum_{j\in [J_{t_0}]} z_{s,j} u_{t_0,j}u_{t_0,j}^\top,\quad \text{where} \quad \widehat G_{1:s,\leq\alpha_t} = P_{E_{t_0,\leq\alpha_t}} \widehat G_{1:s,\leq\alpha_t} P_{E_{t_0,\leq\alpha_t}},
    \end{equation*}
    for some scalars $z_{s,j}\in\Rbb$ and matrix $\widehat G_{1:s,\leq\alpha_t}$. Recall that by construction, the vectors $u_{t_0,j}$ for $j\in[J_{t_0}]$ are perpendicular to $E_{t_0,\leq \alpha_t}$. Altogether, this shows that the vectors $u_{t_0,j}$ for $j\in[J_{t_0}]$ are still eigenvectors of $\widehat G_{1:s}$.

    In summary, we showed that $\widehat G_{1:t-1}$ (and hence $U_t$) is diagonalizable along the subspace $E_{t_0,\leq\alpha_t}$ and vectors $u_{t_0,j}$ for $j\in[J_{t_0}]$. In particular, $\widehat G_{1:t-1}$ (equivalently $U_t$) is diagonal by block along $E_{t_0,\leq \alpha_t},E_{t_0,\alpha_t+1},\ldots,E_{t_0,L}$, which we proved above are equal to $E_{t,\leq\alpha_t},E_{t,\alpha_t+1},\ldots,E_{t,L}$. It only remains to check that along $E_{t,\leq\alpha_t}$, $U_t$ is diagonal by block on $E_{t,1},\ldots,E_{t,\alpha_t}$. This is immediate from their definition since for all $\alpha\in[L]$, $F_{t,\leq\alpha}$ is an eigenspace of $U_t$ (see line~\ref{line:define_F} of \cref{alg:final_alg}): given that we proved $E_{t,\leq\alpha_t}$ is also an eigenspace of $U_t$, so is their intersection. Altogether, this proves that $U_t$ is diagonal by block along $E_{t,1},\ldots,E_{t,L}$.

    \textbf{Claim 4.} This was proved in the previous paragraph while proving Claim 3. 
\end{proof}

Using this structural result and the operator norm bound on the general form of matrix estimators in~\cref{lemma:operator_norm_by_phase_final}, we now bound the norm of estimators $A_{t,\alpha}$ at the end of each layer-$\alpha$ epoch, $\alpha\in[L]$.

\begin{lemma}[Operator norm of epoch estimators]
\label{lemma:bounding_operator_norm}
Fix a layer $\alpha\in[L]$, the beginning of a layer-$\alpha$ epoch $t_b=1\bmod 2^\alpha$, and the corresponding end of epoch $t_e=t_b-1+2^\alpha$. Suppose that with probability one, for all $t\in[t_b,t_e]$ one has $\|G_t\|_{\op}\leq 1$ and $\|G_t\|_*\leq r$. Then,
    \begin{equation*}
        \Pbb\sqb{\|A_{t_e,\alpha}\|_{\op} \geq c_0\frac{r\log(d_{t_b,\leq\alpha}/\delta)}{\mu_\alpha}  \paren{1 + \frac{\log(1/\delta)}{d_{t_b,\leq\alpha}} } \mid\Hcal_{t_b-1}} \leq \delta.
    \end{equation*}
\end{lemma}

\begin{proof}
    Fix $\alpha\in[L]$ and a layer-$\alpha$ epoch $[t_b,t_e]$. From \cref{lemma:well_defined_algorithm}, for all times in the epoch $t\in[t_b,t_e]$ one has $E_{t,\leq\alpha}=E_{t_b,\leq\alpha}$. Hence, if $Z_t=1$, $R_t\succeq 0$, and $I_t=\alpha$, conditionally on the past history we always have
    \begin{equation*}
        (w_t\mid \Hcal_t,Z_t=1,R_t\succeq 0,I_t=\alpha) \sim \mathrm{Unif}(S^{d-1}\cap E_{t_b,\leq\alpha}),
    \end{equation*}
    see the definition of $w_t$ in \cref{alg:final_alg} when $Z_t=1$, $R_t\succeq 0$, and $I_t\in[L]$. 
    Then, we can apply \cref{lemma:operator_norm_by_phase_final} restricted to $E_{t_b,\leq\alpha}$, which shows that with probability at least $1-\delta$, one has
    \begin{align*}
        \|A_{t_e,\alpha}\|_{\op} &\leq \frac{c_0}{d_{t_b,\leq\alpha}\mu_\alpha} \paren{(d_{t_b,\leq\alpha}+N)\log\frac{N+1}{\delta}  + \log^2\frac{1}{\delta}} \max_{t\in[t_b,t_e]} \|P_{E_{t_b,\leq\alpha}}G_tP_{E_{t_b,\leq\alpha}}\|_*\\
        &\leq \frac{c_0 r}{d_{t_b,\leq\alpha}\mu_\alpha} \paren{(d_{t_b,\leq\alpha}+N)\log\frac{N+1}{\delta}  + \log^2\frac{1}{\delta}},
    \end{align*}
    for some universal constant $c_0>0$, where
    \begin{equation*}
        N=\sum_{t=t_b}^{t_e} \1[Z_t=1,R_t\succeq 0,I_t=\alpha].
    \end{equation*}
    Now note that $N$ is stochastically dominated by a binomial $\mathrm{Binom}(t_e-t_b+1,p_{t_b,\alpha}/2)$ which has mean $2^\alpha p_{t_b,\alpha}/2=d_{t_b,\leq\alpha}/8$, since conditionally on the history $\1[Z_t=1,I_t=\alpha]\sim\mathrm{Ber}(p_{t_b,\alpha}/2)$. Then, by Chernoff's bound, conditionally on $\Hcal_{t_b-1}$ and with probability at least $1-\delta$ one has
    \begin{equation*}
        N \leq \frac{d_{t_b,\leq\alpha}}{4} + c_1\sqrt{d_{t_b,\leq\alpha}\log(1/\delta) } + c_1 \log(1/\delta),
    \end{equation*}
    for some $c_1>0$.
    Altogether, this shows that with probability at least $1-2\delta$ conditionally on $\Hcal_{t_b-1}$,
    \begin{equation*}
        \|A_{t_e,\alpha}\|_{\op} \leq \frac{c_2 r \log(d_{t_b,\leq\alpha}/\delta)}{\mu_\alpha} \paren{1 + \frac{\log(1/\delta)}{d_{t_b,\leq\alpha}} } ,
    \end{equation*}
    for some universal $c_2>0$. Up to changing $\delta$ by a factor $2$, this ends the proof.
\end{proof}

The previous result lemma will be useful to show that the eigenvalues remain stable within the eigenspaces $E_{t,1},\ldots,E_{t,L}$ as detailed below. 
For the sake of the proof, it will be useful to define the following quantities: for any $t\in[T]$ and $\alpha\in[L]$, we introduce
    \begin{equation*}
        W_{t+\frac{\alpha}{L+1}} := (c_{t+\frac{\alpha}{L+1}} I - \eta (\widehat G_{1:t-1} + O_t +\sum_{\beta<\alpha} B_{t,\beta}))^{-1},
    \end{equation*}
    where $c_{t+\frac{\alpha}{L+1}}$ is chosen so that $W_{t+\frac{\alpha}{L+1}}\in\Scal_d$. In other terms, $W_t,W_{t+\frac{1}{L+1}},\ldots,W_{t+\frac{L}{L+1}}, W_{t+1}$ correspond to the OMD iterates for sequentially observing the gain matrices $O_t,B_{t,1},\ldots,B_{t,L}$---recall that by definition $\widehat G_t=O_t+\sum_{\alpha\in[L]}B_{t,\alpha}$. If $u_{t,i}$ for $i\in[d]$ is an eigenvector of $W_{t+\frac{\alpha}{L+1}}$, we denote by $\lambda_{t+\frac{\alpha}{L+1},i}$ its corresponding eigenvalue.

    The next lemma proves stability of eigenvalues within each exploration subspace, including on fractional times. This will use a technical sensitivity analysis result, which we defer to \cref{lemma:stability_one_step}.

\begin{lemma}\label{lemma:stability_eigenvalues}
    Fix $\gamma\in(0,1/6]$ such that $L=\log_2(d/\gamma)$ is an integer. Suppose that for all $t\in[T]$ one has $\|G_t\|_{\op}\leq 1$ and $\|G_t\|_*\leq r$ and fix $\delta\in(0,1]$. There is a universal constant $c_4>0$ such that if $\eta\leq \frac{c_4}{rL^2\log^2(dLT/\delta)}$, then with probability at least $1-\delta$ the following holds. For any $t\in[T]$, $\alpha\in[L]$, and index $i\in J_{t,\alpha}$ one has
    \begin{equation*}
        \lambda_{t,i},\lambda_{t+\frac{1}{L+1},i},\ldots,\lambda_{t+\frac{\alpha}{L+1},i}\in\sqb{\frac{3}{4}\mu_\alpha,\frac{9}{4}\mu_\alpha}.
    \end{equation*}
\end{lemma}

\begin{proof}
    Before proving the desired bounds, we need to introduce a few events on which the final result will hold. First, let
    \begin{equation*}
        \Ecal_\delta:=\set{\forall\alpha\in[L],\forall t\in[T],\quad t=0\bmod{2^\alpha} \implies\|A_{t,\alpha}\|_{\op}\leq 2c_0\frac{r\log^2(2dT/\delta)}{\mu_\alpha} },
    \end{equation*}
    which by \cref{lemma:bounding_operator_norm} and the union bound has probability at least $1-\delta/2$---for $t=T$ we can imagine completing the epoch with null gain matrices if necessary. Here, we used the fact that the constraint is non-vacuous only if $\alpha\leq \log_2 T$ (in particular, we may consider for this claim that $L\leq\log_2 T$).

    We next turn to bounding the terms $O_t$ for $t\in[T]$. By construction, these satisfy
    \begin{equation}\label{eq:expectation_op_norm_O_t}
        \Ebb\sqb{\|O_t\|_*\mid \Hcal_{t-1},G_t} = \sum_{i\in[d]} \frac{\lambda_{t,i}}{2}\cdot \frac{2}{\lambda_{t,i}} \dotp{G_t}{u_{t,i}u_{t,i}^\top} = \|G_t\|_*\leq r.
    \end{equation}
    For any $\alpha\in[L]$ consider the beginning of a layer-$\alpha$ epoch $t_b=1\bmod 2^\alpha$ and denote by $t_e=t_b+2^\alpha-1$ the end of that epoch. Again if the end of the epoch should be $T$ we can imagine completing it with null gain matrices. For simplicity, if the epoch terminated early at time $T$, we imagine that we complete the epoch with null gain matrices. From \cref{lemma:well_defined_algorithm} the subspace $E_{t_b,\leq\alpha}$ is invariant over this epoch and is an eigenspace of $U_t$ for $t\in[t_b,t_e]$. We start by defining
    \begin{equation*}
        \widetilde O_{t,\leq\alpha} := P_{E_{t_b,\leq\alpha}} O_t P_{E_{t_b,\leq\alpha}} \1\sqb{\lambda_{t,i_t}\geq \frac{\mu_\alpha}{2}}.
    \end{equation*}
    In other terms, $\widetilde O_{t,\leq\alpha}=O_t$ if $u_{t,i_t}\in E_{t_b,\leq\alpha}$ and $\lambda_{t,i_t}\geq \mu_\alpha/2$. Otherwise, $\widetilde O_{t,\leq\alpha}=0$.
    Then, using \cref{eq:expectation_op_norm_O_t} we directly have
    \begin{equation}\label{eq:expecation_tilde_O}
        \Ebb[\|\widetilde O_{t,\leq\alpha}\|_* \mid \Hcal_{t-1},G_t] \leq \Ebb\sqb{\|O_t\|_*\mid \Hcal_{t-1},G_t} \leq r.
    \end{equation}
    Further, by definition we always have
    \begin{equation}\label{eq:upper_bound_tilde_O}
        0\leq \|\widetilde O_{t,\leq\alpha}\|_*\leq \frac{2}{\lambda_{t,i_t}}\1\sqb{\lambda_{t,i_t}\geq  \frac{\mu_\alpha}{2}} \leq \frac{4}{\mu_\alpha}.
    \end{equation}
    We can also bound its variance as follows:
    \begin{align}
        \Ebb\sqb{\|\widetilde O_{t,\leq\alpha}\|_*^2\mid \Hcal_{t-1},G_t} &= \sum_{i\in[d]}\frac{\lambda_{t,i}}{2}\cdot \paren{\frac{2}{\lambda_{t,i}}\dotp{G_t}{u_{t,i}u_{t,i}^\top}}^2 \1[u_{t,i}\in E_{t_b,\leq\alpha},\lambda_{t,i}>\mu_\alpha/2] \notag\\
        &\leq \frac{4}{\mu_\alpha}\|G_t\|_{\op} \sum_{i\in[d]}\dotp{G_t}{u_{t,i}u_{t,i}^\top} \leq \frac{4}{\mu_\alpha}\|G_t\|_*\leq \frac{4r}{\mu_\alpha}. \label{eq:variance_tilde_O}
    \end{align}
    Then, Freedman's inequality together with \cref{eq:expecation_tilde_O,eq:upper_bound_tilde_O,eq:variance_tilde_O} shows that with $\lambda=\frac{\mu_\alpha}{4}$, with probability at least $1-\delta$,
    \begin{align}
        \sum_{t=t_b}^{t_e} \|\widetilde O_{t,\leq\alpha}\|_*&\leq \sum_{t=t_b}^{t_e} \paren{r + 2\lambda \Ebb[\|\widetilde O_{t,\leq\alpha}\|_*^2\mid \Hcal_{t-1},G_t] }  + \frac{1}{\lambda}\log\frac{1}{\delta}\notag\\
        &\leq 3r(t_e-t_b+1) + \frac{4}{\mu_\alpha}\log\frac{1}{\delta}\leq \frac{4}{\mu_\alpha}(r+\log(1/\delta)).
        \label{eq:freedman_tilde_O_bound}
    \end{align}

    We now consider the event $\Fcal$ on which all these inequalities hold:
    \begin{equation*}
        \Fcal_\delta:=\set{\forall \alpha\in[L],\forall t\in[T],\quad  t=1\bmod 2^\alpha\implies \sum_{s=t}^{\min(t+2^\alpha-1,T)} \|\widetilde O_{s,\leq\alpha}\|_* \leq \frac{4}{\mu_\alpha}(r+\log(2LT/\delta)) },
    \end{equation*}
    which has probability at least $1-\delta/2$ by the union bound. From now, we will consider that the event 
    \begin{equation*}
        \Gcal_\delta:=\Ecal_\delta\cap\Fcal_\delta
    \end{equation*}
    holds. Note that by the union bound, this event has probability at least $1-\delta$.

    \paragraph{Setting up the induction.} We now prove the desired result by induction on $t\in[T]$ under $\Gcal_\delta$. For the sake of the proof, we need to prove something slightly stronger. This requires a few notations: we decompose any $t\geq 1$ in binary $t-1=\sum_{\alpha=0}^{L-1} u_{t,\alpha} 2^\alpha \bmod 2^L$ where $u_{t,\alpha}\in\{0,1\}$ for $\alpha\in\{0,\ldots,L-1\}$, and for convenience we denote
    \begin{equation*}
        n_t:=\sum_{\alpha=0}^{L-1}u_{t,\alpha},
    \end{equation*}
    the number of 1's in the binary decomposition of $t-1$ (ignoring digits $\geq L$). Note that we always have $n_t\leq L$. Last, we use $d(x,C)$ to denote the Euclidean distance of $x\in\Rbb^l$ to a convex set $C\subseteq \Rbb^l$. 

    We aim to show that for any $t\in[T]$, $\alpha\in[L]$, and index $i\in J_{t,\alpha}$, we have
    \begin{equation}\label{eq:main_case_induction}
        d(\lambda_{t,i},[\mu_\alpha,2\mu_\alpha])\leq \mu_\alpha\frac{n_t}{4(L+1)}.
    \end{equation}

    We first prove this result for any $t\in[T]$ and $\alpha<\alpha_t$. By construction, $t$ is the start of a layer-$\alpha_t$ epoch (see line~\ref{line:find_alpha_t}) and hence all subspaces for $\alpha<\alpha_t$ are updated at that time as follows (see line~\ref{line:update_subspaces}): $E_{t,\leq\alpha}=F_{t,\leq\alpha}\cap E_{t-1,\leq\alpha_t}$. We also recall that $U_t$ is diagonalizable along the block $E_{t-1,\leq\alpha_t}$ from \cref{lemma:well_defined_algorithm}. In particular,
    \begin{equation*}
        E_{t,\alpha}=(F_{t,\leq\alpha}\cap F_{t,\leq\alpha-1}^\perp)\cap E_{t-1,\leq \alpha_t},
    \end{equation*}
    which implies that all eigenvalues within $E_{t,\alpha}$ belong to $[\mu_\alpha,2\mu_\alpha)$ by definition of the subspaces $F$ (see line~\ref{line:define_F}). This proves \cref{eq:main_case_induction}.

    For $\alpha=\alpha_t$, the same arguments still apply for the upper bounds on the eigenvalues: all eigenvalues within $E_{t,\alpha}$ belong to $[0,2\mu_\alpha)$. Note that if $\alpha_t=L$ then by construction of $U_t$ (see line~\ref{line:uniform_mixing_U}) all eigenvalues are at least $\gamma/d = \mu_L$. In particular, this proves the desired property \cref{eq:main_case_induction} for all times $t=1\bmod 2^L$ corresponding to beginning of layer-$L$ epochs.

    \paragraph{Inductive argument.} We now fix $t\in[T]$ with $t\neq 1\bmod 2^L$, and suppose that the induction holds for all $s<t$. By construction, here $\alpha_t<L$ and more specifically, $\alpha_t$ is the largest index $\alpha\in\{0,\ldots,L-1\}$ such that $t-1=0\bmod 2^\alpha$ (see line~\ref{line:find_alpha_t}). Equivalently, it is the smallest index such that $u_{t,\alpha}=1$. Let $t_0:=t-2^{\alpha_t}$, which corresponds to the beginning of the previous layer-$\alpha_t$. It also corresponds to the beginning of the layer-$(\alpha_t+1)$ epoch containing $t$. Importantly, note that
    \begin{equation}\label{eq:relation_nbs_beginning_epoch}
        n_t=n_{t_0}+1.
    \end{equation}
    In particular, since $n_t\leq L$ we also have $n_{t_0}<L$.
    Recall that we already proved the desired property \cref{eq:main_case_induction} for $\alpha<\alpha_t$. Additionally, if $\beta_{t_0}>\alpha_t$, since the subspaces $E_{s,\leq\alpha}$ are invariant on $[t_0,t]$ for $\alpha\geq\alpha_t$ from \cref{lemma:well_defined_algorithm}, we have $\beta_t=\beta_{t_0}$, and hence the desired properties are void for $\alpha<\beta_t=\beta_{t_0}$: $E_{t,\leq\alpha}=\{0\}$. 
    
    In summary, it only remains to prove the desired properties for $\alpha\geq\max(\alpha_t,\beta_{t_0})$. For convenience, we introduce the notation $\tilde\alpha_t:=\max(\alpha_t,\beta_{t_0})$. We first briefly argue that $\tilde\alpha_t<L$, that is, $\beta_{t_0}<L$. Indeed, by the induction hypothesis, all $\lambda_{t_0,i}$ eigenvalues in $E_{t_0,L}$ are at most $\frac{9}{4}\mu_L\leq 3\gamma/d< 1/d$. On the other hand, $W_{t_0}\in\Scal_d$ must have an eigenvalue at least $1/d$ and hence $U_{t_0}$ also has an eigenvalue of at least $1/d$.
    
    We start by decomposing
    \begin{equation}\label{eq:easy_decomposition}
        \widehat G_{t_0:t-1} = O_{t_0:t-1} + \sum_{\alpha\in[L]}\sum_{s=t_0}^{t-1} B_{s,\alpha} = O_{t_0:t-1} + \sum_{\alpha\in[\alpha_t]} \sum_{\substack{s\in[t_0,t-1]\\ s=0\bmod 2^\alpha}} B_{s,\alpha}
    \end{equation}
    In the last equality we used the definition of $\alpha_t$ and $t_0$. As a result, using $\Ecal_\delta$ we obtain
    \begin{align}
        \|\widehat G_{t_0:t-1} - O_{t_0:t-1}\|_{\op} 
        &\overset{(i)}{\leq} 3\sum_{\alpha\in[\alpha_t]}\sum_{\substack{s\in[t_0,t-1]\\ s=0\bmod 2^\alpha}} \|A_{s,\alpha}\|_{\op} \notag\\
        &\overset{(ii)}{\leq} 6c_0 r \log^2(2dT/\delta) \sum_{\alpha\in[\alpha_t]}\frac{2^{\alpha_t-\alpha}}{\mu_\alpha} = \frac{6c_0 r \alpha_t\log^2(2dT/\delta) }{\mu_{\alpha_t}} \label{eq:bound_A_terms}
    \end{align}
    In $(i)$ we used the fact that $\|P_{E_{s,<\alpha}}A_{s,\alpha}P_{E_{s,<\alpha}}\|_{\op}\leq \|A_{s,\alpha}\|_{\op}$ and similarly $\|\sum_{i\in J_{s,\alpha}} (u_{s,i}^\top A_{s,\alpha} u_{s,i}) u_{s,i}u_{s,i}^\top\|_{\op} \leq \max_{i\in[d]} |u_{s,i}^\top A_{s,\alpha} u_{s,i}|\leq \|A_{s,\alpha}\|_{\op}$. In $(ii)$ we used $\Ecal_\delta$.
    
    Next, by induction hypothesis, for all $s\in[t_0,t-1]$ and index $i\in J_{s,\alpha}$ one has $\lambda_{s,i}\geq \frac{3}{4}\mu_\alpha$.
    Hence,
    \begin{equation*}
        \widetilde O_{s,\leq\alpha} = P_{E_{s,\leq\alpha}} O_s P_{E_{s,\leq\alpha}},\quad s\in[t_0,t-1].
    \end{equation*}
    Recall that any $\alpha\geq\alpha_t$, the subspace $E_{s,\leq\alpha}$ is invariant over the period $[t_0,t]$ from \cref{lemma:well_defined_algorithm}.
    Therefore, using $\Fcal$ we obtain for any $\alpha\geq\alpha_t$,
    \begin{equation}\label{eq:bound_sum_O_projected_terms}
        \|P_{E_{t_0,\leq\alpha}} O_{t_0:t-1} P_{E_{t_0,\leq\alpha}}\|_* \leq \sum_{s=t_0}^{t-1} \|P_{E_{t_0,\leq\alpha}} O_s P_{E_{t_0,\leq\alpha}}\|_* \overset{(i)}{\leq} \frac{4}{\mu_\alpha}(r+\log(2LT/\delta)),
    \end{equation}
    where in $(i)$ we noted that $[t_0,t-1]$ is a layer-$\alpha_t$ epoch and hence is fully contained within a layer-$\alpha$ epoch. The left-hand side is also equal to $\|P_{E_{t_0,\leq\alpha}} \widehat G_{t_0:t-1} P_{E_{t_0,\leq\alpha}}\|_*$ when $\alpha>\tilde\alpha_t$ by construction.
    Next, we have
    \begin{equation*}
        \| P_{E_{t_0,\leq\tilde \alpha_t}} \widehat G_{t_0:t-1} P_{E_{t_0,\leq\tilde\alpha_t}}\|_{\op}  \leq \|\widehat G_{t_0:t-1} - O_{t_0:t-1}\|_{\op}  + \|P_{E_{t_0,\leq\tilde\alpha_t}}O_{t_0:t-1}P_{E_{t_0,\leq\tilde\alpha_t}}\|_{\op} \overset{(i)}{\leq} c_1r\frac{L\log^2(2dT/\delta)}{\mu_{\tilde\alpha_t}},
    \end{equation*}
    for some universal $c_1>0$, where in $(i)$ we used \cref{eq:bound_A_terms} and \cref{eq:bound_sum_O_projected_terms} for $\alpha=\tilde\alpha_t$. We recall that $\widehat G_{t_0:t-1}$ is diagonal along the blocks $E_{t_0,\leq\tilde\alpha_t}$ and $E_{t_0,\beta}$ or $\beta\in\{\tilde\alpha_t+1,\ldots,L\}$. 
    Additionally, by the induction hypothesis, for all $\alpha>\tilde\alpha_t$ and $i\in J_{t_0,\alpha}$ one has $\lambda_{t_0,i}\in[\frac{3}{4}\mu_\alpha,\frac{9}{4}\mu_\alpha]$. Last, for any $i\in J_{t_0,\leq\tilde\alpha_t}:=\bigcup_{\beta\leq \tilde\alpha_t} J_{t_0,\beta}$ we have $\lambda_{t_0,i}\geq \frac{3}{4}\mu_\alpha$. For any $t\geq 1$ and $i\in[d]$ we denote by
    \begin{equation*}
        \tilde \lambda_{t,i}:=\phi(\lambda_{t,i})\quad \text{where}\quad \phi(x)=\frac{x-\gamma/d}{1-\gamma},
    \end{equation*}
    the corresponding eigenvalue of $W_t$. Since $\gamma\in[0,1/6]$, for all $\alpha>\tilde\alpha_t$, one has $\tilde\lambda_{t_0,i}\leq \frac{6}{5}\cdot \frac{9}{4}{\mu_\alpha}\leq 3\mu_\alpha$. Also, if $\alpha<L$ then $\tilde\lambda_{t_0,i} \geq \phi(\frac{3}{4}\mu_\alpha)\geq \frac{1}{4}\mu_\alpha$ (here we used $\mu_\alpha\geq\mu_{L-1}=2\gamma/d$).

    As a result, we can apply \cref{lemma:stability_one_step} to~$G=\eta \widehat G_{t_0:t-1}$, the subspaces $E_{t_0,\leq\tilde\alpha_t},E_{t_0,\tilde\alpha_t+1},\ldots,E_{t_0,L}$, and the constant 
    \begin{equation*}
        \hat c:=\eta\max(4(r+\log(2LT/\delta)), c_1rL\log^2(2dT/\delta))\leq \frac{1}{2^{12}\cdot 4(L+1)},
    \end{equation*}
    where in the last inequality we used the assumption on $\eta$. Therefore, for any $i\in J_{t_0,\alpha}$ where $\alpha\in(\tilde\alpha_t,L)$, we have
    \begin{align*}
        d(\lambda_{t,i},[\mu_\alpha,2\mu_\alpha]) \leq d(\lambda_{t_0,i},[\mu_\alpha,2\mu_\alpha]) + |\lambda_{t,i}-\lambda_{t_0,i}| &\leq d(\lambda_{t_0,i},[\mu_\alpha,2\mu_\alpha]) + |\tilde\lambda_{t,i}-\tilde\lambda_{t_0,i}|\\
        &\overset{(i)}{\leq} \mu_\alpha\frac{n_{t_0}}{4(L+1)} + 2^{12}\hat c\mu_\alpha \overset{(ii)}{\leq} \mu_\alpha\frac{n_t}{4(L+1)},
    \end{align*}
    where in $(i)$ we used the induction and \cref{lemma:stability_one_step}, and in $(ii)$ we used \cref{eq:relation_nbs_beginning_epoch}. This proves the desired induction property when $\alpha\in(\alpha_t,L)$.
    For $\alpha=L$, the desired lower bound is immediate since all eigenvalues of $U_t$ are at least $\gamma/d=\mu_L$. For the upper bound, \cref{lemma:stability_one_step} implies
    \begin{equation*}
        \lambda_{t,i} - \paren{2\mu_L+\mu_L \frac{n_{t_0}}{4(L+1)}} \leq \paren{\tilde \lambda_{t,i} - \phi\paren{2\mu_L+\mu_L \frac{n_{t_0}}{4(L+1)}} }_+\leq  2^{12}\hat c \mu_L \leq \frac{\mu_L}{4(L+1)}.
    \end{equation*}
    Hence we obtain $\lambda_{t,i}-2\mu_L\leq \mu_L\frac{n_t}{4(L+1)}$.
    
    It remains to focus on $\alpha=\tilde\alpha_t$, for which \cref{lemma:stability_one_step} implies
    \begin{equation*}
        \tilde\lambda_{t,i} \geq \min\paren{\phi\paren{\mu_\alpha-\mu_\alpha \frac{n_{t_0}}{4(L+1)}} , 2\mu_\alpha} - 2^{12}\hat c \mu_\alpha = \phi\paren{\mu_\alpha-\mu_\alpha \frac{n_{t_0}}{4(L+1)}}- 2^{12}\hat c \mu_\alpha.
    \end{equation*}
    Then, the same arguments as above imply $\lambda_{t,i}\geq \mu_\alpha-\mu_\alpha\frac{n_t}{4(L+1)}$. This ends the proof of the induction.

    \paragraph{Extending to fractional times in $(t,t+1)$.} The induction already proves the desired result for integer times $t$. We now fix $\alpha\in[L]$ and prove the desired eigenvalue bounds for time $t_\alpha:=t+\frac{\alpha}{L+1}$. By construction, $W_{t_\alpha} = (W_t^{-1} - \eta \widehat G_{t_\alpha} + c_{t_\alpha}I)^{-1}$ where $\widehat G_{t_\alpha}=O_t+\sum_{\beta<\alpha}B_{t,\beta}$ and $c_{t_\alpha}\in\Rbb$ is such that $W_{t_\alpha}\in\Scal_d$. Hence, $\widehat G_{t_\alpha}$ is diagonal along the blocks $E_{t,<\alpha},E_{t,\alpha},\ldots,E_{t,L}$, and the same arguments as above show that we can appropriately bound $\|P_{E_{t,<\alpha}}\widehat G_{t_\alpha}P_{E_{t,<\alpha}}\|_{\op} $ and $\|P_{E_{t,\beta}}\widehat G_{t_\alpha}P_{E_{t,\beta}}\|_* $ for $\beta>\alpha$. The exact same arguments therefore show that for any $\beta\geq\alpha$ and $i\in J_\beta$,
    \begin{equation*}
        d(\lambda_{t_\alpha,i},[\mu_\beta,2\mu_\beta]) \leq \mu_\beta\frac{n_t+1}{4(L+1)} \leq \frac{\mu_\beta}{4}.
    \qedhere
    \end{equation*}
\end{proof}

\begin{lemma}\label{lemma:stability_one_step}
    Let $\alpha\geq 1$ and orthogonal subspaces $\{0\}\subsetneq E_{\leq\alpha},E_{\alpha+1},\ldots,E_{L-1},E_L$ spanning all $\Rbb^d$. Let $W_0\in\Scal_d$ be a matrix diagonal along these subspaces and such that its eigenvalues along each subspace belong to $[\frac{1}{4}\mu_\alpha,1]$ for $E_{\leq\alpha}$, $[\frac{1}{4}\mu_\beta,3\mu_\beta]$ for $E_\beta$ when $\beta\in(\alpha,L)$, and $[0,3\mu_L]$ for $E_L$.
    
    Let $G$ be a matrix diagonal along the blocks $E_{\leq\alpha},E_{\alpha+1},\ldots,E_L$ and such that
    \begin{equation*}
        \|P_{E_{\leq\alpha}} G P_{E_{\leq\alpha}} \|_{\op} \leq \frac{c}{\mu_\alpha} \quad\text{and} \quad \|P_{E_{\beta}} G P_{E_{\beta}} \|_* \leq \frac{c}{\mu_\beta},\quad \beta>\alpha,
    \end{equation*}
    for some $c\in[0,2^{-12}]$.
    Let $W_1 = (W_0^{-1} + c_1 I-G)^{-1}$ where $c_1\in\Rbb$ is such that $W_1\in \Scal_d$. 
    Next, denote by $\mu_{0,\leq\alpha}$ the smallest eigenvalue of $W_0$ along $E_{\leq\alpha}$, and $\mu_{0,L}$ the largest eigenvalue along $E_L$. Then,
    \begin{equation*}
        P_{E_{\leq\alpha}} W_1 P_{E_{\leq\alpha}} \succeq \paren{\min(\mu_{0,\leq \alpha},2\mu_\alpha) - 2^{12}c \mu_\alpha} P_{E_{\leq\alpha}},\quad P_{E_L} (W_1-W_0) P_{E_L} \preceq   2^{12}c \mu_L P_{E_L},
    \end{equation*}
    and,
    \begin{equation*}
        -2^{12}c\mu_\beta P_{E_{\beta}} \preceq P_{E_{\beta}} (W_1-W_0) P_{E_{\beta}} \preceq 2^{12}c\mu_\beta P_{E_{\beta}},\quad \beta>\alpha.
    \end{equation*}
\end{lemma}

\begin{proof}
    We start by introducing a few notations. We denote by $\Lambda_{0,i}$ for $i\in[d]$ the eigenvalues of $W_0^{-1}$, grouped along $E_{\leq\alpha},E_{\alpha+1},\ldots,E_L$---we denote by $J_{\leq\alpha},J_{\alpha+1},\ldots,J_L$ the set of indices of eigenvalues $\Lambda_{0,i}$ within their corresponding subspaces. By convention, we suppose that the eigenvalues along $E_{\leq\alpha}$ are ordered by non-increasing order. Since $W_0\in\Scal_d$ we have
    \begin{equation}\label{eq:definition_c_t_0}
        \sum_{i\in J_{\leq\alpha}}\frac{1}{\Lambda_{0,i}} + \sum_{\beta=\alpha+1}^L \sum_{i\in J_{\beta}} \frac{1}{\Lambda_{0,i}}=1.
    \end{equation}
    Similarly, we denote by $\Lambda_{1,i}$ the eigenvalues of $W_0^{-1} - G$ grouped along the same subspaces $E_{\leq\alpha},E_{\alpha+1},E_{\alpha+2}\ldots$. As above, by convention, we suppose that the eigenvalues along $E_{\leq\alpha}$ are ordered by non-increasing order. Then, Weyl's inequality implies
    \begin{equation}\label{eq:weyls_inequality}
        \max_{i\in J_{\leq\alpha}}|\Lambda_{1,i}-\Lambda_{0,i}|\leq \|P_{E_{\leq\alpha}}GP_{E_{\leq\alpha}}\|_{\op} \leq \frac{c}{\mu_\alpha}.
    \end{equation}
    On the other hand, for any $\beta>\alpha$, Mirsky’s eigenvalue perturbation theorem \cite{mirsky1960symmetric} gives
    \begin{equation}\label{eq:mirsky_inequality}
        \sum_{i\in J_\beta} |\Lambda_{1,i}-\Lambda_{0,i}| \leq \|P_{E_\beta}GP_{E_\beta}\|_* \leq \frac{c}{\mu_\beta}.
    \end{equation}
    On the other hand, by assumption, we have 
    $\frac{1}{\Lambda_{0,i}} \leq  3 \mu_\beta$ for all $i\in J_\beta$ and $\beta>\alpha$.
    As a result,
    \begin{equation}\label{eq:upper_bound_second_sum}
        \sum_{\beta>\alpha}\sum_{i\in J_\beta}\abs{\frac{1}{\Lambda_{1,i}}-\frac{1}{\Lambda_{0,i}}} \overset{(i)}{\leq} \sum_{\beta>\alpha}\sum_{i\in J_\beta} 10\mu_\beta^2 |\Lambda_{1,i}-\Lambda_{0,i}| \overset{(ii)}{\leq} 10c \sum_{\beta>\alpha} \mu_\beta \leq  10c \mu_\alpha.
    \end{equation}
    In $(i)$ we used $c<1/63$ and the inequality $|\frac{1}{x+y}-\frac{1}{x}|\leq \frac{1.05|y|}{x^2}$ for all $x>0$ and $y\in[-x/21,x/21]$. In $(ii)$ we used \cref{eq:mirsky_inequality}.
    On the other hand, by assumption, $\frac{1}{\Lambda_{0,1}} \geq \frac{1}{4}\mu_\alpha$. Therefore, fixing $\Delta:=2^8\frac{c}{\mu_\alpha}$, we have
    \begin{equation}\label{eq:upper_bound_first_term}
        \sum_{i\in J_{\leq\alpha}} \paren{\frac{1}{\Lambda_{1,i}+\frac{c}{\mu_\alpha} + \Delta}-\frac{1}{\Lambda_{0,i}}} \overset{(i)}{\leq} \frac{1}{\Lambda_{0,1}+\Delta} - \frac{1}{\Lambda_{0,1}} \overset{(ii)}{\leq} \frac{1}{4/\mu_\alpha+\Delta} - \frac{1}{4/\mu_\alpha} \overset{(iii)}{<} -\frac{1}{24}\mu_\alpha^2 \Delta.
    \end{equation}
    In $(i)$ we used $\Delta\geq 0$ and \cref{eq:weyls_inequality}. In $(ii)$ we used $\Lambda_{0,1}\leq \frac{4}{\mu_\alpha}$ and in $(iii)$ we used the fact that $\frac{1}{x+\Delta}-\frac{1}{x}< -\frac{2\Delta}{3x^2}$ whenever $\Delta\in[0,x/2)$. Similarly, we have
    \begin{equation}\label{eq:lower_bound_first_term}
        \sum_{i\in J_{\leq\alpha}} \paren{\frac{1}{\Lambda_{1,i}-\frac{c}{\mu_\alpha} - \Delta}-\frac{1}{\Lambda_{0,i}}} \overset{(i)}{\geq} \frac{1}{\Lambda_{0,1}-\Delta} - \frac{1}{\Lambda_{0,1}} \overset{(ii)}{\geq} \frac{\Delta}{\Lambda_{0,1}^2}  \geq \frac{1}{16}\mu_\alpha^2 \Delta.
    \end{equation}
    In $(i)$ we again used $\Delta\geq 0$ and \cref{eq:weyls_inequality} and in $(ii)$ we used the inequality $\frac{1}{x-y}-\frac{1}{x}\geq \frac{y}{x^2}$ for $y\in[0,x)$. Hence,
    \begin{align*}
        &\sum_{i\in J_{\leq\alpha}}\frac{1}{\Lambda_{1,i}+\frac{c}{\mu_\alpha}+\Delta} + \sum_{\beta>\alpha} \sum_{i\in J_\beta} \frac{1}{\Lambda_{1,i}+\frac{c}{\mu_\alpha}+\Delta} -1\\
        &\overset{(i)}{\leq }\sum_{i\in J_{\leq\alpha}}\paren{\frac{1}{\Lambda_{1,i}+\frac{c}{\mu_\alpha}+\Delta} - \frac{1}{\Lambda_{0,i}}} + \sum_{\beta>\alpha} \sum_{i\in J_\beta} \paren{\frac{1}{\Lambda_{1,i}}-\frac{1}{\Lambda_{0,i}}} \\
        &\overset{(ii)}{<} -\frac{1}{24}\mu_\alpha^2\Delta + 10c\mu_\alpha <0.
    \end{align*}
    In $(i)$ we used \cref{eq:definition_c_t_0} and in $(ii)$ we used \cref{eq:upper_bound_second_sum,eq:upper_bound_first_term}. By definition, one has $\tr((W_0^{-1}-G +c_1I)^{-1})=1$. The previous inequality therefore implies $c_1 \leq \frac{c}{\mu_\alpha}+\Delta$. Similarly, we can combine \cref{eq:upper_bound_second_sum,eq:lower_bound_first_term} to obtain
    \begin{equation*}
        \sum_{i\in J_{\leq\alpha}}\frac{1}{\Lambda_{1,i}-\frac{c}{\mu_\alpha}-\Delta} + \sum_{\beta>\alpha} \sum_{i\in J_\beta} \frac{1}{\Lambda_{1,i}-\frac{c}{\mu_\alpha}-\Delta} -1\geq \frac{1}{16}\mu_\alpha^2\Delta - 10c\mu_\alpha>0.
    \end{equation*}
    which implies $c_1\geq -\frac{c}{\mu_\alpha}-\Delta$. In summary, we precisely obtained $|c_1| \leq \frac{c}{\mu_\alpha} + \Delta \leq \frac{257c}{\mu_\alpha}$. This proves the first claim. 
    
    We now turn to the second claim. Fix $\beta\geq\alpha$ and $i\in J_\beta$ (or $i\in J_{\leq\alpha}$ if $\beta=\alpha$). Then,
    \begin{equation*}
        |\Lambda_{1,i}+c_1 - \Lambda_{0,i}| \leq |\Lambda_{1,i}-\Lambda_{0,i}| + |c_1| \overset{(i)}{\leq} \frac{257c}{\mu_\alpha} + \frac{c}{\mu_\beta}\leq \frac{258c}{\mu_\beta}.
    \end{equation*}
    In $(i)$ we used the previous estimate on $c_1$ and \cref{eq:mirsky_inequality} if $\beta>\alpha$ and \cref{eq:weyls_inequality} if $\beta=\alpha$. 
    
    If $\beta<L$, by assumption we have $\frac{1}{\Lambda_{0,i}}\geq \frac{1}{4}\mu_\alpha$. For convenience, denote $\Lambda_{0,i}^+=\max(\Lambda_{0,i},\frac{1}{2\mu_\beta})$. Then,
    \begin{equation*}
        \frac{1}{\Lambda_{1,i}+c_1}-\frac{1}{\Lambda_{0,i}} \geq \frac{1}{\Lambda_{0,i}^+ +|\Lambda_{1,i}+c_1-\Lambda_{0,i}|} - \frac{1}{\Lambda_{0,i}^+} \geq -2\frac{|\Lambda_{1,i}+c_1 -  \Lambda_{0,i}|}{{\Lambda_{0,i}^+}^2} \overset{(i)}{\geq} - 2^{11}c \mu_\beta.
    \end{equation*}
    In $(i)$ we used $c<1/(258\cdot 4)$, $\frac{1}{4}\mu_\beta\leq \frac{1}{\Lambda_{0,i}^+}\leq 2\mu_\beta$, and the inequality $|\frac{1}{x+y}-\frac{1}{x}|\leq 2\frac{|y|}{x^2}$ for $y\in[-x/2,x/2]$. 

    Last, if $\beta>\alpha$, by assumption we have $\frac{1}{\Lambda_{0,i}}\leq 3\mu_\alpha$. Then,
    \begin{equation*}
        \frac{1}{\Lambda_{1,i}+c_1}- \frac{1}{\Lambda_{0,i}}\leq \frac{1}{\Lambda_{0,i} -|\Lambda_{1,i}+c_1-\Lambda_{0,i}|} - \frac{1}{\Lambda_{0,i}} \leq \frac{284c}{\mu_\beta \Lambda_{0,i}^2}\leq  2^{12}c \mu_\beta.
    \end{equation*}
    where we again used the inequality $\frac{1}{x-y}-\frac{1}{x}\leq 1.1\frac{y}{x^2}$ for $y\in[0,x/11]$. This ends the proof.
    \end{proof}

\subsection{OMD regret bound with respect to estimated gain matrices}
\label{subsec:OMD_regret_bound}

We start by proving a general bound on the stability terms that will arize in the OMD analysis.

\lemmastabterm*

\begin{proof}
    First, note that $W^{-1} +\|G\|_{\op} I-G \succeq W^{-1}$ and $W^{-1} -\|G\|_{\op} I-G \preceq W^{-1}$. Given that $W_1,W\in\Scal_d$, this implies that
    \begin{equation}\label{eq:bound_deviation_c}
        |c| \leq \|G\|_{\op}.
    \end{equation}
    Next, we can decompose
    \begin{equation*}
        W_1^{-1} =: \begin{bmatrix}
            X  &Y \\
            Y ^\top& Z ,
        \end{bmatrix} \quad \text{where} \quad \begin{cases}
            X &=W_{11}^{-1}+c I_{11}\\
            Y &=-G_{12}\\
            Z &=W_{22}^{-1}+c I_{22}-G_{22}
        \end{cases}
    \end{equation*}
    Note that
    \begin{equation*}
        Z \succeq \paren{\frac{1}{\mu}+c }I_{22}-G_{22} \overset{(i)}{\succeq} \paren{\frac{1}{\mu}-2\|G\|_{\op}}I_{22} \succeq \frac{I_{22}}{2\mu},
    \end{equation*}
    where in $(i)$ we used \cref{eq:bound_deviation_c} and $\|G_{22}\|_{\op}\leq\|G\|_{\op}$. As a result,
    \begin{equation}\label{eq:upper_bound_op_norm_Z}
        \|Z ^{-1}\|_{\op}\leq 2\mu.
    \end{equation}

    The Schur complement is defined via $S=X - YZ^{-1}Y^\top$, which in particular satisfies $S^{-1}=(W_1)_{11}$. Schur's complement formula then implies
    \begin{align*}
        \dotp{(W_1-W)_{12}}{G_{12}} =\dotp{(W_1)_{12}}{G_{12}} &=-\dotp{S^{-1}YZ^{-1}}{G_{12}}\\
        &=\dotp{(W_1)_{11}G_{12}Z^{-1}}{G_{12}}\\
        &\leq \|(W_1)_{11}\|_* \|Z^{-1}\|_{\op} \|G_{12}\|_{\op}^2 \overset{(i)}{\leq} 2\mu \|G_{12}\|_{\op}^2\leq 4\mu\|G\|_{\op}^2.
    \end{align*}
    In $(i)$ we used \cref{eq:upper_bound_op_norm_Z} and $\|(W_1)_{11}\|_*\leq\|W_1\|_*=1$.

    Since $W^{-1}-W_1^{-1}=G-cI$, multiplying left and right by $W_1$ and $W$ respectively implies $W_1-W=W_1(G-cI)W$. Therefore,
    \begin{align*}
        \dotp{(W_1-W)_{22}}{G_{22}}&=\tr\paren{(W_1(G-cI)W)_{22}G_{22}} \\
        &= \tr((W_1)_{21}G_{12} W_{22}G_{22}) + \tr((W_1)_{22} (G_{22}-cI_{22}) W_{22}G_{22}) \\
        &\overset{(i)}{\leq} -\tr(S^{-1}YZ^{-1}G_{12} W_{22}G_{22}) + 2\mu \|G\|_{\op}^2\\ 
        &\overset{(ii)}{\leq} 4\mu^2\|G\|_{\op}^3 +2\mu\|G\|_{\op}^2 \overset{(iii)}{\leq} 3\mu\|G\|_{\op}^2.
    \end{align*}
    In $(i)$ we used $\|(W_1)_{22}\|_*\leq 1$, \cref{eq:bound_deviation_c}, and $\|G_{22}\|_{\op}\leq\|G\|_{\op}$. In $(ii)$ we used the same arguments as above. In $(iii)$ we used $\|G\|_{\op}\leq \frac{1}{4\mu}$. Altogether, we obtained
    \begin{equation*}
        \dotp{W_1-W}{G} =2\dotp{(W_1-W)_{12}}{G_{12}} +\dotp{(W_1-W)_{22}}{G_{22}} \leq 11\mu\|G\|_{\op}^2.
    \qedhere
    \end{equation*}
\end{proof}

We are now ready to bound the regret of the OMD subroutine within \cref{alg:final_alg}. Throughout, we use the notation $R(W)=-\log\det(W)$ for the negative log-determinant regularizer on~$\Scal_d$.

\begin{lemma}[OMD regret bound]\label{lemma:compute_OMD_regret}
    Fix $\gamma\in(0,1/6]$ such that  $L= \log_2(d/\gamma)$ is integer, and $\delta\in(0,1]$. Suppose that for all $t\in[T]$ one has $\|G_t\|_{\op}\leq 1$ and $\|G_t\|_*\leq r$. Fix $\eta\leq \frac{c_4}{rL^2\log^2(dLT/\delta)}$.
    Then, when running \cref{alg:final_alg}, for any $U\in\Scal_d$ we have
    \begin{equation*}
        \Ebb\sqb{\sum_{t=1}^T \dotp{U-W_t}{\widehat G_t}} \leq \frac{D_R(U\parallel I/d)}{\eta} + c_2 \eta r^2 \log^4(dT/\delta) L T + c_2\frac{d^2T}{\gamma}\delta,
    \end{equation*}
    for some universal constant $c_2>0$.
\end{lemma}

\begin{proof} 
    Throughout, we suppose that $\Gcal_\delta$ holds, where $\Gcal_\delta$ is the event defined within the proof of \cref{lemma:stability_eigenvalues}, under which the guarantees of \cref{lemma:stability_eigenvalues} hold.
    Following the standard analysis of mirror descent (e.g.~\cite{juditsky2011first}), for any fixed $U\in\Scal_d$ we have
    \begin{multline*}
        \sum_{t=1}^T \paren{\dotp{U-W_t}{O_t} + \sum_{\alpha=1}^L\dotp{U-W_{t+\frac{\alpha}{L+1}}}{B_{t,\alpha}}} \leq \frac{D_R(U\parallel W_1)}{\eta} + \sum_{t=1}^T \left(\dotp{W_{t+\frac{1}{L+1}}-W_t}{O_t} \right. \\
        + \sum_{\alpha\in[L]} \left. \dotp{W_{t+\frac{\alpha+1}{L+1}} - W_{t+\frac{\alpha}{L+1}}}{B_{t,\alpha}} \right).
    \end{multline*}
    We note that by construction $W_1=I/d$. 
    We start by focusing on the left-hand side. Fix $t\in[T]$. By construction, $W_t$ and $O_t$ have the same eigenvectors $u_{t,1},\ldots,u_{t,d}$. As a result, $W_{t+\frac{1}{L+1}}$ also shares the same eigenvectors. Next, for any $\alpha\in[L]$ we have $P_{E_{t,<\alpha}}\sum_{\beta<\alpha}B_{t,\beta} P_{E_{t,<\alpha}}=\sum_{\beta<\alpha}B_{t,\beta}$. As a result, $W_{t+\frac{\alpha}{L+1}}$ is diagonalizable along $E_{t,<\alpha}$ and any vector $u_{t,i}\in E_{t,\geq\alpha}$ for $i\in[d]$ is also an eigenvector of $W_{t+\frac{\alpha}{L+1}}$. On the other hand, note that $B_{t,\alpha}$ does not have any component on these block diagonals. This shows that $\dotp{W_{t+\frac{\alpha}{L+1}}}{B_{t,\alpha}}=0$. Altogether, this implies that
    \begin{equation*}
        \dotp{U-W_t}{O_t} + \sum_{\alpha=1}^L\dotp{U-W_{t+\frac{\alpha}{L+1}}}{B_{t,\alpha}} = \dotp{U-W_t}{O_t} + \sum_{\alpha=1}^L\dotp{U}{B_{t,\alpha}} = \dotp{U}{\widehat G_t} - \dotp{W_t}{O_t} \overset{(i)}{=} \dotp{U-W_t}{\widehat G_t}.
    \end{equation*}
    In $(i)$ we used the fact that $\sum_{\alpha\in[L]}B_{t,\alpha}$ do not have any component along $u_{t,i} u_{t,i}^\top$ by construction. In summary, the previous inequality becomes
    \begin{equation}\label{eq:OMD_bound_final}
        \sum_{t=1}^T \dotp{U-W_t}{\widehat G_t}  \leq \frac{D_R(U\parallel I/d)}{\eta} + \sum_{t=1}^T \paren{\dotp{W_{t+\frac{1}{L+1}}-W_t}{O_t} 
        + \sum_{\alpha\in[L]}  \dotp{W_{t+\frac{\alpha+1}{L+1}}- W_{t+\frac{\alpha}{L+1}}}{B_{t,\alpha}} }.
    \end{equation}

    Fix $t\in[T]$ and $\alpha\in[L]$. For notational convenience, denote $t_\alpha:=t+\frac{\alpha}{L+1}$. Under the event $\Gcal_\delta$, 
    \begin{equation*}
        \|\eta B_{t,\alpha}\|_{\op}\leq 3\eta\|A_{t,\alpha}\|_{\op} \leq 6\eta c_0 \frac{r \log^2(dT/\delta)}{\mu_\alpha} \overset{(i)}{\leq} \frac{1}{12\mu_\alpha},
    \end{equation*}
    where in $(i)$ we used the assumption on $\eta$.
    On the other hand, under the event $\Gcal_\delta$, the eigenvalues of $U_{t_\alpha}$ along $E_{t,\alpha}$ are at most $\frac{9}{4}\mu_\alpha$ by \cref{lemma:stability_eigenvalues}. Therefore, the eigenvalues of $W_{t_\alpha}$ along $E_{t,\alpha}$ are at most $\frac{6}{5}\cdot \frac{9}{4}\mu_\alpha\leq 3\mu_\alpha$ (since $\gamma\in[0,1/6]$). Hence, we can apply \cref{lemma:updated_bound_stability_term} to $W=W_{t_\alpha}$, $G=\eta B_{t,\alpha}$, and the blocks $E_{t,\alpha}^\perp$ and $E_{t,\alpha}$. This gives
    \begin{equation}\label{eq:bound_stability_term_alpha_nonzero}
        \dotp{W_{t_{\alpha+1}}- W_{t_\alpha}}{\eta B_{t,\alpha}} \leq 33 \mu_\alpha \eta^2\|B_{t,\alpha}\|^2_{op} \leq 200 \eta^2 c_0^2 r^2 \log^4(dT/\delta) \frac{\1[t=0\bmod 2^\alpha]}{\mu_\alpha}.
    \end{equation}
    It remains to handle the diagonal term $\dotp{W_{t+\frac{1}{L+1}}-W_t}{O_t} $, which is non-zero when $Z_t=0$. In that case, $\|O_t\|_{\op}=\frac{2\ell_t^2}{\lambda_{t,i_t}}\leq \frac{3}{\tilde\lambda_{t,i_t}}$ since $\gamma\in[0,1/6]$. Hence, we can apply \cref{lemma:updated_bound_stability_term} to $W=W_t$, $G=\eta O_t$, and the blocks $\{u_{t,i_t}\}^\perp$ and $u_{t,i_t}\Rbb$, on which $W_t$ precisely has eigenvalue $\tilde\lambda_{t,i_t}$. This gives
    \begin{equation}\label{eq:bound_stability_term_alpha_zero}
        \dotp{W_{t+\frac{1}{L+1}}- W_t}{\eta O_t} \leq 33 \tilde\lambda_{t,i_t} \eta^2\|O_t\|^2_{op} \leq 40 \eta^2 \frac{\ell_t^4}{\lambda_{t,i_t}} \1[Z_t=0].
    \end{equation}

    Plugging these estimates within \cref{eq:OMD_bound_final} shows that under $\Gcal_\delta$,
    \begin{equation*}
        \sum_{t=1}^T \dotp{U-W_t}{\widehat G_t} \leq \frac{D_R(U\parallel I/d)}{\eta} + c_5\eta\sum_{t=1}^T \frac{\ell_t^4\1[Z_t=0]}{\lambda_{t,i_t}} + c_5 \eta r^2 \log^4(dT/\delta) L T,
    \end{equation*}
    for some universal constant $c_5>0$. Note that for any $t\in[T]$ one has
    \begin{equation*}
        \Ebb\sqb{\frac{\ell_t^4 \1[Z_t=0]}{\lambda_{t,i_t}}\mid \Hcal_{t-1},G_t} = \sum_{i\in[d]} \frac{\lambda_{t,i}}{2} \frac{\dotp{G_t}{u_{t,i}u_{t,i}^\top}^2}{\lambda_{t,i}} \overset{(i)}{\leq} \frac{1}{2}\sum_{i\in[d]}\dotp{G_t}{u_{t,i}u_{t,i}^\top} = \frac{\|G_t\|_*}{2}\leq \frac{r}{2}.
    \end{equation*}
    In $(i)$ we used $\|G_t\|_{\op}\leq 1$. In summary, we obtained
    \begin{equation*}
        \Ebb\sqb{\1[\Gcal_\delta]\sum_{t=1}^T \dotp{U-W_t}{\widehat G_t}} \leq \frac{D_R(U\parallel I/d)}{\eta} + 2c_5 \eta r^2 \log^4(dT/\delta) L T.
    \end{equation*}

    Additionally, we always have
    \begin{align*}
        \sum_{t=1}^T \dotp{U-W_t}{\widehat G_t} \overset{(i)}{\leq} 2\sum_{t=1}^T\|\widehat G_t\|_{\op} &\overset{(ii)}{\leq} \frac{24}{\mu_L}dT\leq \frac{48d^2T}{\gamma}.
    \end{align*}
    In $(i)$ we used $\|U-W_t\|_*\leq \|U\|_*+\|W_t\|_*=2$. In $(ii)$ we used the fact that $U_t\succeq \frac{\gamma}{d}I\succeq \mu_L I$ and $\|G_t\|_{\op}\leq 1$ so that $\ell_t^2\leq 1$.
    Combining the last three inequalities, we obtain
    \begin{equation*}
        \Ebb\sqb{\sum_{t=1}^T \dotp{U-W_t}{\widehat G_t}} \leq \frac{D_R(U\parallel I/d)}{\eta} + c_5 \eta r^2 \log^4(dT/\delta) L T + \frac{48d^2T}{\gamma}\delta,
    \end{equation*}
    where we used $\Pbb[\Gcal_\delta^c]\leq\delta$. This ends the proof.
\end{proof}

As a last step before proving the final result, we check that with high probability, when $Z_t=1$ we indeed sample the queries $w_t$ according to the procedure in lines \ref{line:sample_I_t}--\ref{line:sample_w_to_complete_R}. This shows that the ``good scenario'' for the layered sampling scheme holds with high probability.

\begin{lemma}\label{lemma:R_not_negative}
    Suppose that the assumptions of \cref{lemma:compute_OMD_regret} and $\Gcal_\delta$ hold. Then, for all $t\in[T]$,
    \begin{equation*}
        R_t = U_t-\sum_{\alpha\in[L],d_{t,\leq\alpha}>0} p_{t,\alpha}\frac{P_{E_{t,\leq\alpha}}}{d_{t,\leq\alpha}} \succeq 0.
    \end{equation*}
\end{lemma}

\begin{proof}
    From \cref{lemma:stability_eigenvalues}, under $\Gcal_\delta$, for any $t\in[T]$ and $\alpha\in[L]$ one has
    \begin{equation*}
        P_{E_{t,\alpha}}U_tP_{E_{t,\alpha}} \succeq \frac{3\mu_\alpha}{4}P_{E_{t,\alpha}}.
    \end{equation*}
    Summing these equations and recalling that $U_t$ is diagonal along the blocks $E_{t,1},\ldots,E_{t,L}$ from \cref{lemma:well_defined_algorithm}, we obtain
    \begin{align*}
        U_t = \sum_{\alpha\in[L]} P_{E_{t,\alpha}}U_tP_{E_{t,\alpha}} &\succeq \frac{3}{4} \sum_{\alpha\in[L]}  P_{E_{t,\alpha}} \mu_\alpha\\
        &\succeq \frac{3}{8} \sum_{\alpha\in[L]}  P_{E_{t,\alpha}} \sum_{\beta=\alpha}^L \mu_\beta = \frac{3}{8} \sum_{\beta=1}^L  \mu_\beta P_{E_{t,\leq\beta}} = \frac{3}{2}  \sum_{\beta\in[L],d_{t,\leq\beta}>0}  p_{t,\beta} \frac{P_{E_{t,\leq\beta}}}{d_{t,\leq\beta}}.
        \qedhere
    \end{align*}
\end{proof}

\subsection{Proof of \cref{thm:ub}}
\label{subsec:ubthm}

Given the above lemmas, we are now ready to prove our final regret bound. The last remaining step is to relate the left-hand side of \cref{lemma:compute_OMD_regret} to the actual (pseudo--)regret of the final algorithm \cref{alg:final_alg}, and in particular, relating the estimates $\widehat G_t$ to the true gain matrices $G_t$.

\begin{proof}[Proof of \cref{thm:ub}]
    We start by relating the estimates $\widehat G_t$ to the true gain matrices $G_t$. From \cref{lemma:unbiased_hat_L_final}, for any $t\in[T]$ and $\alpha\in[L]$, we have
    \begin{equation}\label{eq:unbiaised_by_subsapace_leq_alpha}
        \Ebb\sqb{\ell_t^2((d_{t,\leq\alpha}+2)w_t w_t^\top -P_{E_{t,\leq\alpha}}) \mid\Hcal_{t-1},G_t,Z_t=1,R_t\succeq 0,I_t=\alpha} = \frac{2}{d_{t,\leq\alpha}}P_{E_{t,\leq\alpha}}G_t P_{E_{t,\leq\alpha}}.
    \end{equation}
    We next fix $t_b\in[T]$ the beginning of a layer-$\alpha$ epoch---that is $t_b=1\bmod 2^\alpha$---and let $t_e=\min(t_b+2^\alpha-1, T)$ be the end of the corresponding epoch. We recall that the subspaces $E_{s,\leq\alpha},E_{s,<\alpha},E_{s,\alpha}$ are invariant on the epoch for $s\in[t_b,t_e]$ by \cref{lemma:well_defined_algorithm}. Further, for any $s\in[t_b,t_e]$, $W_s$ is diagonal along the blocks $E_{t_b,<\alpha},E_{t_b,\alpha+1},\ldots,E_{t_b,L}$ and the eigenvectors $u_{t_b,i}$ for $i\in J_{t_b,\alpha}$. As a result, for any $s\in[t_b,t_e]$ and matrix $A$, if we denote
    \begin{equation*}
        \Psi_{s,\alpha}(A):= P_{E_{s,\leq\alpha}}AP_{E_{s,\leq\alpha}}-P_{E_{s,<\alpha}}AP_{E_{s,<\alpha}} - \sum_{i\in J_{t_b,\alpha}} (u_{t_b,i}^\top A u_{t_b,i}) u_{t_b,i}u_{t_b,i}^\top,
    \end{equation*}
    then we always have $\dotp{W_s}{\Psi_{s,\alpha}(A) } = 0$. From the previous discussion, $\Psi_{s,\alpha}=\Psi_{t_b,\alpha}$ for all $s\in[t_b,t_e]$.
    Therefore,
    \begin{equation*}
        \Ebb\sqb{\dotp{W_{t_e}}{B_{t_e,\alpha}}}=\Ebb\sqb{\dotp{W_{t_e}}{\Psi_{t_b,\alpha}(A_{t_e,\alpha})}}=0=\Ebb\sqb{\sum_{s=t_b}^{t_e} \dotp{W_s}{\Psi_{s,\alpha}(G_s)}}.
    \end{equation*}
    On the other hand, for any fixed $U\in\Scal_d$, by the linearity of the operator $\Psi_{t_b,\alpha}$ we have 
    \begin{equation*}
        \Ebb\sqb{\dotp{U}{B_{t_e,\alpha}}} = \dotp{U}{\Psi_{t_b,\alpha}(\Ebb[A_{t_e,\alpha}])},
    \end{equation*}
    and
    \begin{align*}
        \Ebb\sqb{A_{t_e,\alpha} } &\overset{(i)}{=} \frac{8}{\mu_\alpha d_{t_b,\leq\alpha}}\Ebb\sqb{\sum_{s=t_b}^{t_e} \1[R_s\succeq 0]\1[Z_s=1,I_s=\alpha] P_{E_{s,\leq\alpha}}G_s P_{E_{s,\leq\alpha}} }\\
        &=  \Ebb\sqb{\sum_{s=t_b}^{t_e} \1[R_s\succeq 0] P_{E_{s,\leq\alpha}}G_s P_{E_{s,\leq\alpha}} }
    \end{align*}
    In $(i)$ we used \cref{eq:unbiaised_by_subsapace_leq_alpha}. We also recall that $E_{s,<\alpha}$ is constant on the epoch $s\in[t_b,t_e]$ from \cref{lemma:well_defined_algorithm}. Combining the last three equalities then gives
    \begin{align*}
        \Ebb\sqb{\dotp{U-W_{t_e}}{B_{t_e,\alpha}}} &= \Ebb\sqb{\sum_{s=t_b}^{t_e} \dotp{U-W_s}{\Psi_{s,\alpha}(G_s)}} - \sum_{s=t_b}^{t_e}\dotp{U}{\Ebb\sqb{\1[R_s\prec 0] \Psi_{s,\alpha}(G_s)}}\\
        &\overset{(i)}{\geq} \Ebb\sqb{\sum_{s=t_b}^{t_e} \dotp{U-W_s}{\Psi_{s,\alpha}(G_s)}} - 3\delta(t_e+1-t_b) .
    \end{align*}
    In $(i)$ we used $\|\Psi_{s,\alpha}(G_s)\|_{\op}\leq 3\|G_s\|_{\op}\leq 3$ and \cref{lemma:R_not_negative} which shows that $\Pbb[R_s\prec 0]\leq\Pbb[ \Gcal_\delta^c]\leq\delta$.

    We next sum all these inequalities for all layer-$\alpha$ epochs $[t_b,t_e]$ and all $\alpha\in[L]$ which gives
    \begin{equation*}
        \Ebb\sqb{\sum_{t=1}^T \dotp{U-W_t}{\sum_{\alpha\in[L]}B_{t,\alpha}}} \geq \Ebb\sqb{\sum_{t=1}^T \dotp{U-W_t}{\widetilde G_t}} - 3\delta LT,
    \end{equation*}
    where 
    \begin{align*}
        \widetilde G_t:= \sum_{\alpha\in[L]}\Psi_{t,\alpha}(G_t) \overset{(i)}{=} G_t -\sum_{i\in[d]} (u_{t,i}^\top G_t u_{t,i}) u_{t,i}u_{t,i}^\top.
    \end{align*}
    In $(i)$ we used the fact that $P_{E_{t,\leq L}}=I$ and $P_{E_{t,<1}}=0$.
    On the other hand,
    \begin{equation*}
        \Ebb[O_t\mid \Hcal_t] = \sum_{i\in[d]} \frac{\lambda_{t,i}}{2}\cdot \frac{2}{\lambda_{t,i}} \dotp{u_{t,i}u_{t,i}^\top}{G_t} u_{t,i}u_{t,i}^\top = \sum_{i\in[d]} (u_{t,i}^\top G_t u_{t,i}) u_{t,i}u_{t,i}^\top.
    \end{equation*}
    We recall that $W_t$ is $\Hcal_t$-measurable since it only depends on $\widehat G_{1:t-1}$. Altogether, we obtained
    \begin{align}
        \Ebb\sqb{\sum_{t=1}^T \dotp{U-W_t}{\widehat G_t}} &= \Ebb\sqb{\sum_{t=1}^T \dotp{U-W_t}{O_t}}  + \Ebb\sqb{\sum_{t=1}^T \dotp{U-W_t}{\sum_{\alpha\in[L]}B_{t,\alpha}}} \notag\\
        &\geq \Ebb\sqb{\sum_{t=1}^T \dotp{U-W_t}{ G_t}} - 3\delta LT. \label{eq:link_estimators_true_G}
    \end{align}

    We now observe that by construction, for any $t\in[T]$ we have
    \begin{equation}\label{eq:play_OMD_matrix}
        \Ebb[w_t w_t^\top\mid\Hcal_{t-1},G_t] = U_t.
    \end{equation}
    Indeed, when $Z_t=1$ and $R_t\succeq 0$ we can check that
    \begin{equation*}
        \Ebb[w_t w_t^\top\mid\Hcal_t,Z_t=1,R_t\succeq 0] = p_{t,0} \frac{R_t}{p_{t,0}} + \sum_{\alpha\in[L]} p_{t,\alpha} \frac{P_{E_{t,\leq\alpha}}}{d_{t,\leq\alpha}} = U_t,
    \end{equation*}
    where by convention $0/0=0$ when $p_{t,\alpha}=0$ for some $\alpha\in\{0,\ldots,L\}$.
    On the other hand, when $Z_t=0$ or $R_t\prec 0$ we have
    \begin{equation*}
        \Ebb[w_t w_t^\top\mid\Hcal_t,Z_t=0 \text{ or }R_t\prec 0] = \sum_{i\in[d]} \lambda_{t,i} u_{t,i} u_{t,i}^\top = U_t.
    \end{equation*}

    For convenience, for a fixed $U\in\Scal_d$ we introduce $U_\gamma:=(1-\gamma)U+\gamma I/d$.
    Then, we can compute
    \begin{align*}
        \Ebb\sqb{\sum_{t=1}^T \dotp{U-w_t w_t^\top}{G_t}}& \overset{(i)}{=} \Ebb\sqb{\sum_{t=1}^T \dotp{U-U_t}{ G_t}}\\
        &=\Ebb\sqb{\sum_{t=1}^T \dotp{U_\gamma-W_t}{ G_t}} + \gamma \Ebb\sqb{\sum_{t=1}^T \dotp{U+W_t-2I/d}{ G_t}} \\
        &\overset{(ii)}{\leq }\Ebb\sqb{\sum_{t=1}^T \dotp{U_\gamma-W_t}{\widehat G_t}} + 3\delta LT + 4\gamma T\\
        &\overset{(iii)}{\leq} \frac{D_R(U_\gamma\parallel I/d)}{\eta} + c_2 \eta r^2 \log^4(dT/\delta) L T + c_2\frac{d^2T}{\gamma}\delta + 4(\delta L+\gamma)T.
    \end{align*}
    In $(i)$ we used \cref{eq:play_OMD_matrix}. In $(ii)$ we used \cref{eq:link_estimators_true_G}, $\|U+W_t-2I/d\|_*\leq \|U\|_* +\|W_t\|_* +2\leq 4$, and $\|G_t\|_{\op}\leq 1$. In $(iii)$ we used \cref{lemma:compute_OMD_regret}. By construction, $U_\gamma\succeq \gamma I/d$. Hence,
    \begin{equation*}
        D_R(U_\gamma\parallel I/d) \overset{(i)}{=} \log\frac{\det(I/d)}{\det(U_\gamma)} \leq d\log\frac{1}{\gamma}.
    \end{equation*}
    In $(i)$ we used $U_\gamma\in\Scal_d$. Plugging this within the previous estimate shows that
    \begin{equation*}
        \Reg_T(\cref{alg:final_alg}) \lesssim  \frac{d}{\eta}\log(1/\gamma) + \eta r^2 \log^4(dT/\delta) L T + \frac{d^2}{\gamma}\delta T + (\gamma+\delta L) T.
    \end{equation*}
    Choosing $\delta=(\frac{\gamma}{dT})^2$ and appropriately simplifying the constraint $\eta\leq \frac{c_4}{rL^2\log^2(dLT/\delta)}$ ends the proof of the first claim.

    We now turn to the second claim. The only part that needs checking is when $\eta \geq \frac{c_0}{r\log^4(dT^2)}\geq \frac{c_0}{16r\log^4(dT)}$. In this case, we have
    \begin{equation*}
        \sqrt{dT} \geq \frac{c_0 T}{16 \log^2(dT)},
    \end{equation*}
    and as a result, the desired bound is vacuous.
\end{proof}

\section{Computational complexity of \cref{alg:final_alg}}
\label{sec:computational_complexity}

In this section, we detail an implementation of \cref{alg:final_alg} that gives $O(d^2 T\log d)$ total runtime. We fix a round $t\in [T]$ and detail the steps performed at that round.

We recall that $\alpha_t$ denotes the coarsest layer $\beta\in[L]$ such that $t$ is the beginning of a layer-$\beta$ epoch.
We denote by $t_0:=t-2^{\alpha_t}$ the beginning of the previous layer-$\alpha_t$ epoch. From Claim 4 of \cref{lemma:well_defined_algorithm}, all eigenvectors $u_{t_0,i}\in E_{t_0,\leq\alpha_t}^\perp$ for round $t_0$ are still eigenvectors of $U_s$ throughout the epoch $s\in[t_0,t]$. Note that $E_{t_0,\leq\alpha_t}$ is also constant throughout the epoch $s\in[t_0,t]$ by Claim 2 of \cref{lemma:well_defined_algorithm}. By construction of the eigenvectors (see line~\ref{line:eigenval_decompose_U} of \cref{alg:subspace_update}) this shows that all eigenvectors
\begin{equation*}
    u_{t,i}=u_{t-1,i},\quad i\in J_{t-1,\beta}=J_{t,\beta},\, \beta>\alpha_t,
\end{equation*}
are kept unchanged at round $t$. Here, we recall that $J_{t,\beta}$ denotes the indices of eigenvectors $u_{t,i}$ spanning $E_{t,\alpha}$ ($U_t$ and $W_t$ are diagonal by block along the exploration subspaces $E_{t,1},\ldots,E_{t,L}$ by Claim 3 of \cref{lemma:well_defined_algorithm} and hence $J_{t,1},\ldots,J_{t,L}$ partition all $[d]$ eigenvectors).

In summary, to complete the eigendecomposition of $U_t$, it suffices to compute the eigendecomposition of $U_t$ along the block $E_{t,\leq\alpha_t}\times E_{t,\leq\alpha_t}$, and a basis of $E_{t,\leq\alpha_t} = E_{t-1,\alpha_t}$ is precisely given by the previous eigenvectors
\begin{equation*}
    \set{u_{t-1,j}:j\in J_{t-1,\beta},\beta\leq\alpha_t}.
\end{equation*}
Before detailing the corresponding implementation of the eigenvalue decomposition of $U_t$, note that the OMD iterate $W_t$ and the mixed iterate $U_t$ share the same eigenvectors $(u_{t,i})_{i\in[d]}$ and the same eigenvalues up to a $\gamma$-mixing with the uniform. In fact, they share the same eigenvalues as 
\begin{equation*}
    W_{t-1}^{-1} - \eta\widehat G_{t-1} = c_{t-1}I -\eta \widehat G_{1:t-1} = W_t^{-1} - (c_t-c_{t-1})I.
\end{equation*}

Hence, we may compute the OMD iterate $W_t$ and the eigendecomposition of $U_t$ together as follows:
\begin{enumerate}
    \item Let $\alpha_t:=\max\{\alpha\in\{0,\ldots,L\}: t=1\bmod 2^\alpha\}$. Let $X_t$ be the concatenation of vectors $u_{t-1,i}$ for $i\in J_{t-1,\beta}$ for $\beta\in[\alpha_t]$. Last, for $i\in[d]$, define $\zeta_{t-1,i}:=\frac{1-\gamma}{\lambda_{t-1,i}-\frac{\gamma}{d}}$
    which correspond to eigenvalues of $W_{t-1}^{-1}$.
    \label{item:start_implementation}
    \item Define $\widehat G_{t-1,\leq\alpha_t}:= X_t^\top \widehat G_{t-1} X_t$ and $\Lambda_{t-1,\leq\alpha_t}^{-1}:= \diag\paren{\zeta_{t-1,i}:i\in J_{t,\beta},\beta\in[\alpha_t]}$.\footnote{Note that by construction $\Lambda_{t-1,\leq\alpha_t}^{-1} = X_t^\top W_{t-1}^{-1} X_t$.}\label{item:reduce_to_smaller_matrix_bad}
    
    \item Compute the eigendecomposition \label{item:compute_decomposition_implementation}
    \begin{equation*}
        \Lambda_{t-1,\leq\alpha_t}^{-1} -\eta \widehat G_{t-1,\leq\alpha_t} = \sum_{i\in[d_{t-1,\leq\alpha_t}]} \nu_{t,i} v_{t,i}v_{t,i}^\top.
    \end{equation*}
    \item For any $\beta>\alpha_t$ and $i\in J_{t-1,\beta}$ we pose $(u_{t,i},\zeta_{t,i}):=(u_{t-1,i} ,\zeta_{t-1,i} - \eta u_{t,i}^\top\widehat G_{t-1} u_{t,i})$.
    We complete this sequence with the vectors and scalars $(X_t v_{t,i},\nu_{t,i})_{i\in [d_{t-1,\leq\alpha_t}]}$. Last, we pose
    \begin{equation*}
        \lambda_{t,i}:= (1-\gamma)\tilde\lambda_{t,i} + \frac{\gamma}{d},\quad \text{where} \quad \tilde\lambda_{t,i}:= \paren{\zeta_{t,i} +\frac{1-s_{t}}{d}}^{-1} \quad\text{and} \quad s_{t}:=\sum_{i\in[d]}\zeta_{t,i}.
    \end{equation*}
    This ends the construction of the eigenvectors and eigenvalues $(u_{t,i},\lambda_{t,i})_{i\in[d]}$ of $U_t$.
\end{enumerate}
It is immediate to check the validity of the procedure, the only observation required is that we can easily check, $\widehat G_{t-1} = O_{t-1} + \sum_{\alpha\in[\alpha_t]}B_{t,\alpha}$, and hence this gain estimator only has non-zero entries on the block $E_{t-1,\leq\alpha_t}\times E_{t-1,\leq\alpha}$ and the remaining diagonal. 
As a result, these steps replace lines~\ref{line:OMD_update}--\ref{line:uniform_mixing_U} of \cref{alg:final_alg} as well as line~\ref{line:eigenval_decompose_U} of \cref{alg:subspace_update}. 

The remainder of \cref{alg:subspace_update} can be implemented in a straightforward manner once we have this eigendecomposition of $U_t$. For completeness, we can use the following steps.
\begin{enumerate}[start=5]
    \item Denote $J_{t-1,\leq\alpha_t}:=\bigcup_{\beta\in[\alpha_t]}J_{t-1,\beta}$. Then, we set
    \begin{equation*}
        J_{t,\alpha}:=\begin{cases}
            J_{t-1,\alpha} & \alpha>\alpha_t\\
            \set{i\in J_{t-1,\leq\alpha_t}: \lambda_{t,i}\geq\mu_{\alpha_t}} &\alpha=\alpha_t\\
            \set{i\in J_{t-1,\leq\alpha_t}: \lambda_{t,i}\in[\mu_\alpha,\mu_{\alpha+1})} &\alpha<\alpha_t.
        \end{cases}
    \end{equation*}
\end{enumerate}
With this construction, the subspace $E_{t,\leq\alpha}$ corresponds to the span of its basis $\{u_{t,i}:i\in \bigcup_{\beta\in[\alpha]} J_{t,\alpha}\}$: this representation is sufficient for the rest of the algorithm. 

We start by checking the steps within the layered sampling procedure \cref{alg:layered_sampling}.
The residual matrix $R_t$ has the same eigenbasis $(u_{t,i})_{i\in[d]}$ and has corresponding eigenvalue 
\begin{equation*}
    r_{t,i}:=\lambda_{t,i}-\frac{p_{t,\alpha}}{d_{t,\leq\alpha}}\1[d_{t,\leq\alpha}>0],\quad i\in J_{t,\alpha},\alpha\in[L].
\end{equation*}
With the eigendecomposition of the residual matrix, we can readily check the condition $R_t\succeq 0$ in line~\ref{line:good_event_holds} and implement the corresponding residual sampling of line~\ref{line:sample_w_to_complete_R} of \cref{alg:layered_sampling}. We can also easily sample uniformly from $S_d\cap E_{t,\leq\alpha}$ as in line~\ref{line:sample_w_good_scenario} of \cref{alg:layered_sampling}, given the basis of $E_{t,\leq\alpha}$. For completeness, we can for instance use the following implementation for \cref{alg:layered_sampling}.

\begin{enumerate}[start=6]
    \item For $\alpha\in[L]$ denote $d_{t,\leq\alpha}:=\sum_{\beta\in[\alpha]}|J_{t,\beta}|$, $p_{t,\alpha}:=\frac{1}{4}\mu_\alpha d_{t,\leq\alpha}$, $p_{t,0}:=1-\sum_{\alpha\in[L]}p_{t,\alpha}$, and sample $Z_t\sim\mathrm{Bernoulli}(1/2)$.
    \item {\bf If} $Z_t=1$ and $\forall i\in[d],r_{t,i}\geq 0$:
    \begin{enumerate}
        \item Sample $I_t\sim (p_{t,\alpha})_{\alpha\leq L}$.
        \item {\bf If} $I_t\in[L]$: set $w_t:=\tfrac{\tilde w_t}{\|\tilde w_t\|}$ where $\tilde w_t:=\sum_{i\in J_{t,\beta},\beta\in[I_t]} Z_{t,i} u_{t,i}$ and $Z_{t,i}\overset{iid}{\sim}\Ncal(0,1)$ are independent Gaussians.
        \item {\bf Else} Sample $i_t\sim (\tfrac{r_{t,i}}{p_{t,0}})_{i\in[d]}$ and set $w_t:=u_{t,i}$.
    \end{enumerate}
    \item {\bf Else} Sample $i_t\sim (\lambda_{t,i})_{i\in[d]}$ and set $w_t:=u_{t,i}$.
\end{enumerate}

It only remains to implement the gain matrix estimation procedure \cref{alg:gain_estimation}. The computation of the matrices $A_{t,\alpha}$ for $\alpha\in[L]$ only takes $O(d^2)$ per round because we only need to update one of these at a time---the estimate for the layer that was picked $I_t\in[L]$ if it exists. To clarify this point, we consider that the procedure computes online the various estimates $A_{t,\alpha}$ at all times, instead of computing them only at the end of their layer-$\alpha$ epoch. Additionally, we can compute the projection matrices $P_{t,\leq\alpha}$ at the beginning of layer-$\alpha$ epochs only: this matrix will be denoted by $P_{\leq\alpha}$ in the following steps. 
The corresponding implementation of \cref{alg:gain_estimation} is given below.

\begin{enumerate}[start=9]
    \item For any $\alpha\in[\alpha_t]$, reset $A_{t-1,\alpha}= 0\in\Rbb^{d\times d}$, define $X_{\leq\alpha}$ as the concatenation of all vectors $\{u_{t,i},i\in J_{t,\beta},\beta\in[\alpha]\}$, and compute $P_{\leq\alpha}:=X_{\leq\alpha}X_{\leq\alpha}^\top$. Receive feedback $\ell_t^2$.
    \item {\bf If} $Z_t=1$, $\forall i\in[d],r_{t,i}\geq 0$, and $I_t\in[L]$: $A_{t,I_t}:=A_{t-1,I_t} + \frac{4}{\mu_{I_t}} \ell_t^2 \paren{(d_{t,\leq I_t}+2)w_tw_t^\top-P_{\leq I_t}}$
    \item Let $\widehat G_t\gets \1[Z_t=0]\frac{2}{\lambda_{t,i_t}} \ell_t^2 w_t w_t^\top$. For any $\alpha\in[L]$ such that $t=0\bmod 2^\alpha$ or $t=T$, do \label{item:end_implementation}
    \begin{equation}\label{eq:implementation_unpdate_G}
        \widehat G_t\gets \widehat G_t + A_{t,\alpha} - X_{\leq\alpha}(X_{\leq\alpha}^\top A_{t,\alpha}X_{\leq\alpha})X_{\leq\alpha}^\top - \sum_{i\in J_{t,\alpha}} (u_{t,i}^\top A_{t,\alpha} u_{t,i}) u_{t,i}u_{t,i}^\top. 
    \end{equation}
    
\end{enumerate}
In summary, steps \ref{item:start_implementation} to \ref{item:end_implementation} exactly implement the steps at round $t$ for \cref{alg:final_alg,alg:subspace_update,alg:layered_sampling,alg:gain_estimation}. We compute the corresponding runtime below.

\begin{proof}[Proof of \cref{prop:computational_complexity}]
    Fix a round $t\in[T]$. All steps in the implementation can be performed in $O(d^2)$ time per iteration, except for
    \begin{itemize}
        \item Step \ref{item:compute_decomposition_implementation} which involves the eigendecomposition of a $d_{t,\leq\alpha_t}\times d_{t,\leq\alpha_t}$ matrix---$O(d_{t,\leq\alpha_t}^3)$ runtime.
        \item Step \ref{item:reduce_to_smaller_matrix_bad} which involves the multiplication of $d\times d$ and $d\times d_{t,\leq\alpha_t}$ matrices when computing $X_t^\top \widehat G_{t-1} X_t$---$O(d^2 d_{t,\leq\alpha_t})$ runtime.
        \item Step \ref{item:end_implementation} when computing \cref{eq:implementation_unpdate_G} for all layers $\alpha\in[L]$ such that $t=0\bmod 2^\alpha$ or $t=T$, equivalently, $\alpha\in[\alpha_{t+1}]$ with the convention $\alpha_{T+1}:=L$. For each $\alpha\in[\alpha_{t+1}]$ the computation \cref{eq:implementation_unpdate_G} can be implemented in $O(d^2 d_{t,\leq\alpha})$ runtime.
    \end{itemize}
    Finally, we recall that under the good event from \cref{lemma:stability_eigenvalues}, with high probability (say probability $\geq 1- 1/(d^3 T)$), for all $t\in[T],\alpha\in[L]$ and $i\in J_{t,\alpha}$ one has $\lambda_{t,i}\geq \frac{3}{4}\mu_\alpha$.
    Since $U_t\in\Scal_d$ has eigenvalues $(\lambda_{t,i})_{i\in[d]}$ this implies
    \begin{equation*}
        d_{t,\leq\alpha} \leq\abs{\set{i\in[d]:\lambda_{t,i}\geq \frac{3}{4}\mu_\alpha}} \leq \frac{4}{3\mu_\alpha} = O(2^\alpha),\quad t\in[T],\alpha\in[L].
    \end{equation*}
    Denote by $\Ecal$ the corresponding event.
    Of course, we also have $d_{t,\leq\alpha}\leq d$.
    Altogether, under $\Ecal$ the runtime of \cref{alg:final_alg} on round $t$ is of order
    \begin{equation*}
        d^2\sum_{t\in[T]}\paren{d_{t,\leq\alpha_t} +\sum_{\alpha\in[\alpha_{t+1}]} d_{t,\leq\alpha}}  \lesssim d^2 \sum_{\alpha\in[L]}\paren{\min(2^{\alpha_t},d) + \min(2^{\alpha_{t+1}},d\alpha_{t+1}) }.
    \end{equation*}
    Note that $\alpha_t=\alpha$ exactly when $t=2^{\alpha}+1 \bmod 2^{\alpha+1}$ if $\alpha<L$, or $t=1\bmod 2^L$ if $\alpha=L$.
    Hence, summing the previous per-round runtime over all $t\in[T]$ yields a total computational complexity $C(d,T)$ satisfying
    \begin{equation*}
        C(d,T) \lesssim d^2\sum_{\alpha\in[L]} \frac{T}{2^\alpha}\cdot \min(2^\alpha,d\alpha) \lesssim d^2 T\paren{2\log_2 d + d\sum_{\alpha\geq \ceil{2\log_2 d}}\frac{\alpha}{2^\alpha}} \lesssim d^2 T \log d,
    \end{equation*}
    under the event $\Ecal$. On $\Ecal^c$ which has small probability ($\leq 1/(d^3 T)$), we may bound all dimensions $d_{t,\leq\alpha}\leq d$ for $\alpha\in[L]$, which gives a $O(d^3 T)$ worst-case runtime. Taking the expectation yields the claimed~$O(d^2 T\log d)$ total runtime.

    As a remark, if one aims to have $O(d^2 T\log d)$ worst-case runtime, we can always stop the procedure whenever the learner observes that the event $\Ecal$ is violated.
\end{proof}

\begin{remark}\label{remark:computational_complexity}
    Within the proof of \cref{prop:computational_complexity}, the runtime bottleneck are both the matrix multiplication in step \ref{item:reduce_to_smaller_matrix_bad} and the gain matrix computation in step \ref{item:end_implementation}. These lead to the $O(d^2 T\log d)$ runtime. We believe that this extra $\log d$ factor is an artifact of the implementation and that using a more involved representation of gain matrices/OMD iterates these steps could be performed more efficiently (for instance a separate basis representation of $E_{t,\leq\alpha}$ within $E_{t,\leq\alpha+1}$ for each $\alpha\in[L]$).
    The main computational bottleneck would then be eigendecomposition in step \ref{item:compute_decomposition_implementation} which leads to a $O(d^2 T)$ runtime. We omit details for the sake of simplicity.
\end{remark}

\section{Proof of the regret lower bound}
\label{app:lb}

The ultimate goal of this section is to carry out the proof of~\cref{thm:lb-intro}. We start with some supporting lemmas, then implement the two reductions mentioned in Section~\ref{sec:lb}, and finally prove the lower bound on the query complexity of the covariance estimation problem; the regret bound of Theorem~\ref{thm:lb-intro} then follows by setting the parameters as explained in Section~\ref{sec:lb_proof-sketch}. 
Throughout, gamma distributions are always parameterized in terms of shape and scale.

\subsection{Supporting lemmas}


The following lemma verifies correctness of the sampling procedure implemented in~\cref{alg:gain_sampling}. 

\begin{lemma}[Exponential sampling]
\label{lem:sampling}
Let $M \in \mathbb{S}^{d}_+$ be deterministic, $r \in [4,d/2]$, and let~$G$ be generated by Algorithm~\ref{alg:gain_sampling} initialized with~$(M,r)$.
Then $G \in \mathbb{S}^{d}_+$ has~$\rank(G) = r$, and
\[
        \langle G,w^\vphtop w^\top\rangle
        \sim 
        \frac{r}{d}\langle M,w^\vphtop w^\top\rangle\Exp(1).
\]
\end{lemma}

\newcommand{\bbeta}{\beta}
\newcommand{\bgamma}{\gamma}
\newcommand{\bxi}{\xi}

\begin{proof}
Recall that, by our parity convention, Algorithm~\ref{alg:gain_sampling} samples jointly independently: $\bbeta \sim\BetaDist(1,\frac{r}{2}-1)$, $\bgamma \sim\GammaDist(\frac{d}{2},1)$, and $V \sim\Unif(\Gr(r,d))$,
where~$\GammaDist$ is in the shape-scale parameterization, and returns
\[
        G=\frac{r\bbeta\bgamma}{d}
        M^{1/2}P_VM^{1/2}.
\]
Since~$r \ge 4$, the distribution of~$\bbeta$ is well-defined.
Clearly, \(G\succeq0\) and \(\rank(G) = r\) with probability~$1$.

Now, fix a unit vector \(w\in\R^{d}\) and write \(z=M^{1/2}w\).  If \(z=0\), the claim is trivial, so assume \(z\ne0\).  By rotational invariance of~$\Unif(\Gr(r,d))$, 
\[
        \frac{\|P_Vz\|^2}{\|z\|^2}\sim
        \BetaDist\left(\frac{r}{2},\frac{d-r}{2}\right).
\]
We use the closure property of beta distributions: if~\(X\sim\BetaDist(a,b)\) and \(Y\sim\BetaDist(c,a-c)\) are independent, with \(0<c<a\), then \(XY\sim\BetaDist(c,a+b-c)\).  Applying this with
$
(a,b,c) = (\frac{r}{2},\frac{d-r}{2},1),
$
\[
        \bbeta\frac{\|P_Vz\|^2}{\|z\|^2}\sim \BetaDist\paren{1,\frac{d}{2}-1}.
\]
Finally, if~\(\bxi \sim\BetaDist(1,\frac{d}{2}-1)\) and \(\bgamma\sim\GammaDist(\frac{d}{2},1)\) are independent, then \(\bxi\bgamma \sim \Exp(1)\).  Therefore
\[
        w^\top G w
        =\frac{r\bbeta\bgamma}{2d}\|P_V M^{1/2} w\|^2
        \sim \frac{r}{2d} (w^\top M w) \Exp(1).
\qedhere
\]
\end{proof}

Next, we prove an upper bound on the norm of the gain matrices constructed in~\cref{alg:adversary}.

\begin{lemma}\label{lemma:bound_operator_norm_overall}
        There is a universal $c_G>0$ such that running \cref{alg:adversary} for any $\nu\in[0,\frac{d}{p}]$ yields 
        \begin{equation*}
            \Pbb\sqb{\max_{t\in[T]}\|G_t\|_{\op} \geq c_G\log(eT/\delta)} \leq \delta.
        \end{equation*}
    \end{lemma}
    \begin{proof}
        We note that $\overline G_t$ has $p$ eigenvalues of at most $1+2\nu$ and all other eigenvalues are at most $1$ hence we can write $\overline G_t^{1/2}\preceq \sqrt{1+2\nu}P_F + P_{F^\perp}$ for some fixed subspace $F$ which depends only on $W_t$ and $E$. Therefore, 
        \begin{align*}
            \|G_t\|_{\op} 
            &\leq \frac{r}{d} \beta_t\gamma_t  ((1+2\nu)\|P_FP_{V_t}P_F\|_{\op} + 2\sqrt{1+2\nu}\|P_{V_t}P_F\|_{\op} + 1)\\
            &= \frac{r}{d} \beta_t\gamma_t (\sqrt{1+2\nu}\|P_{V_t}P_F\|_{\op}+1)^2 \overset{(i)}{\leq }\frac{4r}{d} \beta_t\gamma_t \paren{\frac{d}{p} \|P_{V_t}P_F\|_{\op}^2+1}.
        \end{align*}
        In $(i)$ we used $\nu\leq \frac{d}{p}$. 
        In the rest of the proof, we bound each term in this product. 
        We start with the last factor. Since $V_t\sim\mathrm{Unif}(\mathrm{Gr}(r,d))$, conditionally on the past history, $\|P_{V_t}P_F\|_{\op}^2$ has the same distribution as $\|P_{F_0}P_{F_1}\|_{\op}$ where $F_0$ is a fixed $r$-dimensional subspace and $F_1\sim\mathrm{Unif}(\mathrm{Gr}(p,d))$. Further,
        \begin{equation*}
            \|P_{F_0}P_{F_1}\|_{\op}^2 = \sup_{x\in S^{d-1}\cap F_0} \|P_{F_1}(x)\|^2 \overset{(i)}{\leq} 2 \sup_{x\in \Ncal_0} \|P_{F_1}(x)\|^2,
        \end{equation*}
        where $\Ncal_0$ is a $1/4$-net of $S^{d-1}\cap F_0$, which can be taken so that $|\Ncal_0|\leq 9^{r}$. In $(i)$ we used the standard quadratic form-net lemma, e.g, \cite[Lemma 4.4.1]{vershynin2018high}. Note that by construction, since $F_1\sim\mathrm{Unif}(\mathrm{Gr}(p,d))$, for any $x\in S^{d-1}$ one has $\|P_{F_1}(x)\|^2\sim\mathrm{Beta}(\frac{p}{2},\frac{d-p}{2})$. As a result, denoting $B\sim \mathrm{Beta}(\frac{p}{2},\frac{d-p}{2})$, we have
        \begin{equation*}
            \Pbb\sqb{\|P_{V_t}P_F\|_{\op}^2 \geq 2t} = \Pbb\sqb{\|P_{F_0}P_{F_1}\|_{\op}^2 \geq 2t} \leq 9^{r} \Pbb \sqb{B\geq t} \leq 9^p \Pbb \sqb{B\geq t},
        \end{equation*}
        where in the last inequality we used $r\leq p$. Next, we use \cref{lem:beta-tail-local} to bound for any $t\geq 2$,
        \begin{equation*}
            \Pbb\sqb{B\geq \frac{2p}{d}t} \leq e^{-c(d-p)} + e^{-cpt},
        \end{equation*}
        for some universal constant $c>0$.
        Of course, we also always have $B\in[0,1]$. Altogether, this implies that the event
        \begin{equation*}
            \Ecal_\delta:=\set{\forall t\in[T],\frac{d}{p}\|P_{V_t}P_F\|_{\op}^2+1 \leq c_0\paren{1 + \frac{\min(\log(T/\delta),d)}{p}}}
        \end{equation*}
        has probability at least $1-\delta$ for some choice of universal constant $c_0>0$.
        
        Next, recall that by construction $\beta_t\sim \mathrm{Beta}(1,\frac{r}{2}-1)$ and $\gamma_t\sim\mathrm{Gamma}(\frac{d}{2},1)$. Then, \cref{lem:beta-tail-local} again shows that the event $\Fcal_\delta:=\{\forall t\in[T],\beta_t\leq \frac{c_1}{r}\min(\log (T/\delta),r)\}$ has probability at least $1-\delta$
        for some universal choice of $c_1>0$, and standard Chernoff bounds for $\mathrm{Gamma}(n,1)$ show that the event $\Gcal_\delta=\{\forall t\in[T],\gamma_t\leq d+2\log(T/\delta)\}$ has probability at least $1-\delta$. Under $\Ecal_\delta\cap\Fcal_\delta\cap\Gcal_\delta$, which has probability at least $1-3\delta$,
        \begin{equation*}
            \max_{t\in[T]}\|G_t\|_{\op}\leq c_2\min(\log(T/\delta),r)\paren{1+\frac{\log(T/\delta)}{d}}\paren{1+\frac{\min(\log(T/\delta),d)}{p}} \overset{(i)}{\leq} c_3 \log(T/\delta),
        \end{equation*}
        for some universal constants $c_2,c_3>0$. In $(i)$ we used $r\leq p$. 
        This ends the proof.
    \end{proof}

The next lemma provides a {\em sample augmentation} procedure based on the one in~\cref{alg:gain_sampling}.
\begin{lemma}[Augmentation trick]
\label{lem:augmentation}
Let $X\sim \theta\Exp(1)$ with $\theta>0$.  
There is a Markov kernel that, using only~$X$ and~$\Delta > 0$, gives~$Z|X$ with marginal law $Z \sim (\theta+\Delta)\Exp(1)$.
Namely, sample~$Z|X$ by
\[
N | X \sim \operatorname{Poisson}(X/\Delta), \quad 
Z | N \sim \GammaDist(N+1,\Delta).
\]
\end{lemma}

\begin{proof}
Since all distributions mentioned in the premise are positive, it is convenient to use their Laplace transforms over~$\R_+$, and use that such distributions are defined by their Laplace transforms.
In particular, for~$\xi \sim \GammaDist(k,s)$ in the shape-scale parameterization,
$
\E[e^{-u \xi}]=(1+us)^{-k},
$
~$u \ge 0$.
Conditioning first on $N$ and then on $X$, we calculate the Laplace transform of~$Z$ from the premise:
\begin{align*}
        \E [e^{-uZ}]
        =\E [(1+u\Delta)^{-1-N}] 
        =(1+u\Delta)^{-1} \E\left[\exp\left(-\frac{uX}{1+u\Delta}\right)\right] 
        =\frac{1}{1+u(\theta+\Delta)}.
\end{align*}
On the right, we recognize the Laplace transform of $\GammaDist(1,\theta+\Delta)$, a.k.a.~$(\theta+\Delta)\Exp(1)$. 
\end{proof}

The next 3 lemmas will be used in the last step of the proof (covariance estimation lower bound).


\begin{lemma}
\label{lem:rotated-projector-moments}
Let $E \in \Gr(p,n)$, let~$P$ be an orthogonal projector of rank $r$, and write
$
        \rho :=\Tr(P_{E^\perp}P).
$
Let $U_{E^{\vphantom\perp}}$ (resp.~$U_{E^\perp}$) be Haar on the orthogonal group of $E$ (resp.~$E^\perp$), and set
$
        U:=U_{E^{\vphantom\perp}} \oplus U_{E^\perp}.
$
Then
\[
        \E[UPU^\top]
        =\frac{r-\rho}{p}P_{E^{\vphantom\perp}}+\frac{\rho}{n-p}P_{E^\perp}.
\]
Moreover, in terms of~$\widetilde P:=UPU^\top$ and $A:=\E[\widetilde P]$, we have
$
        \E[(\widetilde P-A)^2]
        = A-A^2\preceq A .
$
\end{lemma}

\begin{proof}
By construction of~$U$,~$A = \E[UPU^\top]$ commutes with any orthogonal transformation preserving~$E$ and $E^\perp$.  
By simple linear algebra, this implies that~$A$ is block-diagonal in any orthonormal basis aligned with~$E$ and~$E^\perp$, with blocks~$a I_{p}$ and~$b I_{n-p}$; in other words,
$
        A = aP_{E^{\vphantom\perp}}+bP_{E^\perp}.
$
Multiplying by $P_{E^\perp}$ and taking the trace, we get~$\tr(AP_{E^\perp}) = b\,\tr(P_{E^\perp}) = b(n-p)$. 
On the other hand, using that~$U$ is orthogonal and~$P_{E^\perp}$ is block-diagonal with blocks~$0_{p \times p}$ and~$I_{n-p}$ in the same basis, we get
$
\tr(AP_{E^\perp}) = \Tr(P_{E^\perp}P)=\rho.
$
This gives~$b = \rho/(n-p)$; similarly~$a = (r-\rho)/p$, and the first claim follows.
Finally, using that~$\widetilde P$ is itself a projector, we get
$
        \E[(\widetilde P-A)^2]
        =\E[\widetilde P^2]-A^2
        =A-A^2
        \preceq A.
$
\end{proof}

The next lemma is a spectral-gap perturbation result in the spirit of the Davis--Kahan theorem. 

\begin{lemma}
\label{lem:gap-perturbation}
Let $P$ be an orthogonal projector, and let
$
        M=aP+b(I-P)
$
with the gap~$\Delta=a-b > 0$. 
Let $S=M+N$ be symmetric, with
$
        \|N\|_{\op}\le \Delta/6,
$
and let~$\widehat P$ is the spectral projector of $S$ onto its top $r=\rank(P)$ eigenvalues. 
Then
\[
        \|\widehat P-P\|_{\op}
        \le \frac{6\|N\|_{\op}}{\Delta}.
\]
\end{lemma}

\begin{proof}
Let~$\Gamma$ be the circular contour of radius~$\Delta/3$ centered at~$a$. 
Clearly, any~$z \in \Gamma$ is at the distance at least~$\Delta/3$ from any eigenvalue of~$M$.
On the other hand, since~$\|N\|_{\op} \le \Delta/6 [< \Delta/4]$, the top-$r$ spectral cluster of $S=M+N$ is obtained by perturbing the~$a$-part of the spectrum of~$M$; 
moreover, it remains inside $\Gamma$, and any~$z \in \Gamma$ is at the distance at least
$\Delta\min(1/3-1/6, \, 2/3 - 1/6) = \Delta/6$
from any eigenvalue of~$S$.
Now, combining the Riesz formula for spectral projectors~\cite{rudin1991functional},
\[
        P=\frac{1}{2\pi i}\int_\Gamma (zI-M)^{-1}\,dz,
        \qquad
        \widehat P=\frac{1}{2\pi i}\int_\Gamma (zI-S)^{-1}\,dz,
\]
with the resolvent identity
$
        (zI-S)^{-1}-(zI-M)^{-1}
        =(zI-S)^{-1}N(zI-M)^{-1}
$
for~$z\in\Gamma$, we get
\begin{align*}
        \|\widehat P-P\|_{\op}
        \le
        \frac{|\Gamma|}{2\pi}
        \sup_{z\in\Gamma}\|(zI-S)^{-1}\|_{\op}
        \|N\|_{\op}
        \sup_{z\in\Gamma}\|(zI-M)^{-1}\|_{\op}  
        \le
        \frac{\Delta}{3} \cdot \frac{6}{\Delta} \|N\|_{\op} \frac{3}{\Delta}
        \le \frac{6\|N\|_{\op}}{\Delta}.
        \qquad\qedhere
\end{align*}
\end{proof}

Finally, we prove a lower bound on the local entropy of the Grassmannian.
For any fixed $r$-dimensional subspace $E_0$ and radius $\nu\geq 0$, we denote by $B_r(E_0,\nu)=\{E:\dim(E)=r, \|P_E-P_{E_0}\|_{\op}\leq \nu\}$ the ball centered at $E_0$ of radius $\nu$ within the space of $r$-dimensional subspaces. Similarly, we denote the sphere centered at $E_0$ of radius $\nu$ by $S_r(E_0,\nu)=\{E:\dim(E)=r, \|P_E-P_{E_0}\|_{\op}= \nu\}$.

\begin{lemma}[Local Grassmannian volume ratio]\label{lemma:ratio_volumes_bound}
    Fix $E\in\Gr(r,d)$ and $\beta\in[0,1]$. Then,
    \begin{equation*}
        \frac{\Vol(B_r(E,2\beta))}{\Vol(B_r(E,\beta))} \geq 2^{r(d-r)}.
    \end{equation*}
\end{lemma}

\begin{proof}
    Let $F$ be a Haar uniform subspace in $\mathrm{Gr}(r,d)$. We denote by $ \theta_1\geq\ldots\geq\theta_r\geq 0$ the principal angles between $E$ and $F$. Then,
    \begin{equation*}
        \frac{\Vol(B_r(E,\beta))}{\Vol(\mathrm{Gr}(r,d))} = \Pbb[\|P_E-P_F\|_{\op}\leq \beta] = \Pbb[\max_{i\in[r]}\sin^2\theta_i \leq \beta^2].
    \end{equation*}
    Denote $\lambda_i = \sin^2\theta_i$ for $i\in[r]$. 
    From \cite{absil2006largest}, 
    their probability density function is
    \begin{equation*}
        \mathrm{dens}(\lambda_1,\ldots,\lambda_r)\propto  \prod_{i<j}|\lambda_i-\lambda_j| \prod_{i\in[r]}  \lambda_i^{\frac{n-2r-1}{2}} (1-\lambda_i)^{-1/2}.
    \end{equation*}
    In particular, this density has a homogeneous part of degree $\binom{r}{2}+ r\frac{n-2r-1}{2} = \frac{r(n-r-2)}{2}$, while the second part is the product of terms $(1-\lambda_i)^{-1/2}$ which is increasing in $\lambda_i$. Hence, this shows that
    \begin{equation*}
         \frac{\Vol(B_r(E,2\beta))}{\Vol(B_r(E,\beta))} \geq 4^{\frac{r(n-r-2)}{2}} 4^r = 2^{r(n-r)}.
    \qedhere
    \end{equation*}
\end{proof}


\subsection{From Bandit PCA to subspace discovery}
\label{sec:regret-to-discovery}

We now carry out the first of the two reductions outlined in Section~\ref{sec:lb}.
To this end, we first introduce the general {\em subspace discovery problem}, in which a learner may query sequentially an intensity oracle to some gain matrix where the main signal is the projector $P_E$ of rank $p\leq d/2$, and the goal is to find $k\in[p]$ orthonormal vectors approximately in $E$.


\begin{figure}[H]
\begin{mdframed}
\begin{quote}
\begin{center}{\bf Subspace discovery problem} 
\end{center}
Sample independently a subspace $E\sim\mathrm{Unif}(\Gr(p,d))$, and $a,b\sim\mathrm{Unif}([0,\alpha])$. We refer to $(E,a,b)$ as the sampled \emph{instance}. The instance data~$(E,a,b)$ is hidden from the learner, and is used to generate the hidden matrix
\begin{equation*}
    H=(1+a)P_E-\frac{p}{d}bP_{E^\perp}.
\end{equation*}
The learner may sequentially (adaptively) query a~{\em $\nu$-intensity oracle}:
\begin{equation}
\label{eq:subspace-discovery-oracle}
w \in \Rbb^d \quad \mapsto \quad Y(w) \sim \dotp{I+\nu H}{ww^\top} \Exp(1),
\end{equation}
each query having a fresh source of randomness.
The learner's task is to output~$k$ orthonormal vectors $u_1,\ldots,u_k \in \R^d$ such that
\begin{equation}
\label{eq:subspace-discovery-condition}
\|P_E (u_i)\|^2 \ge 1-5\alpha, \qquad i \in [k].
\end{equation}
We denote this problem $\mathrm{Disc}(d,p,\alpha,\nu,k)$.\footnotemark
\end{quote}
\end{mdframed}
\end{figure}
\footnotetext{Note that the probability here is over the overall randomness, including that of sampling the instance~$(E,a,b)$. 
If we think of~$\mathrm{Disc}(d,p,\alpha,\nu,k)$ as a {\em class} of problems parameterized by~$d,p,\alpha,k$, along with the~$\nu$-discovery oracle, then the minimal number~$m$ of queries to guarantee~\eqref{eq:subspace-discovery-condition} corresponds to the {smoothed} oracle complexity of this class.}

Note that the sampling of the instance data~$(E,a,b)$ matches that of the Bandit PCA adversary defined in Algorithm~\ref{alg:adversary} (see line~\ref{line:adversary-data}). Further, note that the exponential sampling of the intensity oracle matches the sampling interface in lines~\ref{line:gain-sampling}--\ref{line:issue-reward} of the Bandit PCA adversary from Algorithm~\ref{alg:adversary}---it is implemented in Algorithm~\ref{alg:gain_sampling} and its correctness is checked in Lemma~\ref{lem:sampling}. 
The proposition below gives our first reduction.

\begin{proposition}
\label{prop:regret-discovery}
For some universal constant $c_\nu>0$, one has the following. Let $\alpha\in[0,\tfrac{1}{10}]$.
Suppose that~$\Algo$ is a Bandit PCA algorithm that attains the expected regret
\begin{equation}
\label{eq:low-regret-assumption}
        \E[\Reg_T(\Algo)]\le \frac{r\nu \alpha}{32d} T
\end{equation}
when interacting with the adversary in~\cref{alg:adversary} for parameters $(d,p,\alpha,\nu,J_{\max} = \ceil{\tfrac{\alpha p}{16}},T)$, where
\begin{equation}
\label{eq:gamma-min}
    \nu > \nu_{\min}(d,p,\alpha,T)
    :=c_\nu d\paren{\sqrt{\frac{L_T}{\alpha pT}} + \frac{L_T}{T}}.
\end{equation}
Additionally, suppose provided that $J_{\max}\geq 40$.
Then, for any $\eta\in(0,1]$, there is an algorithm for the subspace discovery problem $\mathrm{Disc}(d,p,\alpha,\nu,\ceil{J_{\max}/6})$ that uses at most $T+O(\frac{p\log (p/\eta)}{\alpha\min(\nu,1)^2})$ queries to the $\nu$-intensity oracle and has the following guarantees. (1) It succeeds with probability at least $1/4$ and (2) with probability at least $1-\eta$ it either abstains or outputs successful vectors.
\end{proposition}

\begin{proof}

The proof proceeds in three steps.
In the first step, we simulate the adaptive adversary of Algorithm~\ref{alg:adversary}, making use of the sampling augmentation trick (Lemma~\ref{lem:augmentation}). 
In the second step, we show that the low regret of~$\Algo$, as per the premise of~\cref{prop:regret-discovery}, implies a large number of triggered epochs (corresponding to the ``if''-clause in line~\ref{line:trigger-rule} of Algorithm~\ref{alg:adversary}).
In the final step, we use~$\Algo$ to construct an algorithm~$\Disc$ for the subspace discovery problem. 

\paragraph{Step 1: Simulation of the adaptive adversary.}
Let $\Algo$ be a Bandit PCA algorithm interacting with the adversary in Algorithm~\ref{alg:adversary} and satisfying~\eqref{eq:low-regret-assumption}. 
Let us simulate this interaction, yet performing an intermediate orthogonality verification step and then using~\cref{lem:augmentation} to produce the same distribution of gain matrices as the one in~\cref{alg:adversary}. 

Namely, consider the following procedure with access to~$\Algo$ and the~$\nu$-oracle~\eqref{eq:subspace-discovery-oracle}. (Line numbers refer to those in Algorithm~\ref{alg:adversary}, and internal variables are also in correspondence to those in Algorithm~\ref{alg:adversary}.)

\begin{figure}[H]
\begin{mdframed}
\begin{itemize}
\item {\bf Initialization:} inputs and initialization are the same as in Algorithm~\ref{alg:adversary} (line~\ref{line:adversary-init}). 
Additionally, initialize the candidate list~$\cU \leftarrow \{\emptyset\}$.
\item {\bf Rounds}~$t \in [T]$:
\begin{enumerate}
\item
Receive~$\Algo$'s query~$w_t \in S^{d-1}$.
\item
{\bf If}~$j = J_{\max}$ {\bf then} issue~$y_t = 0$ as the reward~$\ell_t^2$ to~$\Algo$ (cf.~line~\ref{line:expected_gain}).
\item
{\bf Else}:
    \begin{enumerate}[label=(\roman*)]
    \item Compute~$q_t:=P_{W_j^\perp}w_t$ and~$\Delta_t := 1 - \|q_t\|^2$. \label{step:orthogonality-check}
    \item Query the~$\nu$-intensity oracle~\eqref{eq:subspace-discovery-oracle} at~$q_t$ and receive 
    \[
    X_t \sim \dotp{I+\nu H}{q_tq_t^\top} \Exp(1).
    \]
    \item Sample first~$N_t | X_t \sim \operatorname{Poisson}(X_t/\Delta_t)$, then~$Y_t | N_t \sim \GammaDist(N_t+1,\Delta_t)$. \label{step:emulation}
    \item Issue~$y_t = \frac{r}{d} Y_t$ as the reward~$\ell_t^2$ to~$\Algo$.
    \end{enumerate}
\item
Proceed to lines~\ref{line:trigger-rule}--\ref{line:new-epoch-start}. If the new epoch starts ("if" condition in line~\ref{line:trigger-rule} is true), then add the normalized rounded vector~$u_j:=P_{W_j^\perp}(w_{\hat i})$, where $\boldsymbol{i}$ is produced in line~\ref{line:rounding} to the candidate list:
\[
\cU \leftarrow \cU \cup \{ u_j/\|u_j\| \}.
\]
By convention, if $u_j=0$, $u_j/\|u_j\|$ is any arbitrary unit vector in $W_j^\perp$.
\end{enumerate}
\item {\bf Verification:} for each candidate $z \in\mathcal \cU$, take
$N=C_{\rm ver} \frac{\log (p/\eta)}{\alpha^2\min(\nu,1)^2} $
queries of~$\nu$-intensity oracle~\eqref{eq:subspace-discovery-oracle} at $z$, average them, and accept $z$ if the average is at least
$
        1+\nu(1-\tfrac{7}{2}\alpha).
$
\item {\bf Output:} if there are at least $\lceil \frac{J_{\max}}{6}\rceil$ accepted candidates, return $\lceil \frac{J_{\max}}{6}\rceil$ of them arbitrarily.
\end{itemize}
\end{mdframed}
\caption{Simulation process used in the proof of Proposition~\ref{prop:regret-discovery}. }
\label{fig:simulation}
\end{figure}


Crucially, in this process, we access the hidden data~$E,a,b$ (and the derived~$H$) only through the intensity oracle~\eqref{eq:subspace-discovery-oracle}; this is possible thanks to the sample augmentation trick invoked in step~\ref{step:emulation}.
Note also that the computation in step~\ref{step:orthogonality-check} uses~$W_j$, but the corresponding variable in Algorithm~\ref{alg:adversary} is {\em not} hidden from~$\Algo$.
Finally, by Lemma~\ref{lem:augmentation} this simulation is consistent, in the following exact sense.
\begin{lemma}
\label{lem:simulation}
The transcript~$(w_t, \ell_t^2)_{t \in [T]}$ generated by the above process has the same distribution as the one obtained when~$\Algo$ interacts with the adaptive adversary in Algorithm~\ref{alg:adversary}.
\end{lemma}

\begin{proof}
By Lemma~\ref{lem:augmentation}, conditionally on~$(w_t,\Hcal_{t-1})$ where~$\Hcal_{t-1}$ is the filtration corresponding to the first~$t-1$ rounds, the measurement~$Y_t$ produced in step~\ref{step:emulation} is distributed exponentially with scale
\[
\Delta_t + q_t^\top (I + \nu H) q_t = \Delta_t + \| q_t \|^2 + \nu w_t^\top P_{W_j^\perp} H P_{W_j^\perp} w_t = w_t^\top \left(I + \nu P_{W_j^\perp} H P_{W_j^\perp}\right) w_t.
\]
Taking into account lines~\ref{line:expected_gain}--\ref{line:gain-sampling} of Algorithm~\ref{alg:adversary}, we conclude that~$y_t = \frac{r}{d} Y_t$ has the same distribution (conditionally on~$w_t,\mathcal{F}_{t-1}$) as~$\ell_t^2$ of Algorithm~\ref{alg:adversary}. From this moment, we can proceed by induction.
\end{proof}


\paragraph{Step 2: Low regret of~$\Algo$ implies many completed epochs.}
We shall invoke Lemma~\ref{lem:time-uniform-exp} on the sequence~$y_t$, with filtration~$(w_{t},\Hcal_{t})$ and conditional expectations
\[
\mu_t:=\E[y_t|\Hcal_{t-1},w_t],
\]
with~$K=J_{\max}$, with~$L =\log(32eJ_{\max})+\log\log(eT)$
(cf.~line~\ref{line:adversary-init} of Algorithm~\ref{alg:adversary}), $\tfrac{1}{32}$ as violation probability, and
$\mu_\star=2\frac{r}{p}$.
Note that the latter is possible since
$\mu_t=\frac{r}{d}w_t^\top\overline G_t w_t
        \le \frac{r}{d}\|\overline G_t\|_{\op}
        \le\frac{(1+(1+\alpha)\nu)r}{d} \leq 2\frac{r}{p},
$ where in the last inequality we used $\nu \le \frac{d}{p}\leq\frac{d}{r}$. We denote by ~$\cE$ the corresponding high-probability event in the proof of Lemma~\ref{lem:time-uniform-exp}.

Now, recall the rule triggering the start of a new epoch in line~\ref{line:trigger-rule} of Algorithm~\ref{alg:adversary}:
\begin{equation}
\label{eq:trigger-rule}
        \frac{1}{t - \tau_j} \sum_{s=\tau_j+1}^{t} \ell_s^2
        \ge
        \frac{r}{d}
        (1+\nu(1-\alpha))+ C_{\mathrm{adv}} \frac{r}{p}\paren{\sqrt{\frac{L_T}{t-\tau_j}} + \frac{L_T}{t-\tau_j}},
\end{equation}
for some universal constant~$C_{\mathrm{adv}} > 0$.
To bound the regret, consider a random comparator~$u$ such that, conditionally on the subspace~$E$,~$u$ is distributed uniformly on the unit sphere of~$E$, independently from anything else. 
\newcommand{\comp}{\textsf{comp}}
Writing~$\E_{\comp}$ for the conditional expectation over~$u \sim \Unif(S(E))$, at epoch~$j$ 
\[
\begin{aligned}
\E_{\comp}[u^\top \overline G_t u] 
= 1 + \nu \E_{\comp} \left[u^\top P_{W^\perp} H P_{W^\perp} u \right] 
&= 1 + \nu \left\langle P_{W^\perp} H P_{W^\perp}, \E_{\comp} [u u^\top]\right\rangle \\
&= 1 + \frac{\nu}{p} \tr \left(P_{W^\perp} H P_{W^\perp} P_E \right). 
\end{aligned}
\]
Plugging in~$H = (1+a) P_E - b\frac{p}{d} P_{E^\perp}$, we get
\begin{equation}
\label{eq:comp-main-expansion}
\frac{1}{p} \tr \left( P_{W^\perp} H P_{W^\perp} P_E \right) 
= \frac{1+a}{p} \tr \left( P_{W^\perp} P_E P_{W^\perp} P_E \right) - \frac{b}{d} \tr \left(P_{W^\perp} P_{E^\perp} P_{W^\perp} P_E \right). 
\end{equation}
Now, let us express the two traces in terms of~$\tr(P_E) = p$ and the traces of~$P_W P_E$ and~$P_W P_E P_W P_E$.
For the first one, we get 
\begin{align}
\tr \left( P_{W^\perp} P_E P_{W^\perp} P_E \right) 
&= \tr \left( (I-P_W) P_E (I-P_W) P_E \right) \notag \\
&= p - 2\tr (P_W P_E) + \tr \left( P_W P_E P_W P_E \right) \label{eq:comp-first-trace-expansion} \\
&\ge p - 2\tr (P_W P_E), \label{eq:comp-first-trace-inequality}
\end{align}
and for the second one
\begin{align}
\tr \left(P_{W^\perp} P_{E^\perp} P_{W^\perp} P_E \right) 
&\;= \tr \left(P_{W^\perp} P_E \right) - \tr \left(P_{W^\perp} P_E P_{W^\perp} E \right) \notag\\
&\;= p - \tr \left(P_{W} P_E \right) - \tr \left(P_{W^\perp} P_E P_{W^\perp} E \right) \notag\\
&\stackrel{\eqref{eq:comp-first-trace-expansion}}{=} 
    \tr (P_{W} P_E) - \tr \left( P_W P_E P_W P_E \right) \notag \\
&\;\le\; \tr (P_{W} P_E). \label{eq:comp-second-trace-inequality}
\end{align}
Plugging the inequalities in~\eqref{eq:comp-first-trace-inequality} and~\eqref{eq:comp-second-trace-inequality} in~\eqref{eq:comp-main-expansion}, we get:
\begin{align*}
\frac{1}{p} \tr \left( P_{W^\perp} H P_{W^\perp} P_E \right)  
&\ge \frac{1+a}{p} (p - 2\tr (P_W P_E)) - \frac{b}{d} \tr (P_{W} P_E) \\
&\overset{(i)}{\geq} (1+a)\paren{1-\frac{\alpha}{4}}-\frac{\alpha b}{8}.
\end{align*}
In $(i)$ we used the fact $\tr (P_W P_E) \le \tr (P_W) = j \le J_{\max} \le\frac{\alpha p}{8}$.

Fix $t\in[T]$ and denote by $\Hcal_t$ the history up to that round included. In summary, taking the expectation over $a$ and $b$, we obtained
\begin{align}
    \E_\comp \Ebb[u^\top G_t u \mid E,\Hcal_{t-1}] &= \frac{r}{d}\sqb{1+\nu\E_\comp \Ebb[u^\top P_{W_t^\perp} H_t P_{W_t^\perp} u \mid E,\Hcal_{t-1}]} \notag\\
    &\geq \frac{r}{d} \sqb{1+\nu\paren{\paren{1+\frac{\alpha}{2}}\paren{1-\frac{\alpha}{4}} - \frac{\alpha^2}{16} }}\geq \frac{r}{d}(1+\nu(1+\alpha/5)). \label{eq:comparator-reward}
\end{align}
In the last inequality we used $\alpha\leq 1/10$.

On the other hand, let us upper-bound the conditional mean reward of~$\Algo$'s sequence~$(w_{t})_{t \in [T]}$ on the event~$\cE$. 
Consider one completed epoch: $\mathcal I_j =\{\tau_j+1,\ldots,\tau_{j+1}\}$.
All proper prefixes of this epoch, i.e.~all~$\{ \tau_j+1, \dots, t\}$ with~$\tau_j + 1 \le t < \tau_{j+1}$, failed the test~\eqref{eq:trigger-rule} and we denote by $n_j:=\tau_{j+1}-\tau_j-1$ their number. 
Moreover, by Lemma~\ref{lem:time-uniform-exp}, on~$\cE$ we have, simultaneously over all~$j \in [J_T-1]$, that
\begin{equation}
\label{eq:MDS-bound}
\frac{1}{n_j} \left| \sum_{s=\tau_j+1}^{\tau_{j+1}-1} (\mu_s-y_s) \right|
\le 2C\frac{r}{p}\left(\sqrt{\frac{L}{n_j}}+\frac{L}{n_j}\right).
\end{equation}
Combining~\eqref{eq:MDS-bound} with the negation of~\eqref{eq:trigger-rule} for~$t = \tau_{j+1}$, and taking~$C_{\mathrm{adv}}$ larger than~$\max(2C\log(16),1)$ in~\eqref{eq:MDS-bound}, we get
\[
        \frac1{n_j}\sum_{s=\tau_j+1}^{\tau_{j+1}-1}\mu_s
        \le
        \frac{r}{d}
        (1+\nu(1-\alpha))+2C_{\mathrm{adv}}\frac{r}{p}\left(\sqrt{\frac{L_T}{n_j}}+\frac{L_T}{n_j}\right).
\]
For the triggering round itself ($t = \tau_{j+1}$), we can simply bound~$\mu_t\le \mu_\star=2$.  
The final epoch~$j = J_T$ is treated in the same way (without the final round).
As the result, summing over~$j \in [J_T]$ and invoking Cauchy--Schwarz in the form
\[
\sum_{j \in [J_{\max}]} \sqrt{n_j} 
\le \sqrt{ J_{\max} \sum_{j \in [J_{\max}]} n_j }
\le \sqrt{J_{\max} T} ,
\]
we get
\begin{align}
\label{eq:learner-mean-bound}
        \sum_{t \in [T]} \mu_t
        &\le
        \frac{r}{d} 
         (1+\nu(1-\alpha))T  +  2C_{\mathrm{adv}}\frac{r}{p}\paren{\sqrt{J_{\max}T L_T}+J_{\max}L_T+J_{\max}}.
\end{align}
Since $J_{\max} =\ceil{\tfrac{\alpha p}{16}}$, the assumption in~\eqref{eq:gamma-min} implies (for large enough absolute constant~$c_\nu > 0$) that 
\begin{equation}
\label{eq:stat-margin-bound}
        2C_{\mathrm{adv}}\left(\sqrt{J_{\max}T L_T}+J_{\max}L_T+J_{\max}\right)
        \le \frac{p\nu\alpha}{5d}  T.
\end{equation}
Now, define the event $\cB := \{J_T \ge J_{\max}\}$.
We are about to show a lower bound on~$\Pp(\cE \cap \cB).$
To that end, by combining the comparator reward lower bound~\eqref{eq:comparator-reward} with the upper bound on the learner's regret from~\eqref{eq:learner-mean-bound}--\eqref{eq:stat-margin-bound},
we get
\begin{align*}
    \Ebb_\comp \Ebb\sqb{\sum_{t\in[T]}( u^\top G_t u -\mu_t)} \geq \frac{r\nu\alpha}{d} T\cdot \paren{\Pbb[\Ecal\cap\Bcal^c] - \frac{4}{5}(1-\Pbb[\Ecal\cap\Bcal^c])}.
\end{align*}
Here we used the fact that $\mu_t\leq \frac{r}{d} w_t^\top \overline G_t w_t \leq \frac{r}{d}(1+\nu(1+\alpha))$. In particular, if $\Pbb[\Ecal\cap\Bcal^c]\geq \tfrac{1}{2}-\frac{1}{32}$, we have
\begin{equation*}
    \Reg_T > \frac{r\nu\alpha}{32d}T,
\end{equation*}
contradicting \cref{eq:low-regret-assumption}. As a result, we obtained
\begin{equation}
\label{eq:many-triggers-strong}
\Pp(\cE \cap \cB) = \Pp(\cE) - \Pp(\cE \cap \cB^c) \ge 1-\frac{1}{32}-\paren{\frac{1}{2}-\frac{1}{32}} =\frac{1}{2}.
\end{equation}

\paragraph{Step 3: Completed epochs yield verifiable discoveries by~$\Disc$.}
We now analyze the candidate list $\cU$ formed by the simulation procedure (see~\cref{fig:simulation}), as well as its final verification step.
First, we show that on the event~$\cE$, every triggered epoch gives a constant-probability rounded candidate whose Rayleigh quotient with respect to~$H$ is close to 1.
Second, we show that the verification step allows to identify such candidates with high probability.

Fix a completed epoch~$\mathcal{I} = \mathcal I_j = \{\tau_j + 1, \dots \tau_{j+1}\}$ and let~$W=W_j$ be its explored subspace. 
Put
\[
        q_s:=P_{W^\perp}w_s,
        \qquad
        Q_{\mathcal I}:=\frac1{|\mathcal I|}\sum_{s\in\mathcal I}q_s^\vphtop q_s^\top.
\]
Since~$\mathcal{I}$ was completed, \eqref{eq:trigger-rule} holds at its final round.
On~$\cE$, since~$C_{\mathrm{adv}}$ in~\eqref{eq:trigger-rule} is larger than~$2C$ in~\eqref{eq:MDS-bound},
\[
        \frac1{|\mathcal I|}\sum_{s\in\mathcal I}\mu_s
        =\frac{r}{d}\left(1+\nu\langle H,Q_{\mathcal I}\rangle\right)
        \ge
        \frac{r}{d}\left(1+\nu(1-\alpha)\right),
\]
whence
\begin{equation}
\label{eq:average-certificate}
        \langle H, Q_{\mathcal I}\rangle\ge 1-\alpha.
\end{equation}

Now, let $\boldsymbol{i}_j \sim \Unif(\mathcal{I})$ be the random index used by the simulation procedure in~\cref{fig:simulation} for this epoch (see~line~\ref{line:rounding} of Algorithm~\ref{alg:adversary}). 
Since
$
        \langle H,q_s^\vphtop q_s^\top\rangle
        \le (1+\alpha)\|q_s\|^2
        \le 1+\alpha
$
for each~$s\in\mathcal I$, the deficits
\[
        \xi_s^{(j)} := (1+\alpha)-\langle H,q_s^\vphtop q_s^\top\rangle
\]
are nonnegative, and \eqref{eq:average-certificate} gives $\Ebb_{\boldsymbol{i}_j \sim \Unif(\mathcal{I})}[\xi^{(j)}_{\boldsymbol{i}_j}]\leq 2\alpha$. We now introduce the indicator
\begin{equation*}
    \zeta_j:=\1\set{\xi^{(j)}_{\boldsymbol{i}_j} \leq 2\Ebb_{\boldsymbol{i}_j \sim \Unif(\mathcal{I})}[\xi^{(j)}_{\boldsymbol{i}_j}]}.
\end{equation*}
Note that under $\Ecal$ for any completed epoch $\Ical_j$, the previous arguments show that if $\zeta_j=1$ then $\dotp{H}{q_{\boldsymbol{i}_j}q_{\boldsymbol{i}_j}^\top}\geq 1-3\alpha$. In particular, the corresponding rounded vector $z=q_{\boldsymbol{i}_j}/\|q_{\boldsymbol{i}_j}\|$ satisfies
\begin{equation}
\label{eq:good-rounded-direction}
        \langle H,z^\vphtop z^\top\rangle
        =
        \frac{\langle H,q_{\boldsymbol{i}}^\vphtop q_{\boldsymbol{i}}^\top\rangle}{\|q_{\boldsymbol{i}}\|^2}
        \ge 1-3\alpha.
\end{equation}
Next, by Markov's inequality, the sequence $(\zeta_j-\tfrac{1}{2})_j$ is a submartingale increment sequence. Therefore, Azuma-Hoeffding's inequality implies that the event
\begin{equation*}
    \Ccal:=\set{\sum_{j\in[J_T]} \zeta_j \geq \frac{J_T}{2}-\frac{J_{\max}}{3}}
\end{equation*}
has probability at least $1-e^{-J_{\max}/18}\geq 1-\tfrac{1}{8}$ by assumption on $J_{\max}$. Note that under $\Ecal\cap\Bcal\cap\Ccal$, we showed that \cref{eq:good-rounded-direction} holds for at least $\tfrac{1}{2}J_T-\tfrac{1}{3}J_{\max}=\tfrac{1}{6}J_{\max}$ of the candidates $z\in\Ucal$.

It remains to check that the verification step in the simulation procedure (Fig.~\ref{fig:simulation}) accepts the good candidates and rejects the bad ones, simultaneously over the entire~$\cU$.
For a fixed $z \in \cU$, define
\[
        \mu(z):=1+\nu \langle H,z^\vphtop z^\top\rangle.
\]
Each~$Y_i (z)$ has law $\mu(z)\Exp(1)$, where~$\mu(z) \leq 1+2\nu$. 
The fixed-time exponential Bernstein bound gives, for
$N=C_{\rm ver}\paren{\frac{\max(1,\nu)}{\alpha\nu}}^2 \log (p/\eta)$
repeated queries~$Y_1(z),\dots Y_N(z)$ with a sufficiently large constant~$C_{\rm ver}$, 
\begin{equation}
\label{eq:verification-fixed-candidate}
        \Pp\left\{
        \left|\frac{1}{N}\sum_{i\in[N]}Y_i(z)-\mu(z)\right|
        >\frac{\alpha\nu}{2}
        \right\}
        \le \frac{\eta}{d}.
\end{equation}
Together with the union bound, since there can be at most $J_T\leq J_{\max}\leq p$ candidate vectors, on an event $\Dcal$ of probability at least $1-\eta$, the events from \cref{eq:verification-fixed-candidate} hold for all candidates~$z\in\Ucal$ simultaneously.

Next, recall that the acceptance threshold in the simulation procedure is
$1+\nu(1-\tfrac{7}{2}\alpha)$. Hence, under $\Dcal$, whenever $z\in\Ucal$ satisfies $\langle H,z^\vphtop z^\top\rangle \geq 1-3\alpha$, this candidate is accepted since $\mu(z) \geq 1+\nu(1-3\alpha)-\alpha\nu/2$. Conversely, if a candidate $z\in\Ucal$ is accepted, under $\Dcal$ we must have $\mu(z)\geq 1+\nu(1-\tfrac{7}{2}\alpha) - \alpha\nu/2 = 1+\nu(1-4\alpha)$.

In summary, under $\Dcal$ which has probability at least $1-\eta$, the procedure either abstains or outputs $\ceil{J_{\max}/6}$ orthonormal vectors $(u_i)_{i\in[k]}$ satisfying
\begin{equation*}
    \|P_E(u_i)\|^2 \geq \frac{\dotp{H}{u_iu_i^\top}}{1+a} \geq \frac{1-4\alpha}{1+\alpha} \geq 1-5\alpha,
\end{equation*}
and hence it succeeds.

Additionally, under $\Ecal\cap\Bcal\cap\Ccal\cap\Dcal$, at least $\ceil{J_{\max}/6}$ candidates in $\Ucal$ satisfy \cref{eq:good-rounded-direction} and hence the above arguments imply that at least these candidates are accepted, and in turn the procedure $\Disc$ succeeds. 
Combining \eqref{eq:many-triggers-strong} and the failure probability for $\Ccal$ and $\Dcal$, this success probability for the procedure $\Disc$ is at least
\[
        \frac12-\frac18-\eta = \frac38-\eta.
\]
Without loss of generality we may assume $\eta\leq \frac{1}{8}$ and hence the success probability is at least~$\frac{1}{4}$. The simulation stage uses at most $T$ oracle queries, and the verification stage uses at most
$J_{\max} N=O(\alpha p N)$
additional queries.  This proves the proposition.
\end{proof}

\subsection{From subspace discovery to covariance estimation}
\label{sec:disc-to-cov}

For simplicity, from now we fix the value $\alpha=1/40$. We recall that in the subspace discovery problem, the learner needs to find $m\in[p]$ vectors within the unknown space $E$. To simplify the analysis, we further reduce the problem to the case when the learner needs to fully recover the subspace $E$. This gives the following simplified problem.

\begin{figure}[H]
\begin{mdframed}
\begin{quote}
\begin{center}
\textbf{Covariance estimation problem} 
\end{center}
As in the subspace discovery problem, sample independently $E\sim\mathrm{Unif}(\mathrm{Gr}(p,d))$, and $a,b\sim \mathrm{Unif}([0,\alpha])$. Define $H=(1+a)P_E-\frac{p}{d}b P_{E^\perp}$.
The learner has access to the same $\nu$-intensity oracle which, when queried at~$w\in\Rbb^d$, returns a sample 
\[
Y(w)\sim \dotp{I+\nu H}{ww^\top} \mathrm{Exp}(1).
\]

The learner's task is to output a subspace $\widehat E$ of dimension $p$ such that
\begin{equation*}
    \|P_{\widehat E}-P_E\|_{\op} \leq \beta,
\end{equation*}
for some fixed precision parameter $\beta>0$, with sequential access to the $\nu$-intensity oracle.\footnotemark{}
We denote this problem as $\mathrm{CovEst}(d,p,\nu,\beta)$.
\end{quote}
\end{mdframed}
\end{figure}
\footnotetext{Here, we have the same remark (regarding smoothed complexity) as in the case of the subspace discovery problem.}

We reduce the subspace discovery problem to the covariance estimation problem below.

\begin{lemma}\label{lemma:reduction_to_case_k_tilde_d}
    There is a universal constant $c_n>0$ such that the following holds. Fix $k\in[p]$ and $\beta\in(0,1]$. Suppose that there is an algorithm for the subspace discovery problem $\mathrm{Disc}(d,p,\alpha,\nu,k)$ with $m$ oracle queries, which succeeds with probability at least $q$, and with probability at least $1-\eta$ it either abstains or outputs successful vectors. Then there is an algorithm for the covariance estimation problem $\mathrm{CovEst}(d,p,\nu,\beta)$ with at most $nm$ queries where $n\leq c_n\frac{p\log d}{\beta^2 kq}$, which succeeds with probability at least $2/3-n\eta$.
\end{lemma}

The main intuition for this result is that up to applying a uniform rotation, a subspace discovery algorithm approximately finds a uniformly distributed $k$-dimensional subspace of $E$. Repeating this procedure $\tilde O(p/k)$ times then recovers the full $r$-dimensional space $E$.

\begin{proof}
    Fix an algorithm $\texttt{alg}$ for the subspace discovery problem with $k\in[p]$, $m$ queries and success probability $q$. We start by making the output of the algorithm invariant by $E$-preserving rotations.
    Specifically, we consider the algorithm which first samples a uniform Haar rotation $R$ then emulates $\texttt{alg}$ up to the rotation $R$. Specifically, Whenever the procedure makes a query $w$ to the, we query the vector $Rw$ and return the observed scalar to the original algorithm. If the procedure outputs $z_1,\ldots,z_k$, the symmetrized procedure outputs $R^{-1}z_1,\ldots,R^{-1}z_k$. We denote by $\widetilde{\texttt{alg}}$ this procedure.

    By construction, $\widetilde{\texttt{alg}}$ emulates the run of $\texttt{alg}$ for the instance with the subspace $\widetilde E:=R(E)$, which is still distributed as a uniform $p$-dimensional subspace conditionally on $E$. Hence, the output of $\widetilde{\texttt{alg}}$ enjoys the same average-case guarantees as $\texttt{alg}$.

    Next, fix any $E$-preserving rotation $R_0$, that is $R_0(E)=E$. Further, since $R_0^{-1}(E)=E$, running $\texttt{alg}$ with the rotation $R$ or $RR_0^{-1}$ yields exactly the same stochastic output $z_1,\ldots,z_k$. Hence, when using rotation $R$ and $RR_0^{-1}$ respectively, the final output of $\widetilde{\texttt{alg}}$ is $R^{-1}(z_i)$ for $i\in[k]$ and $R_0R^{-1}(z_i)$ for $i\in[k]$ respectively. Since $RR_0^{-1}$ is still distributed as a uniform rotation, this shows that the final output of $\widetilde{\texttt{alg}}$ is stochastically invariant by $R_0$.

    \paragraph{Construction of the covariance estimation procedure.} We are now ready to construct a procedure for the covariance estimation problem. Specifically, we sequentially run $\widetilde{\texttt{alg}}$ for $n=\ceil{\tfrac{2^{10}C_1^2p \log d}{\beta^2 kq}}$ independent times, where $C_1\geq 1$ is a universal constant that will be specified later. We then focus on the runs that produced an output: we denote by $n'$ their number and denote by $Z_i:=[z_{i,1},\ldots,z_{i,k}]$ their outputs ($i\in[n']$). Finally, we run PCA on the combined covariance matrix
    \begin{equation*}
        \sum_{i\in[n']} Z_iZ_i^\top
    \end{equation*}
    and output the span of its top $r$ eigenvectors $\widehat E$. By construction, the procedure uses $\leq nm$ queries.

    \paragraph{Performance analysis.} Throughout, we will reason conditionally on $E$, since $\widetilde{\texttt{alg}}$ conditionally enjoys the same guarantees as $\texttt{alg}$. By the union bound and the guarantee of $\texttt{alg}$, under an event $\Ecal$ of probability at least $1-n\eta$, all outputs $Z_i=[z_{i,1},\ldots,z_{i,k}]$ for $i\in[n']$ are successful. This implies that under $\Ecal$ for all $i\in[n']$, one has
    \begin{equation*}
        \tr(P_{E^\perp}Z_iZ_i^\top) = \sum_{j\in[k]} \|P_{E^\perp}(z_{i,j})\|^2 \leq 5\alpha k.
    \end{equation*}
    Additionally, by the guarantee on $\texttt{alg}$, each independent run succeeds with probability at least $q$. Hence, $n'$ stochastically dominates a Binomial $Y\sim \mathrm{Binom}(n,q)$. In particular, on an event $\Fcal$ of probability at least $1-e^{-qn/8}$ one has $n'\geq qn/2$.

    From the rotation invariance of $\widetilde{\texttt{alg}}$ proved above, the outputs $Z_i$ for $i\in[n']$ are stochastically identical to $R_iZ_i$ for $i\in[n']$ where $R_i$ are independent uniform $E$-invariant rotations. Specifically, $R_i=R_{i,E}R_{i,E^\perp}$ where $R_{i,E}$ is a uniform rotation of $E$ and $R_{i,E^\perp}$ is an independent uniform rotation of $E^\perp$. Note that these rotations $R_i$ for $i\in[n']$ are without loss of generality independent of the events $\Ecal,\Fcal$ and all other random variables except $E$.

    We therefore reason conditionally on $\Hcal:=\sigma(n',Z_1,\ldots,Z_{n'})$ and we suppose that $\Ecal,\Fcal$ hold. Hence, in the next steps, all expectations are meant with respect to $R_i$ for $i\in[n']$ only. Denote $P_i:=(R_iZ_i)(R_iZ_i)^\top$ and $S:=\sum_{i\in[n']}P_i$. Then, by Lemma~\ref{lem:rotated-projector-moments}, we have
    \begin{equation*}
        \Ebb[S\mid \Hcal]=\sum_{i\in[n']} \Ebb[P_i\mid\Hcal] = aP_E + bP_{E^\perp},
    \end{equation*}
    where
    \begin{equation*}
        a = \frac{kn'}{p} -\frac{1}{p}\sum_{i\in[n']} \tr(P_{E^\perp}Z_iZ_i^\top) \quad \text{and} \quad b=\frac{1}{d-p}\sum_{i\in[n']} \tr(P_{E^\perp}Z_iZ_i^\top).
    \end{equation*}
    Hence, under $\Ecal\cap\Fcal$, the gap between the eigenvalues in $E$ and $E^\perp$ is precisely
    \begin{equation}\label{eq:gap_two_subspaces}
        a-b \geq \frac{kn'}{p} - \frac{5\alpha kn'}{p} - \frac{5\alpha kn'}{d-p}\geq \frac{kn'}{2p}=:\Delta,
    \end{equation}
    since $\alpha= 1/40$.
    Now note that conditionally on $\Hcal$, the random matrices $P_{F_i}-\Ebb[P_{F_i}\mid\Hcal]$ for $i\in[n']$ are independent and operator norm at most $1$. Also,
    \begin{equation*}
        \sum_{i=1}^{n'}\mathrm{Var}[P_i \mid\Hcal] \preceq \sum_{i=1}^{n'} \Ebb[P_i\mid\Hcal] \preceq \frac{kn'}{p} P_E + \frac{5\alpha kn'}{d-p} P_{E^\perp}\preceq \frac{kn'}{p}I. 
    \end{equation*}
    Hence, the matrix Bernstein inequality in Lemma~\ref{lem:tropp-bernstein} implies that with probability $1-1/8$ conditionally on $\Hcal$ one has for some universal constant $C_1$ (defined within \cref{lem:tropp-bernstein}):
    \begin{equation*}
        \norm{S - (aP_E+bP_{E^\perp}) }_{\op} \leq C_1\paren{\sqrt{\frac{kn'}{p}\log d} + \log d} \overset{(i)}{\leq} \frac{\beta\Delta}{6}.
    \end{equation*}
    In $(i)$ we used $\Fcal$ and the definition of $n$. We denote by $\Gcal$ this event. Then, under this event, \cref{lem:gap-perturbation} applied to $M=aP_E+bP_{E^\perp}$ shows that 
    \begin{equation*}
        \|P_E-P_F\|_{\op} \leq \beta.
    \end{equation*}
    Altogether, this shows that the procedure succeeds under $\Ecal\cap\Fcal\cap\Gcal$ which has probability at least $1-n\eta-e^{-qn/8}-1/8\geq 2/3-n\eta$. This ends the proof.
\end{proof}

\subsection{Lower bound for covariance estimation}
\label{sec:lb-cov}

In the last step of the proof we show a query lower bound for the covariance estimation problem, using ideas inspired from \cite{chen2023does}.

\begin{lemma}\label{lemma:query_lower_bound_subspace_estimation}
    There is a universal constant $\beta_0\in(0,1]$ such that the following holds. Fix $p\in[32\log d, \frac{d}{2}]$. For any $\nu\in[0,\frac{d}{p}]$, any algorithm for the covariance estimation problem $\mathrm{CovEst}(d,p,\nu,\beta_0)$ uses at least $\beta_0\frac{p^2 d}{\min(\nu,1)^2}$ queries or succeeds with probability at most $1/3$.
\end{lemma}

\begin{proof}
    We suppose that $\beta\leq \alpha/40$ throughout this proof.
    We denote by $w_1,\ldots,w_m$ the queries made by the algorithm and $Y_1,\ldots,Y_m$ the corresponding oracle responses for $m\geq 1$ queries. Without loss of generality, we may assume $w_1,\ldots,w_m\in S^{d-1}$ since the oracle gives the same information irrespective of the norm of the query (except if the query is $0$ in which case no information is released).
    We recall the notations $B_p(E_0,\nu)$ and $S_p(E_0,\nu)$ respectively for the ball and sphere within $\mathrm{Gr}(p,d)$ centered at $E_0$ and of radius $\nu$. For convenience, we will drop the indices $p$ throughout this proof. Note that $S(E_0,\nu)$ is invariant by any rotation preserving $E_0$ and $E_0^\perp$. Note that for any $E\in S(E_0,\nu)$ we have
    \begin{equation}\label{eq:relation_to_rho_variable}
        \tr(P_EP_{E_0^\perp}) = \tr((P_E-P_{E_0})P_{E_0^\perp}) \leq \|P_E-P_{E_0}\|_*\overset{(i)}{\leq} 2p\|P_E-P_{E_0}\|_{\op}\leq 2p\nu,
    \end{equation}
    where in $(i)$ we used the fact that $P_E-P_{E_0}$ has rank at most $2p$.

    Next, fix a subspace $E$ of dimension $p$ and denote $\rho:=\tr(P_EP_{E_0^\perp})/p\in[0,1]$. Let $R$ be a uniform rotation that preserves $E_0$ and $E_0^\perp$ and consider the random variable $F=RE$. From \cref{lem:rotated-projector-moments}, we have
    \begin{equation}
        \Ebb[P_{F}] = (1-\rho)P_{E_0} + \frac{\rho p}{d-p} P_{E_0^\perp}.
    \end{equation}
    For convenience, we introduce the notation $H_{E,a,b}:=(1+a)P_E - \frac{p}{d} bP_{E^\perp}=(1+a+\frac{p}{d}b)P_E-\frac{p}{d}bI$. Then, for any $a,b\in\Rbb$, the previous inequality implies
    \begin{equation}\label{eq:identity_mean_ok}
        \Ebb_F[H_{F,a,b}] = H_{E_0,a_0,b_0},
    \end{equation}
    where
    \begin{equation}\label{eq:definition_adjustments}
        a_0:=a+ \rho\paren{1+a+\frac{p}{d}b},\quad \text{and} \quad b_0:=b-\frac{d\rho}{d-p}\paren{1+a+\frac{p}{d}b}.
    \end{equation}
    In particular, since $S(E_0,\nu)$ is invariant by $E_0,E_0^\perp$-preserving rotations, \cref{eq:identity_mean_ok} also holds if $F\sim\mathrm{Unif}(S(E_0,\nu))$ for any $\nu>0$.

    We now fix $a_0,b_0\in[20\beta,\alpha-20\beta]$ and for any $\rho \in[0,4\beta]$ we denote by $a_\rho$ and $b_\rho$ the quantities $a$ and $b$ satisfying \cref{eq:definition_adjustments} for the parameters $a_0,b_0,\rho$. An important remark is that we still have $a_\rho,b_\rho\in[0,\alpha]$. Indeed, $\rho(1+2\alpha),\frac{d\rho}{d-p}(1+2\alpha)\leq 5\rho\leq 20\beta$ and as a result, using \cref{eq:definition_adjustments} we have $|a_\rho-a_0|,|b_\rho-b_0|\leq 20\beta$. More precisely, by continuity, for any $b\in[0,\alpha]$ we can check that there is $a(b)\in[a_0-20\beta,a_0]$ which satisfies the first equality of \cref{eq:definition_adjustments} (recall that $a_0-20\beta\geq 0$ by assumption). Note that $a(b)$ is also continuous for $b\in[0,\alpha]$. Using the continuity of $a(b)$ we can then check that there also exists $b\in[b_0,b_0+20\beta]$ (recall that $b_0+20\beta\leq \alpha$) for which $a(b)$ and $b$ satisfy both constraints of \cref{eq:definition_adjustments}, ending the proof of the claim $a_\rho,b_\rho\in[0,\alpha]$.

    Next, we introduce the random variables
    \begin{equation*}
        E\sim\mathrm{Unif}(B(E_0,2\beta))\quad \text{and}\quad  (a,b)=(a_\rho,b_\rho), \quad \text{where} \quad \rho:=\frac{\tr(P_EP_{E_0^\perp})}{p}.
    \end{equation*}
    Importantly, by \cref{eq:relation_to_rho_variable} we have $\rho\leq 4\beta$ and as a result, the previous paragraph implies that we always have $a,b\in[0,\alpha]$.
    We aim to study the posterior placed on the ball $B(E_0,2\beta)$ after the learner makes all queries. To do so, we denote by $p(E,\alpha_E,\alpha_{E^\perp}\mid w_{1:m},Y_{1:m})$ the posterior density distribution for these variables after making all queries and observing their response. Note that initially, by construction the distribution is uniform on $\mathrm{Gr}(p,d)\times[0,\alpha]^2$. 
    Therefore,
    \begin{equation}
        p(E,a,b\mid w_{1:m},Y_{1:m}) \propto \prod_{i\in[m]} p(Y_i\mid E,a,b,w_i) \overset{(i)}{\propto} \prod_{i\in[m]} \frac{\exp(-Y_i/(1 +\nu\dotp{H_{E,a,b}}{w_iw_i^\top}))}{1+\nu\dotp{H_{E,a,b}}{w_iw_i^\top}}\label{eq:posterior_calculation_new_proof}
    \end{equation}
    In $(i)$ we used $\|w_i\|=1$ and the definition of the response $Y_i$. Then, by Jensen's inequality,
\begin{align}
    &\log \Ebb_{E,a,b}\sqb{  \frac{ p(E,a,b\mid w_{1:m},Y_{1:m}) }{p(E_0,a_0,b_0\mid w_{1:m},Y_{1:m})}} \notag\\
    &\geq \Ebb_{E,a,b}\sqb{ \log \frac{ p(E,a,b\mid w_{1:m},Y_{1:m}) }{p(E_0,a_0,b_0\mid w_{1:m},Y_{1:m})}}\notag \\
    &\overset{(i)}{=}  \sum_{i\in[m]} Y_i\cdot \Ebb_{E,a,b}\sqb{\frac{1}{1+\nu\dotp{H_{E_0,a_0,b_0}}{w_iw_i^\top} }-\frac{1}{1+\nu\dotp{H_{E,a,b}}{w_iw_i^\top}}} \notag\\
        &\quad\, -\sum_{i\in[m]} \Ebb_{E,a,b}\sqb{\log\frac{1+\nu\dotp{H_{E,a,b}}{w_iw_i^\top}}{1+\nu\dotp{H_{E_0,a_0,b_0}}{w_iw_i^\top}}} \notag\\
    &\overset{(ii)}{\geq}  \sum_{i\in[m]} \frac{\nu Y_i}{1+\nu\dotp{H_{E_0,a_0,b_0}}{w_iw_i^\top}}\cdot \Ebb_{E,a,b}\sqb{\frac{\dotp{H_{E,a,b}-H_{E_0,a_0,b_0}}{w_iw_i^\top}}{1+\nu\dotp{H_{E,a,b}}{w_iw_i^\top} }} \notag\\
    &\overset{(iii)}{=} -\sum_{i\in[m]} \frac{\nu^2 Y_i}{(1+\nu\dotp{H_{E_0,a_0,b_0}}{w_iw_i^\top})^2}\cdot \Ebb_{E,a,b}\sqb{\frac{\dotp{H_{E,a,b}-H_{E_0,a_0,b_0}}{w_iw_i^\top}^2}{1+\nu\dotp{H_{E,a,b}}{w_iw_i^\top} }} \notag\\
    &\overset{(iv)}{\geq} -2\nu^2 \sum_{i\in[m]} \frac{Y_i}{1+\nu \dotp{H_{E_0,a_0,b_0}}{w_iw_i^\top}}\cdot \Ebb_{E,a,b}\sqb{\frac{\dotp{H_{E,a,b}-H_{E_0,a_0,b_0}}{w_iw_i^\top}^2}{(1+\nu \|P_E(w_i)\|^2)(1+\nu\|P_{E_0}(w_i)\|^2)}} . \label{eq:max_progress_made}
\end{align}
In $(i)$ we used \cref{eq:posterior_calculation_new_proof}.
In $(ii)$ we used the concavity of $\log$ and the fact that $\Ebb_F[H_{F,a,b}]=H_{E_0,a_0,b_0}$ by \cref{eq:identity_mean_ok} to fully remove the second sum. In $(iii)$, we again used \cref{eq:identity_mean_ok} and in $(iv)$ we used $(1-\nu\frac{p}{d}\alpha)^2\geq 1/2$.

\begin{lemma}\label{lemma:tighter_computations}
    Fix $d$ sufficiently large. With the definitions and assumptions detailed above, for any unit vector $w\in S^{d-1}$, we have
    \begin{equation*}
        \Ebb_{E,a,b}\sqb{\frac{\dotp{H_{E,a,b}-H_{E_0,a_0,b_0}}{ww^\top}^2}{(1+\nu \|P_E(w)\|^2)(1+\nu\|P_{E_0}(w)\|^2)}} \leq \frac{c_V \beta^2}{p\max(1,\nu)^2},
    \end{equation*}
    for some universal constant $c_V>0$.
\end{lemma}

For convenience, for any subspace $E_0$ and parameters $a_0,b_0\in[0,\alpha]$, we define
\begin{equation*}
    Z_{1:m}(E_0,a_0,b_0):= \frac{1}{m}\sum_{i\in[m]} \frac{Y_i}{1+\nu \dotp{H_{E_0,a_0,b_0}}{w_iw_i^\top}}.
\end{equation*}
Then, \cref{eq:max_progress_made} and \cref{lemma:tighter_computations} together show that
\begin{equation}\label{eq:lb_progress_final}
    \log \Ebb_{E,a,b}\sqb{  \frac{ p(E,a,b\mid w_{1:m},Y_{1:m}) }{p(E_0,a_0,b_0\mid w_{1:m},Y_{1:m})}} \geq -2c_V \frac{\beta^2\nu^2 m}{p\max(1,\nu)^2} Z_{1:m}(E_0,a_0,b_0).
\end{equation}

We recall that by construction we had $(a,b)=(a_\rho,b_\rho)$ where $\rho=\tr(P_EP_{E_0^\perp})/p=:\rho(E,E_0)$. That is, we can rewrite \cref{eq:definition_adjustments} as
\begin{equation*}
        \begin{bmatrix}
a_\rho\\
b_\rho
\end{bmatrix} := (I+\rho A)^{-1} 
\paren{\begin{bmatrix}
    a_0\\
    b_0
\end{bmatrix}
+\rho e_0},\quad \text{where} \quad A:=\begin{bmatrix}
    1 &\frac{p}{d}\\
    -\frac{d}{d-p} & -\frac{p}{d-p}
\end{bmatrix},\quad e_0:=\begin{bmatrix}
    -1\\
    \frac{d}{d-p}
\end{bmatrix}.
    \end{equation*}
For notational convenience we will write $a_{\rho,a_0,b_0}$ and $b_{\rho,a_0,b_0}$ for these quantities $a$ and $b$ to specify which parameters they are constructed with. We then integrate \cref{eq:lb_progress_final} over all
\begin{equation*}
    (E_0,a_0,b_0)\in \Rcal(\tilde E_0,\beta):= B(\tilde E_0,\beta) \cap\set{Z_{1:m}(E_0,a_0,b_0)\leq 5} \cap \{a_0,b_0\in [20\beta,\alpha-20\beta]\},
\end{equation*} 
which gives
\begin{align*}
    &\int_{(E_0,a_0,b_0)\in \Rcal(\tilde E_0,\beta)} p(E_0,a_0,b_0\mid w_{1:m},Y_{1:m})
    \\
    &\leq e^{\frac{10c_V\beta^2\nu^2 m}{p\max(1,\nu)^2}} \int_{(E_0,a_0,b_0)\in \Rcal(\tilde E_0,\beta)}  \int_{E\in B(E_0,2\beta)} \frac{p(E,a_{\rho(E,E_0),a_0,b_0},b_{\rho(E,E_0),a_0,b_0}\mid w_{1:m},Y_{1:m})}{\Vol(B(E_0,2\beta))}\\
    &= e^{\frac{10c_V\beta^2\nu^2 m}{p\max(1,\nu)^2}} \int_{E\in B(\tilde E_0,3\beta)} \int_{\substack{(E_0,a_0,b_0)\in \Rcal(\tilde E_0,\beta)\\ E_0\in B(E,2\beta)}} \frac{p(E,a_{\rho(E,E_0),a_0,b_0},b_{\rho(E,E_0),a_0,b_0}\mid w_{1:m},Y_{1:m})}{\Vol(B(\tilde E_0,2\beta))} \\
    &\overset{(i)}{\leq} e^{\frac{10c_V\beta^2\nu^2 m}{p\max(1,\nu)^2}} \int_{E\in B(\tilde E_0,3\beta)} \int_{\substack{E_0\in B(\tilde E_0,\beta)\\ a,b\in[0,\alpha]}} \frac{ |\det(I+\rho(E,E_0) A)|}{\Vol(B(\tilde E_0,2\beta))} p(E,a,b\mid w_{1:m},Y_{1:m})\\
    &\overset{(ii)}{\leq} e^{\frac{10c_V\beta^2\nu^2 m}{p\max(1,\nu)^2}} (1+20\beta)^d \frac{\Vol(B(\tilde E_0,\beta))}{\Vol(B(\tilde E_0,2\beta))} \int_{\substack{E\in B(\tilde E_0,3\beta)\\ a,b\in[0,\alpha]}} p(E,a,b\mid w_{1:m},Y_{1:m}) \\
    &\overset{(iii)}{\leq} \exp\paren{\frac{10c_V\beta^2\nu^2 m}{p\max(1,\nu)^2} + 20\beta d - p(d-p)\log 2} \leq \exp\paren{\frac{10c_V\beta^2\nu^2 m}{p\max(1,\nu)^2} -\frac{pd}{8}}.
\end{align*}
In $(i)$ we used the property that $a_{\rho,a_0,b_0},b_{\rho,a_0,b_0}\in[0,\alpha]$ when $\rho\leq 4\beta$ and $a_0,b_0\in[20\beta,\alpha-20\beta]$ as proved above. In $(ii)$ we used $\|I+ \rho(E,E_0) A\|_{\op}\leq 1+4\beta \|A\|_{\op}\leq 1+20\beta$. In $(iii)$ we used the fact that $\int_{E,a,b} p(E,a,b\mid w_{1:m},Y_{1:m})=1$ when integrating over the full support, and \cref{lemma:ratio_volumes_bound} to lower bound the ratio of the volumes $B(\tilde E_0,\beta)$.

We now go back to the estimation problem when $E$ is sampled as a Haar $p$-dimensional subspace. We denote by $\widehat E$ the output of the algorithm given the queries $w_{1:m}$ and responses $Y_{1:m}$. Then, the previous inequality shows that if $m\leq \frac{p^2 d}{160c_V\beta^2\min(\nu,1)^2}$, one has
\begin{align*}
    &\Pbb\sqb{\{\|P_{\widehat E}-P_E\|_{\op}\leq \beta\}\cap \{Z_{1:m}(E,a,b)\le 5\}} \\
    &\leq \Pbb[a_0\notin[20\beta,\alpha-20\beta]] + \Pbb[b_0\notin [20\beta,\alpha-20\beta]]  + e^{pd/16-pd/8}\\
    &\leq \frac{80\beta}{\alpha} + e^{-pd/16} = 2^{12}\beta + e^{-pd/16}.
\end{align*}

Next, recall that by definition, at each iteration $i\in[m]$, the variables $Y_i/(1+\nu\dotp{H_{E,a,b}}{w_iw_i^\top})$ are distributed as exponentials $\mathrm{Exp}(1)$ independent from the past history. Hence, standard Chernoff bounds imply
\begin{equation*}
    \Pbb[Z_{1:m}(E,a,b)>5] \leq e^{-2m}.
\end{equation*}
In summary, the algorithm succeeds with probability at most
\begin{equation*}
    \Pbb[\|P_{\widehat E}-P_E\|_{\op}\leq \beta] \leq 2^{12}\beta + e^{-pd/16} + e^{-2m}.
\end{equation*}
Taking $\beta$ sufficiently small and $d$ sufficiently large ends the proof.
\end{proof}

\begin{proof}[Proof of \cref{lemma:tighter_computations}]
    For convenience, we denote by $A(E)$ the quantity within the expectation which we aim to bound.
    Let $R$ be a uniform random rotation preserving $E_0,E_0^\perp$, that is, $R=UV$ where $U$ (resp. $V$) is a Haar random rotation of $E_0$ (resp. $E_0^\perp$). By construction, the distribution of $E$ is invariant by $R$, hence $E$ and $F:=RE$ share the same distribution. We also denote $s:=\|P_{E_0}(w)\|$, and define
    \begin{equation*}
        x:=\frac{1}{s}RP_{E_0}(w) = U\frac{P_{E_0}(w)}{\|P_{E_0}(w)\|} \quad \text{and} \quad y:=\frac{1}{\sqrt{1-s^2}}RP_{E_0^\perp}(w) = V\frac{P_{E_0^\perp}(w)}{\|P_{E_0^\perp}(w)\|},
    \end{equation*}
    where by default if $P_{E_0}(w)=0$, the normalized vector $\frac{P_{E_0}(w)}{\|P_{E_0}(w)\|}$ is arbitrary within $S^{d-1}\cap E_0$ and similarly for $y$.
    In particular, $x\sim\mathrm{Unif}(S^{d-1}\cap E_0)$ and $y\sim\mathrm{Unif}(S^{d-1}\cap E_0^\perp)$ are independent. In the following computations, all expectations are meant only with respect to the random rotation $R$, that is, conditionally on $E$. Then,
    \begin{align}
        &\dotp{H_{F,a,b}-H_{E_0,a_0,b_0}}{ww^\top}^2 \notag\\
        &\overset{(i)}{=} \paren{1+a+\frac{p}{d}b}^2 (\dotp{P_F}{ww^\top} - \Ebb[\dotp{P_F}{ww^\top}] )^2 \notag \\
        &\overset{(ii)}{\leq} 4 (s^4 (x^\top P_E x-\Ebb_x[x^\top P_E x])^2 + 4s^2(x^\top P_E y)^2  + (y^\top P_E y-\Ebb_y[y^\top P_E y])^2 ), \label{eq:upper_bound_numerator}
    \end{align}
    In $(i)$ we used \cref{eq:identity_mean_ok} and in $(ii)$ we used the identity $(u+v+w)^2\leq 3(u^2+v^2+w^2)$ and $a,b\in[0,\alpha]$. 
    
    We next lower bound the denominator. Note that if $\nu s^2\leq 16$ then we directly have
    \begin{equation}\label{eq:lower_bound_denominator}
        1+\nu\|P_F(w)\|^2 \geq 1 \geq \frac{1}{17}(1+\nu s^2).
    \end{equation}
    Otherwise, this implies $s^2 > \frac{16}{\nu} \geq \frac{16p}{d}$ since $\nu\in[0,\frac{d}{p}]$. We will then use the following event
    \begin{equation*}
        \Ecal:=\set{\|P_E(y)\|^2\leq  \frac{p}{d}}.
    \end{equation*}
    Then, under $\Ecal$ one has $\|P_E(y)\|\leq s/4$. Therefore,
    \begin{align*}
        \|P_F(w)\|^2 &= s^2x^\top P_Ex +2s\sqrt{1-s^2} x^\top P_E y + (1-s^2)y^\top P_E y \\
        &\geq s^2(1-\|P_E-P_{E_0}\|_{\op})^2 - 2s\|P_E (y)\| \\
        &\overset{(i)}{\geq} s(3s/4-2\|P_E(y)\|) \overset{(ii)}{\geq} \frac{s^2}{4}.
    \end{align*}
    In $(i)$ we used $\|P_E-P_{E_0}\|_{\op}\leq 2\beta\leq \alpha/10$ and in $(ii)$ we used $\Ecal$.
    
    In summary, under $\Ecal$, \cref{eq:lower_bound_denominator} always holds. Hence, together with the upper bound \cref{eq:upper_bound_numerator} we obtained
    \begin{align*}
        \Ebb_R \sqb{A(F) \1[\Ecal]}  \leq 80 \cdot \frac{s^4\Var_x[x^\top P_E x] + s^2\Ebb_{x,y}[(x^\top P_E y)^2] + \Var_y[y^\top P_E y]}{(1+\nu s^2)^2}.
    \end{align*}

We can further compute explicitly each term as follows
\begin{align*}
    \mathrm{Var}_x[x^\top P_E x] &= \frac{2\|P_{E_0}P_EP_{E_0}- P_{E_0}\tr(P_EP_{E_0})/p\|_F^2}{p(p+2)}  \\
    &\leq \frac{2}{p}\|P_{E_0}P_EP_{E_0}- P_{E_0}\tr(P_EP_{E_0})/p\|_{\op}^2\\ 
    &\overset{(i)}{\leq} \frac{2}{p}(\|P_E-P_{E_0}\|_{\op} + |1-\tr(P_EP_{E_0})/p|)^2 \leq \frac{8}{p}\|P_E-P_{E_0}\|_{\op}^2 \leq \frac{32\beta^2}{p},
\end{align*}
where in $(i)$ we used $\tr((P_{E_0}-P_E)P_{E_0})\leq p\|P_{E_0}-P_E\|_{\op}$.
Similarly, we have
\begin{align*}
    \mathrm{Var}_y[y^\top P_E y]&\leq\frac{2\|P_{E_0^\perp}P_EP_{E_0^\perp}\|_F^2}{(d-p)(d-p+2)} \leq \frac{8}{d^2}\|P_{E_0^\perp}(P_E-P_{E_0})P_{E_0^\perp}\|_F^2 \overset{(i)}{\leq}  \frac{32p}{d^2}\|P_E-P_{E_0}\|_{\op}^2 \leq  \frac{2^7\beta^2 p}{d^2},
\end{align*}
where in $(i)$ we noted that $P_{E_0^\perp}(P_E-P_{E_0})P_{E_0^\perp}$ has rank at most $2p$.
Finally,
\begin{align*}
    \Ebb_{x,y}[(x^\top P_E y)^2]&=\frac{\|P_{E_0}P_EP_{E_0^\perp}\|_F^2}{p(d-p)} \leq \frac{2}{pd}\|P_{E_0}(P_E-P_{E_0})P_{E_0^\perp}\|_F^2 \leq \frac{2^5\beta^2}{d}.
\end{align*}
Putting together the last four inequalities implies
\begin{align}
    \Ebb_R \sqb{A(F)}  &\leq 2^{12}\frac{\beta^2 p}{d^2} \cdot \frac{(1+ds^2/p)^2}{(1+\nu s^2)^2}+ \Ebb_R[\|H_{F,a,b}-H_{E_0,a_0,b_0}\|_{\op}^2 \1[\Ecal^c]] \notag\\
    &\leq 2^{14}\frac{\beta^2 p}{d^2}\paren{\frac{1+d/p}{1+\nu}}^2 + 20\beta^2 \Pbb[\Ecal^c], \label{eq:variance_estimate_bound}
\end{align}
where in the last inequality we used $(1+a+b)^2\leq 5/4$ and $\|P_F-P_{E_0}\|_{\op}\leq 2\beta$. 

We next upper bound the probability $\Pbb[\Ecal^c]$. Note that 
\[
\|P_E(y)\|^2 = y^\top P_{E_0^\perp}P_EP_{E_0^\perp} y\leq \|P_{E_0^\perp}P_EP_{E_0^\perp}\|_{\op} \|P_{\tilde E}(y)\|^2
\]
where $\tilde E$ is the span of the $p$ eigenvectors of $P_{E_0^\perp}P_EP_{E_0^\perp}$ which has rank at most $p$. Notably, since $y\sim\mathrm{Unif}(S^{d-1}\cap E_0^\perp)$, we have $\|P_{\tilde E}(y)\|^2\sim \mathrm{Beta}(\frac{p}{2},\frac{\dim(E_0^\perp)-p}{2})$. Also, $\|P_{E_0^\perp}P_EP_{E_0^\perp}\|_{\op}\leq \|P_E-P_{E_0}\|_{\op}\leq 2\beta$. Altogether, this shows that
$\|P_E(y)\|^2 \preceq_{st}  2\beta\cdot \mathrm{Beta}(\frac{p}{2},\frac{d-p}{2})$.
Therefore, using \cref{lem:beta-tail-local} and the fact that $\beta\leq\alpha/40$ and $p\geq 32\log d$, we have
\begin{equation*}
    \Pbb[\Ecal^c] \leq \Pbb_{Y\sim \mathrm{Beta}\paren{\frac{p}{2},\frac{d-p}{2}}}\sqb{Y>\frac{p}{2\beta d}}  \leq \frac{2}{d^2}.
\end{equation*}
Plugging this estimate within \cref{eq:variance_estimate_bound} gives the desired result:
\begin{equation*}
    \Ebb_R[A(F)] \leq \frac{c_V\beta^2}{p\max(1,\nu)^2},
\end{equation*}
for some universal constant $c_V\geq 1$. Taking the expectation over $E$ and recalling that $\Ebb[A(E)]=\Ebb_{E,R}[A(F)]$ ends the proof.
\end{proof}

\subsection{Completing the proof of Theorem~\ref{thm:lb-intro}}
\label{sec:lb-final}

We are now ready to prove our main regret lower bound for online bandit PCA against the constructed randomized instance.

\begin{theorem}\label{thm:lower_bound_online_PCA}
    There are universal constants $c_0,c_1,c_2>0$ such that the following holds. Fix $d$ sufficiently large, a sparsity integer $r\in[c_1\log d,d/2]$, and $T\in[c_1 d\log d, e^{d^2}]$. If $T<c_0^2r^2d/\log d$, we pose $\nu=\frac{d}{r}$ and $p=r$. Otherwise, we define
    \begin{equation*}
        \nu:=c_0\sqrt{\frac{d^3}{T\log d}}\quad \text{and} \quad p:= \floor{\frac{d}{\max(\nu,2)}}.
    \end{equation*}
    Then, these parameters satisfy $\nu\in[0,\frac{d}{p}]$ and $r\leq p\leq d/2$. Further, on the corresponding constructed randomized instance, any algorithm $alg$ satisfies
    \begin{equation*}
        \Ebb[\Reg_T(alg)] \geq c_2 \min\paren{\sqrt{\frac{r^2 dT}{\log d}},T}.
    \end{equation*}
\end{theorem}
\begin{proof}
    Let $c_0\in(0,1)$ be a constant whose value will be fixed later. We first check that the constraints $r\leq p\leq d/2$ and $\nu\in[0,\frac{d}{p}]$ for the constructed instance are met. These constraints are immediate if $T<c_0^2r^2d/\log d$ by construction since $p=r\leq d/2$.
    We now turn to the case $T\geq c_0^2r^2d/\log d$. By construction, we have $p\leq \frac{d}{\nu}$ or equivalently, $\nu\leq \frac{d}{p}$. 
    Next, by definition of $\nu$ and since $T\geq c_0^2r^2d/\log d$ we have $\nu\leq \frac{d}{r}$ hence $r\leq \frac{d}{\nu}$. We also have $r\leq \frac{d}{2}$ by assumption. Altogether this shows that $r\leq p$ since $r$ is also an integer. Finally, $p\leq d/2$ is immediate by construction.

    In the next step of the proof we check that the assumptions for applying 
    \cref{prop:regret-discovery} are met. First, we can take $c_1>0$ to be a sufficiently large constant so that $J_{\max}=\ceil{\alpha p/16}\geq \ceil{\alpha r/16}\geq 40$. Next, by assumption on $T$, we have $\log\log(eT)\leq 3\log d$ and in particular, $L_T\leq 4\log(ed)$ since $p\leq d$.

    Next, in the case when $T\geq c_0^2r^2d/\log d$, we can take $c_1\geq 2c_\nu \max(c_0,1/c_0)$ sufficiently large so that $p\geq r\geq c_1\log d$ and $L_T\leq 4\log(ed)$ imply for $d$ sufficiently large:
    \begin{equation*}
        \nu_{\min}(d,p,\alpha,T) \leq \frac{c_0d}{2\sqrt T} + \frac{ 4d \log(ed)}{T} \leq \frac{\nu}{2} + \frac{\nu}{2}=\nu.
    \end{equation*}
    Next, in the case when $T<c_0^2r^2d/\log d$ we can take $c_1$ sufficiently large (depending on $c_0$) so that $T\geq c_1d\log d$ and $L_T\leq 4\log(ed)$ imply $\nu_{\min}(d,p,\alpha,T) \leq \frac{d}{p}=\nu$. In summary, in all cases we have $J_{\max}\geq 40$ and $\nu\in[\nu_{\min}(d,p,\alpha,T),\frac{d}{p}]$. We are now ready to apply the various reductions. 

    Suppose that there exists an algorithm $alg$ that achieves expected regret at most $\frac{\alpha r\nu}{32d}T$ on the constructed randomized instance, then \cref{prop:regret-discovery} shows that there is an algorithm for the subspace discovery problem with $k=\ceil{J_{\max}/6}=\Theta(p)$ which (1) uses $m:=T+c_Q\frac{p\log (p\log d)}{\min(\nu,1)^2}$ queries for some universal constant $c_Q>0$, (2) succeeds with probability at least $1/4$, and (3) either abstains or outputs successful vectors with probability at least $1-o(1/\log d)$. Note that if $T\geq c_0^2r^2d/\log d$ we have $\frac{p\log (p\log d)}{\min(\nu,1)^2}\leq\max(c_0^2\frac{\log d}{d^2}T,2d\log d)$ hence $m\leq 2T$ for $d$ and $c_1$ sufficiently large constants since $r\geq c_1\log d$ ($c_1$ may depend on $c_0$). Next if $T<c_0^2r^2d/\log d$ we have $\frac{p\log (p\log d)}{\min(\nu,1)^2} = p\log (p\log d) \leq 2d\log d$. Hence choosing $c_1\geq 4$, in all cases we have $m\leq 2T$.
    
    In turn, \cref{lemma:reduction_to_case_k_tilde_d} shows that there is an algorithm for the covariance estimation problem with $\beta=\beta_0$ that uses at most $m':=c_3T\log d$ queries and succeeds with probability at least $2/3-o(1)\geq 1/2$, for some universal constant $c_3>0$ (independent of $c_0$). Then, \cref{lemma:query_lower_bound_subspace_estimation} gives
    \begin{equation*}
        c_3T\log d=m'\geq \beta_0\frac{p^2 d}{\min(\nu,1)^2} \geq \beta_0 \frac{d^3}{8\min(\nu,1)^2\max(\nu,1)^2} =\beta_0 \frac{ d^3}{8\nu^2} = \frac{\beta_0}{8c_0^2}T\log d.
    \end{equation*}
    We then choose $c_0<\sqrt{\beta_0/(8c_3)}$ to reach a contradiction. This shows that any algorithm $alg$ suffers regret at least $\frac{\alpha r\nu}{32d} T$ on this randomized instance. This corresponds to $\Omega(\sqrt{\frac{r^2d}{\log d}T})$ regret when $T\geq c_0^2r^2d/\log d$ and $\Omega(T)$ regret otherwise, which ends the proof.
\end{proof}

To obtain the desired regret lower bounds for rank-$r$ online bandit PCA from \cref{thm:lb-intro}, it suffices to rescale the previous randomized instance.

\begin{proof}[Proof of \cref{thm:lb-intro}]
    Fix $r\in[d]$ and $T\in[c_1 d\log d, e^{d^2}]$ as in \cref{thm:lower_bound_online_PCA}. We suppose $r\geq c_1\log d$ for now and define $p:=\min(r,d/2)$. We consider the randomized instance from \cref{thm:lower_bound_online_PCA} for these parameters $r$ and $p$ and horizon $T$. By 
    \cref{lemma:bound_operator_norm_overall}, the event
    \begin{equation*}
        \Ecal:=\set{\max_{t\in[T]}\|G_t\|_{\op}\leq 3c_G \log (eT)}
    \end{equation*}
    has probability at least $1-T^{-2}$. We then rescale the instance by a corresponding factor: if at time $t$ the instance selected $G_t$, we select $\widetilde G_t:=G_t/(3c_G\log (eT))$ if $\|G_t\|_{\op}\leq 3c_G\log (eT)$ and we select $\widetilde G_t:=0$ otherwise. 
    By construction, for all $t\in[T]$ we have $\|\widetilde G_t\|_{\op}\leq 1$. Also, by construction, we always have
    \begin{equation*}
        \mathrm{rank}(\widetilde G_t)\leq \mathrm{rank}(G_t) \leq r.
    \end{equation*}
    Last, under $\Ecal$, this instance exactly coincides with rescaled gain matrices $G_t$ and hence achieves the same performance under this event (note that the returned losses are also simply their rescaling by the same factor). In the following, we use indices $\widetilde G_{1:T}$ or $G_{1:T}$ to emphasize whether we consider the regret of the algorithm when using the gain matrices $\widetilde G_{1:T}$ or $G_{1:T}$. Then,
    \begin{align*}
        \abs{\Ebb_{\widetilde G_{1:T}}[\Reg_T(alg)] - \frac{\Ebb_{ G_{1:T}}[\Reg_T(alg)]}{3c_G\log (eT)}  } &\leq \Ebb\sqb{\1[\Ecal^c]\sum_{t\in[T]}(\|G_t\|_{\op}+\|\widetilde G_t\|_{\op})}\\
        &\leq 2T\Ebb\sqb{ \max_{t\in[T]}\|G_t\|_{\op} \cdot \1[\max_{t\in[T]}\|G_t\|_{\op}>3c_G\log (eT)]} \\
        &\overset{(i)}{\leq} 20c_G\frac{\log (eT)}{T}.
    \end{align*}
    In $(i)$ we used \cref{lemma:bound_operator_norm_overall}. Together with the lower bound on $\Ebb_{ G_{1:T}}[\Reg_T(alg)]$ from \cref{thm:lower_bound_online_PCA} this proves the desired regret lower bound $\Omega(\frac{1}{\log (eT)}\min(\sqrt{\frac{r^2d T}{\log d}},T))$ on the gain matrices $\widetilde G_{1:T}$.

    \paragraph{Edge cases $T\geq e^{d^2}$ or $T\leq c_1d\log d$ or $r\leq c_1\log d$.} In these cases, we can directly use the naive lower bound $\Omega(\min(\sqrt{dT},T))$ which holds even for multi-armed bandits, equivalently, when the matrices $G_t$ have fixed known eigenvectors. Under the corresponding randomized instance we therefore have
    \begin{equation*}
        \Ebb_{G_{1:T}}[\Reg_T(alg)] \gtrsim \min(\sqrt{dT},T) \gtrsim \frac{\min(\sqrt{r^2dT/\log d},T)}{\log (eT)},
    \end{equation*}
    where in the last inequality we used (1) $\log T\geq d^2$ in the case $T\geq e^{d^2}$, (2) $r\leq c_1\log d$ in that case, or (3) $\min(\sqrt{dT},T)\gtrsim T$ in the case $T\leq c_1 d\log d$. Note that in the classical mutli-armed bandit lower bound instance, the rewards can be made $1$-sparse, which corresponds to all gain matrices having rank $1$ in the online bandit PCA formulation. Hence, the above regret lower bound implies the desired result.
\end{proof}

\section{Concentration inequalities}
\label{sec:concentration_inequalities}

We will use the following tail bound for 
the beta distribution.

\begin{lemma}[Beta tail bound]
\label{lem:beta-tail-local}
Let $Y\sim\Beta(a,b)$ with $a,b\ge1$.  There is a universal constant $c>0$ such that, for any $t\geq 4$,
\begin{equation}
\label{eq:deviation_beta_random_var}
        \Pbb\sqb{Y\ge \frac{a}{a+b}t}
        \le e^{-b/8} + e^{-a(t-2)/8}.
\end{equation}
In particular, 
if $Y\sim\Beta(\frac{p}{2},\frac{d-p}{2})$,~$p\ge32\log d$, and $\beta\le 1/16$, then
$
        \Pbb\{Y>\frac{p}{4\beta d}\}\le 2d^{-2}
$
for $d$ sufficiently large.
\end{lemma}

\begin{proof}
Represent $Y=X/(X+Z)$ with independent $X\sim\GammaDist(a,1)$ and $Z\sim\GammaDist(b,1)$.  If $Y\ge u$, then either $X\ge at/2$ or $Z\le b/2$, and~\eqref{eq:deviation_beta_random_var} follows by the standard Chernoff bound for~$\GammaDist(n,1)$. The specialization follows by plugging in~$a=b=d/2$ and the bounds on~$d,\beta$.
\end{proof}

The following lemma is a version of the matrix Bernstein inequality (see~\cite[Theorem 1.6]{tropp2012user}).

\begin{lemma}
\label{lem:tropp-bernstein}
Let $X_1,\ldots,X_N \in \mathbb{S}^{d}_+$ be independent, mean-zero, such that
$
        \|\sum_{i \in [N]} \E[X_i^2]\|_{\op}\le \sigma^2
$
and~$\|X_i\|_{\op}\le L$ a.s.
Then there is a universal constant~$C > 0$ such that
\[
        \Pr\Bigg\{
        \Bigg\|\sum_{i \in [N]} X_i\Bigg\|_{\op}
        \ge
        C\left(\sqrt{\sigma^2(u+\log d)}+L(u+\log d)\right)
        \Bigg\}
        \le e^{-u} \qquad \forall u \ge 0.
\]
\end{lemma}

The next lemma gives a Bernstein-type martingale bound, corresponding to our situation.

\begin{lemma}[Time-uniform Bernstein bound]
\label{lem:time-uniform-exp}
Let $(\mathcal F_t)_{t\ge0}$ be a filtration.  For $t \in [T]$, let
$
        X_t=\mu_t Z_t,
$
where $\mu_t$ is $\mathcal F_{t-1}$-measurable, $\mu_t \in [0,\mu_\star]$, and $Z_t | \mathcal F_{t-1} \sim  \Exp(1)$. 
Fix a positive integer $K$ and let~$\tau_0,\tau_1,\ldots,\tau_K$ be stopping times with values in $\{0\} \cup [T]$ 
and for $j \in \{0\} \cup [K]$ and $n\le T-\tau_j$, set
\[
        S_{j,n}:=\sum_{t=\tau_j+1}^{\tau_j+n}(X_t-\mu_t).
\]
Then there is a universal constant $C>0$ such that, for every $\delta\in(0,1)$, with probability at least~$1-\delta$ the following inequality holds simultaneously over all $0 \le j \le K$ and~$1\le n\le T-\tau_j$:
\begin{equation}
\label{eq:time-uniform-exp}
        \frac{|S_{j,n}|}{n}
        \le
        C\mu_\star\bigg(\sqrt{\frac{L}{n}}+\frac{L}{n}\bigg),\quad \text{where} \quad L:=\log\frac{eK}{\delta}+\log\log(eT).
\end{equation}
\end{lemma}

\begin{proof}
For all~$|\lambda|\le (2\mu_\star)^{-1}$, the conditional moment generating function of $D_t:=X_t-\mu_t$ satisfies
\[
\begin{aligned}
        \E\bigl[\exp(\lambda D_t)\mid \mathcal F_{t-1}\bigr]
        &=\frac{\exp(-\lambda\mu_t)}{1-\lambda\mu_t}  
        \le \exp(2\lambda^2\mu_t^2)
        \le \exp(2\lambda^2\mu_\star^2).
\end{aligned}
\]
This remains true if we replace~$\lambda$ by~$-\lambda$.  Hence, for each fixed sign $\sigma\in\{-1,1\}$ and admissible $\lambda$,
\[
        M_n^{\sigma,\lambda}
        :=\exp\bigg(
        \sigma\lambda\sum_{t=1}^nD_t
        -2\lambda^2\mu_\star^2 n
        \bigg)
\]
is a nonnegative supermartingale.  
By the optional stopping theorem for supermartingales (for example,~Theorem 28 in~\cite{dellacherie2011probabilities}), for every stopping time $\tau_j$, the restarted process
\[
        M_{j,n}^{\sigma,\lambda}
        :=\exp\left(
        \sigma\lambda S_{j,n}
        -2\lambda^2\mu_\star^2 n
        \right),
        \qquad n\ge0,
\]
is a nonnegative supermartingale (initialized from $1$).
Now, fix $j$ and a dyadic block $2^q\le n<2^{q+1}$, where $q=\{0\} \cup [Q]$ with $Q:=\lceil\log_2(eT)\rceil$.  Let
\[
        x_q:=\log \left(\frac{4(K+1)(q+2)^2}{\delta}\right).
\]
First let~$x_q\le 2^q/8$.  Taking
$
        \lambda_q:=\frac{1}{4\mu_\star}\sqrt{\frac{x_q}{2^q}}
        \le \frac{1}{2\mu_\star}
$
and applying Ville's inequality to $M_{j,n}^{\sigma,\lambda_q}$, we get
\begin{align}
\label{eq:Ville-1}
\Pp\left\{
        \sup_{2^q\le n<2^{q+1}}\sigma S_{j,n}
        >16\mu_\star\sqrt{2^q x_q}
        \right\}                                      
\le 
        \exp\left(
        -16\lambda_q \mu_\star\sqrt{2^q x_q}
        +2\lambda_q^2\mu_\star^2 2^{q+1}
        \right)
        \le e^{-x_q} .
\end{align}
If $x_q>2^q/8$, we instead use the deterministic bound
$
        \sup_{2^q\le n<2^{q+1}}|S_{j,n}|
        \le \sum_{t=\tau_j+1}^{\tau_j+2^{q+1}}(X_t+\mu_t)
$
together with the exponential-supermartingale argument at the fixed value
$\lambda=(4\mu_\star)^{-1}$.  
This gives
\begin{equation}
\label{eq:Ville-2}
        \Pp\left\{
        \sup_{2^q\le n<2^{q+1}}\sigma S_{j,n}
        >C\mu_\star x_q
        \right\}
        \le e^{-x_q}
\end{equation}
after adjusting the constant.
Combining~\eqref{eq:Ville-1}--\eqref{eq:Ville-2} and increasing $C$ once again, we obtain for every
$j,q$ and both signs
\[
        \Pp\left\{
        \sup_{2^q\le n<2^{q+1}} |S_{j,n}|
        >C\mu_\star\bigl(\sqrt{2^q x_q}+x_q\bigr)
        \right\}
        \le 2e^{-x_q}.
\]
Taking the union bound over $j \in \{0\} \cup [K]$ and~$q \in \{0\} \cup [Q]$ and both signs, we bound the probability of the violation event as
\[
        \sum_{j=0}^K\sum_{q=0}^Q2e^{-x_q}
        \le
        \sum_{j=0}^K\sum_{q=0}^\infty
        \frac{\delta}{2(K+1)(q+2)^2}
        \le \delta .
\]
On the complementary event, if $2^q\le n<2^{q+1}$, then $2^q\le n$ and
$x_q\le L$. 
Thus,~\eqref{eq:time-uniform-exp} is proved up to changing the universal constant $C$.
\end{proof}

\bibliographystyle{alpha}
\bibliography{refs_arxiv}

\end{document}